\documentclass[opre,nonblindrev]{informs3_homepage}              

\usepackage{natbib}
 \def\bibfont{\small}%
 \bibpunct[, ]{(}{)}{,}{a}{}{,}%

\usepackage[colorlinks=true,breaklinks=true,bookmarks=true,urlcolor=blue,
     citecolor=blue,linkcolor=blue,bookmarksopen=false,draft=false]{hyperref}

\def\EMAIL#1{\href{mailto:#1}{#1}}

\TheoremsNumberedThrough     
\theoremstyle{TH}

\EquationsNumberedThrough    

\newtheorem{thrm}{Theorem}
\newtheorem{lemm}{Lemma}
\newtheorem{deffn}{Definition}

\usepackage{algorithm}
\usepackage{algpseudocode}
\usepackage{dsfont}
\usepackage{setspace}
\newenvironment{prf}{\textbf{Proof.}~}{~\hfill $\blacksquare$}

\def \eps {\epsilon}

\def \R {\mathbb{R}}
\def \C {\mathbb{C}}
\def \ssconstant {\alpha}
\def \diff {\mathrm{d}}
\def \statedim {p}
\def \controldim {q}

\def \policy {\boldsymbol{\widehat{\pi}}}
\def \optimalpolicy {\boldsymbol{\pi_{\mathrm{opt}}}}
\def \Qxu {Q_{xu}}
\def \Qx {Q_x}
\def \Qu {Q_u}
\def \Qmat {Q_x}
\def \Rmat {Q_u}
\def \truth {\boldsymbol{\theta_0}}

\def \stabradii {\boldsymbol{\zeta_0}}
\def \auxA {D}
\def \auxQ {R}

\def \curve {\mathcal{C}}
\def \stabstep {\boldsymbol{\kappa_\mathrm{}}}

\def \JordanMat {\Gamma}

\newcommand{\Mnorm}[2]{{\left\vert\kern-0.30ex\left\vert #1 \right\vert\kern-0.30ex\right\vert}}
\newcommand{\norm}[2]{{\left\vert\kern-0.35ex\left\vert #1
		\right\vert\kern-0.35ex\right\vert}}
\newcommand{\eigmax}[1]{\overline{\lambda} \left( #1 \right)}
\newcommand{\eigmin}[1]{\underline{\lambda} \left( #1 \right)}
\newcommand{\tr}[1]{\boldsymbol{\mathrm{tr}} \left( #1 \right)}

\newcommand{\PP}[1]{%
	\mathbb{P}{\ifthenelse{ \equal{#1}{} }{}{\left(#1\right)}}
}%
\newcommand{\E}[1]{\mathbb{E}\left[#1\right]}

\newcommand{\regret}[2]{\mathrm{\textbf{Regret}}_{#2}\left(#1\right)}

\newcommand{\RiccSol}[1]{{P}\left(#1\right)}
\newcommand{\Optgain}[1]{{K}\left(#1\right)}
\newcommand{\Gainmat}[1]{K_{#1}}
\newcommand{\Fmatrix}[1]{L_{#1}}

\newcommand{\optavecost}[1]{\overline{\mathcal{J}}^\star }

\newcommand{\optdisccost}[1]{\mathcal{J}_\gamma^\star}
\newcommand{\instantcost}[2]{\norm{Q^{1/2}\statetwo{#1}\left(#2\right)}{2}^2}
\newcommand{\order}[1]{ {O} \left(#1\right)}

\newcommand{\Amat}[1]{A_{#1}}
\newcommand{\Bmat}[1]{B_{#1}}
\newcommand{\CLmat}[1]{D_{#1}}
\newcommand{\estpara}[1]{\boldsymbol{\widehat{\theta}_{#1}}}
\newcommand{\estA}[1]{\widehat{A}_{#1}}
\newcommand{\estB}[1]{\widehat{B}_{#1}}
\newcommand{\estD}[1]{\widehat{D}_{#1}}

\newcommand{\empiricalcovmat}[1]{\widehat{\Sigma}_{#1}}
\newcommand{\empmean}[1]{\widehat{\mu}_{#1}}
\newcommand{\filter}[1]{\mathcal{F}_{#1}}

\newcommand{\state}[1]{\boldsymbol{x}_{#1}}
\newcommand{\statetwo}[1]{\boldsymbol{z}_{#1}}
\newcommand{\optstatetwo}[1]{\boldsymbol{z}_{#1}(\optimalpolicy)}
\newcommand{\statethree}[1]{y_{#1}}
\newcommand{\action}[1]{\boldsymbol{u}_{#1}}
\newcommand{\itointeg}[4]{\int\limits_{#1}^{#2} {#3} \diff {#4}}
\newcommand{\indicator}[1]{\mathds{1}{\left\{#1\right\}}}

\newcommand*{\BM}[1]{\mathbb{W}_{#1}}
\newcommand{\dither}[1]{w_{#1}}

\newcommand{\BMcov}[1]{\Sigma_{\mathbb{W}}}
\newcommand{\dithercoeff}[1]{\sigma_{w}}
\newcommand{\para}[1]{\boldsymbol{\theta_{#1}}}
\newcommand{\normaldist}[2]{\boldsymbol{N} \left( #1, #2 \right)}
\newcommand{\sigfield}[1]{\mathcal{F} \left( #1 \right)}

\newcommand{\mosteig}[1]{\log \eigmax{ \exp \left( #1 \right)}}

\newcommand{\episodetime}[1]{\boldsymbol{\tau}_{#1}}
\newcommand{\paraspace}[1]{\mathcal{S}_{#1}}
\newcommand{\regterm}[1]{\beta_{#1}}
\newcommand{\randommatrix}[1]{\Phi_{#1}}
\newcommand{\RandMat}[1]{\Psi_{#1}}
\newcommand{\erterm}[1]{{E}_{#1}}

\newcommand{\trans}[1]{D}
\newcommand{\mult}[1]{\boldsymbol{r}_{#1}}
\newcommand{\diag}[1]{\mathrm{diag}\left(#1\right)}

\newcommand{\posterior}[1]{\mathcal{D}_{#1}}
\newcommand{\basis}[1]{\mathrm{e}_{#1}}
\newcommand{\cond}[1]{\mathrm{cond}\left(#1\right)}
\newcommand{\event}[1]{\mathcal{E}_{#1}}

\begin{document}

\RUNTITLE{Thompson Sampling for Controlling Diffusion Process}

\TITLE{Analysis of Thompson Sampling for Controlling Unknown Linear Diffusion Processes}

\RUNAUTHOR{Faradonbeh, Shirani, and Bayati}
\ARTICLEAUTHORS{%
\AUTHOR{Mohamad Kazem Shirani Faradonbeh}
\AFF{Department of Mathematics, Southern Methodist University, \EMAIL{mohamadksf@smu.edu}}
\AUTHOR{Sadegh Shirani}
\AFF{Graduate School of Business, Stanford University, \EMAIL{sshirani@stanford.edu}}
\AUTHOR{Mohsen Bayati}
\AFF{Graduate School of Business, Stanford University, \EMAIL{bayati@stanford.edu}}
} 

\ABSTRACT{%
Linear diffusion processes serve as canonical continuous-time models for dynamic decision-making under uncertainty. These systems evolve according to \emph{drift} matrices that specify the instantaneous rates of change in the expected system state, while also experiencing continuous random disturbances modeled by Brownian noise. For instance, in medical applications such as artificial pancreas systems, the drift matrices represent the internal dynamics of glucose concentrations and how insulin levels influence them. Classical results in stochastic control provide optimal policies under perfect knowledge of the drift matrices. However, practical decision-making scenarios typically feature uncertainty about the drift; in medical contexts, such parameters are patient-specific and unknown, requiring adaptive policies capable of efficiently learning the drift matrices while simultaneously ensuring system stability and optimal performance.

We study the popular Thompson sampling algorithm for decision-making in linear diffusion processes with unknown drift matrices. For this algorithm that designs control policies \emph{as if} samples from a posterior belief about the parameters fully coincide with the unknown truth, we establish efficiency. That is, Thompson sampling learns optimal control actions fast, incurring only a square-root of time regret, and also learns to stabilize the system in a short time period. To our knowledge, this is the first such result for Thompson sampling in a diffusion process control problem. Moreover, our empirical simulations in three settings that involve blood-glucose and flight control demonstrate that Thompson sampling significantly improves regret, compared to the state-of-the-art algorithms, suggesting it explores in a more guarded fashion. Our theoretical analysis includes characterization of a certain \emph{optimality manifold} that relates the geometry of the drift matrices to the optimal control of the diffusion process, among others. We expect the technical contributions to be of independent interest in the study of decision-making under uncertainty problems.}%


\maketitle

\section{Introduction} \label{intro}
Decision-making within dynamic uncertain environments lies at the heart of operations research, spanning many domains from healthcare to supply chain and queueing systems~\citep{song1996inventory,zeltyn2005call,whitt2005engineering,green2006managing,bertsimas2006robust}. When mechanisms governing the evolution of the phenomena under consideration (e.g., the transition rates in a Markovian model) are known, classical methods in stochastic control yield to optimal policies~\citep{harrison1988brownian,yong1999stochastic,pham2009continuous}. However, in practice, decision‐makers often face uncertainties regarding the parameters of the governing dynamical mechanisms, necessitating adaptive approaches that {learn} and {optimize} \emph{concurrently}. Reinforcement learning (RL) has emerged as a powerful paradigm for addressing such sequential decision-making problems, offering a principled framework to balance exploring to gather information about unknown parameters, versus exploiting in the sense of executing optimal decisions the current knowledge prescribes~\citep{sutton2018reinforcement,bertsekas2019reinforcement,szepesvari2022algorithms}.

From insulin delivery devices for blood-glucose regulation to autonomous flight control, decision-making policies increasingly rely on automation approaches~\citep{fox2019reinforcement,koch2019reinforcement}. In many such settings, state-space models are utilized letting the system dynamics evolve continuously over time and its future trajectory be influenced directly by control decisions. Among continuous-time stochastic models, controlled diffusion processes are particularly of interest as they effectively capture uncertainties encountered in practice while maintaining analytical traceability. The list of applications include financial engineering \citep{merton1975optimum}, inventory control \citep{harrison1983instantaneous,wein1992dynamic}, and queueing networks \citep{harrison1988brownian,reiman1984open,dai1991steady}. Mathematically speaking, diffusion processes consist of stochastic differential equations (SDE) that model dynamical systems whose temporal evolution depends on both a deterministic component reflecting the memory (known commonly as the `drift' term), as well as a random `noise' component that captures unpredictable fluctuations~\citep{oksendal2013stochastic}.

Diffusion processes with known parameters offer a flexible framework that enables researchers to derive useful analytical and numerical forms for optimal control policies~\citep{fleming2006controlled,pham2009continuous}. However, in many applications that the governing system dynamics contain unknown parameters, a \emph{learn-as-you-go} decision-making policy is required to merge some learning mechanism (i.e., parameter estimation) into optimal control decisions. This dual mandate of concurrently satisfying diversity and near-optimality of control decisions constitute a canonical paradigm in RL under the terminology of exploration-exploitation dilemma. In addition, a critical distinction from RL problems in Markov decision processes with finite (or bounded) state spaces is that in diffusion control, ensuring \emph{system stability} is a fundamental challenge for designing control policies. Stability prevents the state from growing unbounded, which can lead to catastrophic failures and to rendering optimality irrelevant.

Ensuring the above-mentioned stability introduces a new decision-making challenge that is partially similar to safety concerns (a useful review on safe RL is the work of \cite{gu2024review}.) Broadly speaking, when controlling an unknown system with the dual objectives of parameter estimation and decision optimization, decision-makers face an extra conflict. Poorly chosen exploratory actions (that can be informative for learning), may destabilize the system and cause state trajectories to diverge rapidly. The consequences of such instabilities are particularly severe in some critical applications. For example, in insulin delivery systems, instability can lead to dangerous blood-glucose fluctuations posing significant risks to patient health, and in autonomous flight control systems, instability may result in catastrophic loss of aircraft altitude control.
Unique characteristics of diffusion processes—including unbounded state spaces and the infinitesimal nature of state transitions—introduce additional complexities not addressed by existing frameworks, limiting the development of guaranteed stabilizing procedures for these systems. We elaborate later in this section on the stability problem together with its treatment, while full technicalities are provided in its corresponding section.

In this paper, we study the design and analysis of RL policies for linear stochastic differential equations with unknown drift parameters while minimizing a quadratic cost function—a canonical problem in stochastic control~\citep{yong1999stochastic}. We address the two fundamental challenges of balancing exploration versus exploitation, as well as system stabilization under uncertainty, by developing algorithms based on Thompson sampling (TS)~\citep{thompson1933likelihood}. This approach relies on calculating a Bayesian posterior belief about the unknown drift parameters, using the observed state trajectory, and then designing a control policy by treating a sample from this posterior as the true system parameters.

While TS is studied in various problems~\citep{russo2017tutorial,abeille2018improved,ferreira2018online,faradonbeh2020adaptive}, its ability to efficiently learn optimal policies for diffusion processes while ensuring stability has remained unknown. We establish that TS successfully learns to stabilize unknown linear diffusion processes after interacting with the environment for a short time period, providing the first such theoretical guarantee for unknown diffusion processes. We further derive regret bounds showing that TS achieves the square-root of time regret, demonstrating its efficiency. These results contribute to the growing literature on RL algorithms for continuous-time control systems, highlighting the effectiveness of Bayesian approaches for safety-critical applications. Our results also characterize how other quantities impact the performance and are validated through numerical experiments for three real systems, as will be elaborated shortly. 

Central to our analysis is the characterization of a manifold that reveals the relationship between drift parameters in diffusion processes and optimal control strategies. Using the geometry of this manifold, we study how the posterior beliefs in TS distribute in directions \emph{along} the manifold (corresponding to near-optimal control policies), as well as in directions \emph{orthogonal} to the manifold that reflect the exploration and the performance of the algorithm for estimating the drift parameters. We establish efficiency of TS by showing that the algorithm automatically manages its belief across both the orthogonal and tangent directions without any prior awareness of the geometry. This constitutes another technical contribution of this work: a mathematical expression for assessing endogenous learning algorithms (such as TS) and providing an equation for formulating the accumulation of certainty over time. These two technical developments as well as the rest of contributions that will be discussed at the end of this section, offer valuable insights for analysis of RL algorithms for sequential decision-making under uncertainty.

The organization of the subsequent sections is as follows. We begin with an instructive example in Subsection~\ref{sec:example} that illustrates the key challenges for decision-making under uncertainty in an unknown linear diffusion process, followed by a discussion of our contributions and main results in Subsection~\ref{sec:contributions}. We discuss additional relevant literature in Section~\ref{sec:li_rev}. Section~\ref{ProblemSection} formally formulates the problem, while Section~\ref{StabilitySection} presents Algorithm~\ref{algo1} which utilizes TS for learning to stabilize the process, along with its high-probability performance guarantees. In Section~\ref{CostMinSection}, we extend our study of TS for minimizing a quadratic cost function and rigorously establish rates for both estimation accuracy and regret bounds. Section~\ref{NumericalSection} demonstrates our theoretical findings through comprehensive numerical experiments across multiple application domains. The paper is concluded with remarks on implications and future research directions in Section~\ref{concluding_rem}. For space considerations, all auxiliary lemmas and their proofs are delegated to the appendices.

\subsection{An Illustrative Example: Artificial Pancreas}
\label{sec:example}
To illustrate how the theoretical developments of this paper relate to practical situations, we consider artificial pancreas (AP) systems for diabetes. An AP represents a control system designed to regulate blood glucose levels in patients with diabetes mellitus (T1DM)~\citep{russell2014outpatient}. Unlike a healthy pancreas that automatically regulates insulin production in the body, T1DM patients lack this ability and must manually manage insulin dosing multiple times a day. An AP system automates this process by combining continuous glucose monitoring with automated injection through an insulin pump~\citep{doyle2014closed}.

The dynamics of blood glucose concentration can be approximated by linear stochastic differential equations similar to our upcoming model in \eqref{dynamics}, albeit with simplifications. Bergman's minimal model and its extensions incorporate meal disturbances and insulin actions, providing representations that balance complexity and tractability~\citep{bergman1979quantitative}. The control inputs correspond to the insulin infusion rates, while the diffusion term captures various sources of disturbance and randomness, including meal absorption variability, sensor inaccuracies, and metabolic fluctuations~\citep{sirlanci2023simple}.

Several aspects of AP systems illustrate key challenges in diffusion process control that align with our theoretical investigations:

\paragraph{Parameter uncertainty:} The drift matrices in glucose-insulin dynamics contain patient-specific parameters that vary significantly between individuals~\citep{hua2022personalized}. This parallels the problem of learning unknown parameters that is central to our work. Note that since each AP set needs to be \emph{personalized} to its patient, such schemes are strongly preferred to follow a learn-as-you-go structure. 

\paragraph{Stability considerations:} Hypoglycemia (i.e., overly low blood glucose) can be life-threatening, making system stability a primary requirement in the design and analysis of AP systems~\citep{doyle2014closed}. Consequently, AP systems require to adaptively plan to ensure that high and low values, as well as fast variations in glucose levels, all are precluded. Our theoretical focus on stabilization-guarantees addresses a fundamental concern in such safety-critical applications where avoiding instability is vital.

\paragraph{Performance criteria:} Minimizing the variability of the blood glucose level together with its deviations from the healthy target interval for the patient in the long-term, is the main goal of the AP system~\citep{gondhalekar2016periodic}. This connects to our regret analysis, which quantifies the gap in the performance (of an AP) relative to an optimal policy (e.g., a natural pancreas).

Overall, AP systems must balance quickly learning patient parameters while simultaneously infusing optimal dosages of insulin(s), which is the focus of the current paper. However, an important limitation of our framework when considering AP applications stems from our assumption of a linear time-invariant drift model. In practice, patient parameters may exhibit significant distortions from such dynamical models, due to possible insulin sensitivity changes throughout the day originating from variations in the patterns of meals, exercise, stress, and sleep~\citep{visentin2014university}. Our model assumes fixed drift parameters which remains reasonable as long as there is no systematic change in patient's temporal dynamics. Modern AP systems renew their linear time-invariant dynamical models after each systematic change or adopt time-varying parameters~\citep{haidar2016artificial}. Both of these methods for AP systems (i.e., learning-based adaptive resets and time-variant dynamics) fall beyond the scope of this paper, while the results of our work paves the road towards studying them as interesting future directions.

\subsection{Main Results and Contributions}
\label{sec:contributions}
This work first establishes that TS learns to stabilize diffusion processes. Specifically, in Theorem~\ref{StabThm} in Section~\ref{StabilitySection}, we provide the first theoretical guarantee for learning-based stabilization of linear diffusion processes with unknown drift matrices, showing that the success probability of Algorithm~\ref{algo1} grows to 1 at an exponential rate that depends on the time allocated to stabilization. Further, effects of other factors on the success probability are provided, including state dimension, control dimension, diffusion covariance, stability margin, and the eigen-structure of the system transition matrices. 

Then, we establish efficiency of TS in optimally balancing the trade-off between the exploration and exploitation. The performance criterion is the widely-used regret; the cumulative gap between the operating costs of TS and optimal policy that reflects the performance degradation due to uncertainty about the drift matrices. In Theorem~\ref{RegretThm} in Section~\ref{CostMinSection}, we show that the regret incurred by Algorithm~\ref{algo2} grows as the \emph{square-root of time} (up to logarithmic factors), while the squared estimation error decays at the same rate. Furthermore, we prove that both quantities exhibit linear growth with the total number of unknown parameters in the drift matrices. To the authors' knowledge, these results constitute the first theoretical performance guarantees for TS in controlling diffusion processes.

Moreover, through extensive numerical simulations with drift matrices of three real systems, in Section~\ref{NumericalSection} we demonstrate that TS outperforms the competing candidates. This includes flight control systems for both X-29A and Boeing 747 aircrafts, as well as an artificial pancreas system for glucose monitoring. These numerical experiments corroborate the theoretical result on stabilization under uncertainty that our method learns fast to stabilize with high probability. Furthermore, they demonstrate that TS achieves smaller regret when averaging across multiple repetitions of the experiment, and more importantly, TS substantially lowers worst-case regret compared to the state-of-the-art RL policies~\citep{faradonbeh2023online}. Intuitively, this performance advantage stems from the more informed exploration mechanism inherent to TS.

It is important to emphasize that the theoretical analysis of RL policies for diffusion processes presents significant challenges not encountered in other settings. The conventional analytical tools are not applicable due to the uncountable cardinality of random variables and control decisions in diffusion processes (as apposed to discrete-time settings), and due to the noise dominance in the infinitesimal dynamics of the system (leading to high errors in estimation), among others. To address such challenges, we make the following principal contributions.
\begin{enumerate}
    \item 
    We develop a set of results to precisely quantify estimation accuracy when learning from finite-length data trajectories of diffusion processes. First, we establish non-asymptotic and uniform upper bounds for \emph{double-stochastic Ito integrals} and show that their distributions have a sub-exponential tail property (Lemmas~\ref{CrossTermLem1} and~\ref{CrossTermLem2}). Furthermore, since our analysis needs to cope with ill-conditioned information matrices, we develop stochastic inequalities for \emph{self-normalized continuous-time martingales} (Lemma~\ref{SelfNormalizedLem}).
    \item 
    The fact that stability involves the eigenvalues of the transition matrix, necessitates useful bounds for translating the effects of uncertainties on eigen-structures. The highly nonlinear relationships between the eigenvalues of the closed-loop transition matrix, its entries, and errors in estimating the drift matrices, render this task challenging. To address that, we perform a tight \emph{eigenvalue perturbation-analysis} showing Holder continuity of eigenvalues with respect to entries (Lemma~\ref{EigPerturbLem}). The bounds are fully characterized based on multiplicities of eigenvalues, dimensions of eigenspaces, and the angles between the eigenvectors. We also establish the \emph{Lipschitz continuity} of the closed-loop dynamics with respect to estimation errors of drift matrices by developing new techniques based on integrating along matrix-valued curves (Lemma~\ref{LipschitzLemma}). 
    \item 
    Every control action taken by a decision-maker instantaneously determines the operating cost at the time, and also has steady future consequences by steering the system state according to the process in \eqref{dynamics}. So, for regret analysis, one requires to quantify the effects of sub-optimal control actions on both current as well as future increases in the cost functions. We employ useful tools from stochastic analysis (e.g., Ito isometry) and develop new techniques, to precisely express the regret a generic non-anticipating control policy incurs (Lemma~\ref{GeneralRegretLem}). Intuitively, this result can be deemed as a \emph{stochastic, non-asymptotic, and generalized} (to non-optimal policies) version of the Hamilton-Jacobi-Bellman equation.   
    \item 
    To ensure that the space of unknown drift matrices is sufficiently explored while avoiding excessive explorations, Algorithm~\ref{algo2} utilizes sampling from a data-driven posterior distribution. Analysis of this endogenous randomization scheme is technically involved and demands new approaches. Our analysis relies on \emph{optimality manifolds} that the exploitation and exploration in TS control policy correspond to their tangent space and orthogonal subspace, respectively (Lemma~\ref{OptManifoldLemma}). Through the geometry of these manifolds, we quantify the exploratory behavior of posterior samples and how the information is shaped as time proceeds. Specifically, to establish the regret bound and the estimation rates in Theorem~\ref{RegretThm}, we identify a random matrix that perfectly characterizes the \emph{minimal information} the signals provide about the unknown drift matrices (Lemma~\ref{MinPELem}). 
\end{enumerate}

These technical advances enable the theoretical analysis presented in this paper and also set a cornerstone for analyzing RL algorithms in continuous-time stochastic systems more broadly.

\section{Relevant Literature} \label{sec:li_rev}
Classical results formalize optimal policies for continuous-time stochastic linear systems through the Riccati and the Hamilton-Jacobi-Bellman equations \citep{yong1999stochastic}. Infinite-time consistency results have been established under various technical assumptions \citep{duncan1990adaptive,caines1992continuous}, followed by alternating policies that cause (small) linear regrets~\citep{mandl1989consistency,duncan1999adaptive,caines2019stochastic}. Robust control variants designed to hedge against worst-case drift scenarios were thoroughly examined as well~\citep{ioannou1996robust,lewis2017optimal,subrahmanyam2019identification,umenberger2019robust}. In contradistinction to these foundational works that assume known drift matrices or focus on asymptotic guarantees, our approach provides finite-time learning guarantees under complete uncertainty about the drift matrices.

From a computational point of view, pure exploration approaches are considered to compute optimal policies based on multiple trajectories of system's state and control action~\citep{rizvi2018output,doya2000reinforcement,wang2020reinforcement}, for which a useful survey is available~\citep{jiang2020learning}. However, papers that study exploration versus exploitation and provide non-asymptotic estimation rates or regret bounds are limited to a few recent works in offline RL~\citep{basei2021logarithmic,szpruch2021exploration,faradonbeh2022bayesian}, with the exception of Randomized-Estimates policies~\citep{faradonbeh2023online} that our Thompson sampling approach outperforms as will be demonstrated in the numerical experiments. 

This work also relates to safe RL framework that focuses on performing exploration subject to some explicit safety constraints~\citep{gu2024review,garcia2015comprehensive}. Key algorithmic directions include risk-sensitive criteria \citep{chow2015risk}, Lyapunov-guided constraint satisfaction \citep{chow2018lyapunov}, conservative exploration in contextual bandits \citep{kazerouni2017conservative}, and model-based certificates \citep{berkenkamp2017safe}. These approaches are exclusively developed for avoiding unsafe regions (in discrete-time settings) and are inapplicable to study of stability requirements in this work.

The existing literature also studies the efficiency of TS for learning optimal decisions in multi-armed bandits~\citep{agrawal2012analysis,agrawal2013further,gopalan2015thompson,kim2017thompson,abeille2017linear,hamidi2020worst,hamidi2023elliptical}. In this stream of research, it is shown that over time, the posterior distribution concentrates around low-cost arms~\citep{russo2014learning,russo2016information,russo2017tutorial}. TS is studied in further discrete-time settings with the environment represented by parameters that belong to a continuum, and Bayesian and frequentist regret bounds are shown for linear-quadratic regulators~\citep{feldbaum1960dual,florentin1962optimal,tse1973actively,sternby1976simple,stengel1994optimal,wittenmark1995adaptive,klenske2016dual,abeille2018improved,ouyang2019posterior,faradonbeh2018bfinite,faradonbeh2020adaptive,sudhakara2021scalable}. However, effectiveness of TS in highly noisy environments that are modeled by diffusion processes remains unexplored to date.

From a practical point of view, Thompson sampling has been successfully applied across various domains. In revenue management, researchers employ it for dynamic pricing \citep{besbes2009dynamic,ferreira2018online,bastani2022meta}, while others demonstrate its effectiveness for inventory control and assortment optimization~\citep{cheung2017thompson,agrawal2019mnl,zhang2025thompson}. Unlike these applications that primarily focus on bandit models or memoryless settings in finite spaces, our work addresses the challenging unbounded state spaces of diffusion processes, where ensuring stability is a fundamental requirement before optimal decision-making can occur.

\section{Problem Statement} \label{ProblemSection}
In this section, we establish the mathematical framework for our subsequent statements and the continuous-time reinforcement learning problem under consideration. We begin by introducing the requisite notation and the probability space that formalizes our problem domain. We then formulate the linear diffusion control problem with unknown drift matrices.

\subsection{Preliminaries and Notations} 
\label{sec:notations}
 
The following notational conventions will be employed throughout the paper. The smallest and the largest eigenvalues of matrix $M$, in magnitude, are denoted by $\eigmin{M}$ and $\eigmax{M}$, respectively. For a vector $a$, $\norm{a}{}$ is the $\ell_2$ norm, and for a matrix $M$,  $\norm{M}{}$ is the operator norm that is the supremum of $\norm{Ma}{}$ for $a$ on the unit sphere. 
$\normaldist{\mu}{\Sigma}$ is Gaussian distribution with mean $\mu$ and covariance $\Sigma$. If $\mu$ is a \emph{matrix} (instead of vector), then $\normaldist{\mu}{\Sigma}$ denotes a distribution on matrices of the same dimension as $\mu$, such that \emph{all columns are statistically independent} and share the covariance matrix $\Sigma$.
As will be elaborated shortly, we show the dynamics parameters of the diffusion processes via $\para{}$. That is, transition matrices $\Amat{} \in \R^{\statedim \times \statedim}$ together with input matrices $\Bmat{} \in \R^{\statedim \times \controldim}$ are jointly denoted by the $(\statedim+\controldim) \times \statedim$ parameter matrix $\para{} = \left[\Amat{},\Bmat{}\right]^{\top}$.
We employ $\vee$ and $\wedge$ for maximum and minimum, respectively. Finally, let $a \lesssim b$ express that $a \leq \ssconstant_0 b$, for some fixed constant $\ssconstant_0$.

We fix the complete probability space $(\Omega,\left\{\filter{t}\right\}_{t\geq0}, \mathbb{P})$, where $\Omega$ is the sample space, $\left\{\filter{t}\right\}_{t\geq0}$ is a continuous-time filtration (i.e., an increasing family of sigma-fields), $\mathbb{P}$ is the probability measure defined on $\filter{\infty}$, and $\mathbb{E}$ is expectation with respect to $\mathbb{P}$. In addition, we define the $\statedim$ dimensional Wiener process $\left\{\BM{t}\right\}_{t \geq 0}$. That is, $\left\{\BM{t}\right\}_{t \geq 0}$ is a multivariate Gaussian process with independent increments and with the stationary covariance matrix $\BMcov{}$, such that for all $0 \leq s_1 \leq s_2 \leq t_1 \leq t_2$, we have 
\begin{equation} \label{WienerEq}
\begin{bmatrix}
\BM{t_2}-\BM{t_1} \\
\BM{s_2}-\BM{s_1}
\end{bmatrix} \sim \normaldist{\begin{bmatrix}
	0_{\statedim} \\ 0_{\statedim}
	\end{bmatrix}}{\begin{bmatrix}
	(t_2-t_1) \BMcov{} & 0_{\statedim \times \statedim} \\
	0_{\statedim \times \statedim} & (s_2-s_1) \BMcov{}
	\end{bmatrix}}.
\end{equation}
Existence, construction, continuity, and non-differentiability of Wiener processes (also known as Brownian motions) are well-established facts in the literature of stochastic analysis~\citep{karatzas2012brownian}.

\subsection{Diffusion Control Problem with Unknown Drift Matrices}
\label{sec:diff_control_problem}
Consider a dynamical system in continuous time so at time instant $t \geq 0$, the state of the system is denoted by a $\statedim$ dimensional vector $\state{t}$ that obeys the Ito stochastic differential equation:
\begin{equation} \label{dynamics}
\diff \state{t} = \left( \Amat{0} \state{t} + \Bmat{0} \action{t} \right) \diff t + \diff \BM{t}\,,
\end{equation}
where the \emph{drift matrices} $\Amat{0}$ and $\Bmat{0}$ are unknown,
$\action{t} \in \R^{\controldim}$ is the control action at any time $t \geq 0$, and it is
designed based on values of $\state{s}$ for $s\in[0,t]$. The matrix $\Bmat{0}\in \R^{\statedim \times \controldim}$ models the influence of the control action on the state dynamics over time, while $\Amat{0} \in \R^{\statedim \times \statedim}$ is the (open-loop) transition matrix reflecting interactions between the coordinates of the state vector $\state{t}$. The diffusion term $\BM{t}$ is a non-standard Wiener process with covariance matrix $\BMcov{}$.

We collectively use $\truth=\left[\Amat{0},\Bmat{0}\right]^{\top} \in \R^{\left(\statedim+\controldim\right)\times \statedim}$ to denote the unknown drift parameters and assume that $\BMcov{}$ is a positive definite matrix. This is a common assumption in learning-based control literature to ensure feasibility of accurate estimation over time ~\citep{jiang2020learning,basei2021logarithmic,faradonbeh2023online,szpruch2021exploration}. In scenarios where $\BMcov{}$ exhibits singularity—arising from model over-parameterization or structural redundancies in the system representation—a state projection onto the reachable subspace can be formulated. This yields a reduced-dimension process with a positive definite covariance matrix~\citep{ioannou1996robust}, and the analytical framework and theoretical guarantees established in this paper remain valid.

Linear diffusion models such as \eqref{dynamics} have been applied in various settings. For example, \cite{harrison2004dynamic} adopt a similar SDE to capture the queue–length process in a multiclass service system, while \cite{sirlanci2023simple} leverage a linear SDE to describe the evolution of blood–glucose levels for diabetes management. Considering stochastic systems evolving according to \eqref{dynamics}, we next define the decision-making criteria.

The goal is to study efficient RL policies that design the control action $\left\{\action{t}\right\}_{t \geq 0}$, based on the observed system-state by the time, as well as the previously applied control actions, to minimize the following long-run average cost:
\begin{equation} \label{AveCostDef}
\limsup\limits_{T \to \infty} \frac{1}{T} \itointeg{0}{T}{ \left( \norm{ \begin{bmatrix}
		\Qx & ~\Qxu \\ \Qxu^{\top} & ~\Qu
		\end{bmatrix}^{1/2} \begin{bmatrix}
	\state{t} \\ \action{t}
	\end{bmatrix}}{2}^2 \right) }{t},
\end{equation}
subject to uncertainties around $\Amat{0}$ and $\Bmat{0}$. Above, the cost is determined by the positive definite matrix $Q=\begin{bmatrix}
\Qx & \Qxu \\ \Qxu^{\top} & \Qu
\end{bmatrix}$, where $\Qx \in  \R^{\statedim \times \statedim}$, $\Qu \in \R^{\controldim \times \controldim}$, $\Qxu \in \R^{\statedim \times \controldim}$. In fact, $Q$ serves as the weighting matrix that defines the relative importance of different coordinates of $\state{t}$ and $\action{t}$ in the objective function. Intuitively, minimizing the cost function in \eqref{AveCostDef} incentivizes the policy to minimize state while employing control actions of minimal magnitude. Assuming the cost coefficients matrix $Q$ is known to the policy, the problem is to minimize~\eqref{AveCostDef} by applying the policy
\begin{equation} \label{InputPolicyMappingEq}
\action{t} = \policy \left( Q, \left\{ \state{s} \right\}_{0 \leq s \leq t}, \left\{ \action{s} \right\}_{0 \leq s < t} \right).
\end{equation}
Without loss of generality, and for the ease of presentation, we follow the canonical formulation that sets $\Qxu=0$; one can simply convert the case $\Qxu \neq 0$ to the canonical form, by employing a rotation to $\state{t},\action{t}$~\citep{chen1995linear,yong1999stochastic,bhattacharyya2022linear}. It is well-known that if, \emph{hypothetically}, the truth $\truth$ was known, then an optimal policy $\optimalpolicy$ could be explicitly found by solving the continuous-time algebraic Riccati equation. Precisely speaking, for a generic drift matrix $\para{}=\left[\Amat{},\Bmat{}\right]^\top$, we find the symmetric $\statedim \times \statedim$ matrix $\RiccSol{\para{}}$ that satisfies
\begin{equation} \label{ARiccEq}
\Amat{}^{\top} \RiccSol{\para{}} + \RiccSol{\para{}} \Amat{} - \RiccSol{\para{}} \Bmat{} \Qu^{-1} \Bmat{}^{\top} \RiccSol{\para{}} + \Qx=0.
\end{equation}
Specifically, for the true parameter $\truth=\left[\Amat{0},\Bmat{0}\right]^\top$, 
we let $\RiccSol{\truth}$ solve the above equation, and define the policy 
\begin{equation} \label{OptimalPolicy} 
\optimalpolicy : \:\:\:\:\: \action{t}= - \Qu^{-1} \Bmat{0}^{\top}\RiccSol{\truth} \state{t}, \:\:\: \forall t \geq 0.
\end{equation}
Then, the linear time-invariant policy $\optimalpolicy$ minimizes the average cost in~\eqref{AveCostDef}~\citep{chen1995linear,yong1999stochastic}. This optimal policy also stabilizes the system such that under $\optimalpolicy$, the diffusion process $\state{t}$ does not grow unbounded with time. Below, we define the concept of stabilizability and elaborate on relevant properties.
\begin{deffn} \label{StabDef}
	The process in \eqref{dynamics} is stabilizable if there exists a feedback matrix $\Gainmat{}$ such that all eigenvalues of the closed-loop matrix $\CLmat{}={\Amat{0}+\Bmat{0}\Gainmat{}}$ have strictly negative real-parts. In this case, $\Gainmat{}$ is referred to as a stabilizer, and $\CLmat{}$ is termed a stable closed-loop matrix.
\end{deffn}
We assume that the process \eqref{dynamics} with the drift parameter $\truth$ is stabilizable. Therefore, $\RiccSol{\truth}$ exists, is unique, and can be computed according to the Riccati equation in \eqref{ARiccEq}~\citep{chen1995linear,yong1999stochastic}. Furthermore, it is well-known that real-parts of all eigenvalues of $\CLmat{0}=\Amat{0}-\Bmat{0}\Qu^{-1} \Bmat{0}^{\top} \RiccSol{\truth}$ are negative, implying that the matrix $\exp \left( \CLmat{0} t \right)$ decays exponentially fast as $t$ grows~\citep{chen1995linear,yong1999stochastic}. So, under the linear feedback in \eqref{OptimalPolicy}, the closed-loop transition matrix is $\CLmat{0}$ and the solution of~\eqref{dynamics} is the Ornstein–Uhlenbeck process~\citep{karatzas2012brownian}:
$$\state{t}=e^{\CLmat{0}t} \state{0} + \itointeg{0}{t}{e^{\CLmat{0}(t-s)}}{\BM{s}},$$
which evolves in a stable manner because of $\left|\eigmax{\exp \left( \CLmat{0} \right)}\right|<1$. In the sequel, we use~\eqref{ARiccEq} and refer to the solution $\RiccSol{\para{}}$ for different stabilizable $\para{}$. More details about the above optimal feedback-policy can be found in the aforementioned references.

In the absence of complete knowledge of $\truth$, any adaptive policy $\policy$ must achieve two critical objectives simultaneously. First, it must ensure system stability through appropriate data collection and learning. Second, it must minimize the performance gap relative to the optimal policy $\optimalpolicy$ in~\eqref{OptimalPolicy}. The cumulative performance degradation due to this gap is the \emph{regret} of the policy $\policy$, that we aim to minimize. Whenever the control action $\action{t}$ is designed by the policy $\policy$ according to~\eqref{InputPolicyMappingEq}, we concatenate the resulting state and control action to get the observation vector 
$$\statetwo{t}(\policy)=\left[\state{t}^{\top},\action{t}^{\top}\right]^{\top} \in \R^{\statedim+\controldim}.$$
If it is clear from the context, we drop $\policy$. Similarly, we let $\optstatetwo{t}$ denote the observation vector under optimal policy $\optimalpolicy$. Considering the cost function \eqref{AveCostDef}, the regret at time $T$ is defined as
\begin{eqnarray*}
\regret{T}{\policy} = \itointeg{0}{T}{ \left( \norm{Q^{1/2} \statetwo{t}(\policy)}{2}^2 - \norm{Q^{1/2} \optstatetwo{t}}{2}^2 \right) }{t}.
\end{eqnarray*}
Beyond stability preservation and regret minimization, a secondary criterion is the statistical efficiency in estimating the unknown parameter $\truth$ from a single observed trajectory generated by $\policy$. Denoting the parameter estimate at time $t$ as $\estpara{t}$, we seek to characterize the convergence rate of the estimation error $\Mnorm{\estpara{t}-\truth}{2}$ with respect to the time horizon $t$, state dimension $\statedim$, and control dimension $\controldim$. This analysis provides fundamental insights into the sample complexity of the learning problem and the scalability of the proposed approach.

\section{Stabilizing Linear Diffusion Processes} \label{StabilitySection}
This section establishes that Thompson sampling (TS) learns to stabilize the diffusion process~\eqref{dynamics}. We begin with an intuitive analysis of the stabilization problem under uncertainty. Since the optimal policy in~\eqref{OptimalPolicy} stabilizes the process in \eqref{dynamics}, a natural candidate for stabilizing a system with unknown drift matrices $\Amat{0},\Bmat{0}$ is a linear feedback controller of the form $\action{t}=\Gainmat{}\state{t}$. With this control law and closed-loop transition matrix $\CLmat{}=\Amat{0}+\Bmat{0}\Gainmat{}$, the state evolution follows 
$$\state{t}=e^{\CLmat{}t} \state{0} + \itointeg{0}{t}{e^{\CLmat{}(t-s)}}{\BM{s}}.$$

Observe that if the real-part of at least one eigenvalue of $\CLmat{}$ is non-negative, the state vector $\state{t}$ grows unbounded with time $t$ \citep{karatzas2012brownian}. Hence, addressing system stability is a critical prerequisite to cost minimization in our optimization framework. Without stabilization, the regret grows at least linearly—and potentially superlinearly—with time. Notably, when $\Amat{0}$ has eigenvalues with non-negative real-parts, appropriate feedback control becomes essential to ensure stability.

In addition to minimizing the cost, the solution of the algebraic Riccati equation in~\eqref{ARiccEq} provides a reliable and widely-used framework for stabilization, as discussed after Definition~\ref{StabDef}. Accordingly, due to uncertainty about $\truth$, one can solve~\eqref{ARiccEq} and find $\RiccSol{\estpara{}}$, only for an approximation $\estpara{}$ of $\truth$. Then, we expect to stabilize the system in~\eqref{dynamics} by applying a linear feedback that is designed for the approximate drift matrix $\estpara{}$. Technically, we need to ensure that all eigenvalues of $\Amat{0}-\Bmat{0}\Qu^{-1} \estB{}^{\top} \RiccSol{\estpara{}}$ lie in the open left half-plane. To establish this condition with high probability, our methodology must address three fundamental challenges:
\begin{enumerate}
	\item 
    Efficient parameter estimation: developing algorithms that guarantee rapid convergence to a sufficiently accurate approximation of $\truth$ with minimal data requirements,
	\item 
    Robustness analysis: precisely characterizing how estimation errors $\estpara{}-\truth$ propagate through the closed-loop system to affect the eigenvalue distribution of $\Amat{0}-\Bmat{0}\Qu^{-1} \estB{}^{\top} \RiccSol{\estpara{}}$, and
	\item 
    Failure mitigation: constructing contingency strategies for cases where the stabilization procedure fails to achieve the desired eigenvalue configuration.
\end{enumerate} 
Note that the last challenge is unavoidable; learning from finite data can never be perfectly accurate, and so any finite-time stabilization procedure has a (possibly small) positive failure probability.

Algorithm~\ref{algo1} provides a systematic solution to these challenges by employing randomized control actions to construct a posterior belief $\posterior{}$ about the unknown parameter $\truth$. This approach represents a decision-making strategy under uncertainty. Despite the posterior distribution not being tightly concentrated around $\truth$—meaning any sample $\estpara{}$ from $\posterior{}$ provides only a crude approximation of $\truth$—it nevertheless yields strong performance guarantees. Indeed, Theorem~\ref{StabThm} establishes that the probability of failure for Algorithm~\ref{algo1} exhibits exponential decay with respect to the execution time. Moreover, the already small failure probability can be reduced further through iterative sampling from $\posterior{}$. Consequently, our methodology guarantees stabilization under parameter uncertainty after a limited number of resampling iterations from $\posterior{}$, providing both theoretical assurance and practical efficiency.

To proceed, for some integer $\stabstep$, let $\{\dither{n}\}_{n=0}^{\stabstep}$ be a sequence of independent Gaussian vectors with the distribution $\dither{n} \sim \normaldist{0}{ \dithercoeff{}^2 I_{\controldim}}$, for some fixed constant $\dithercoeff{}$. Suppose that we aim to devote the time length $\episodetime{}$ to collect observations for learning to stabilize. Note that stabilization is performed before moving forward to the main objective of minimizing the cost function; accordingly, the stabilization time length $\episodetime{}$ is desired to be as short as possible. We divide this time interval of length $\episodetime{}$ to $\stabstep$ sub-intervals of equal length, and randomize an initial linear feedback policy by adding $\left\{\dither{n}\right\}_{n=0}^{\stabstep}$. That is, for $n=0,1, \cdots, \stabstep-1$, Algorithm~\ref{algo1} employs the control action
\begin{equation} \label{PerturbActionEq}
\action{t}=\Gainmat{}\state{t} + \dither{n}, \text{ ~~~~~~for~~~~~~~ }  \frac{n \episodetime{}}{\stabstep} \leq t < \frac{(n+1) \episodetime{}}{\stabstep},
\end{equation} 
where $\Gainmat{}$ is an initial stabilizing feedback matrix chosen such that all eigenvalues of $\Amat{0}+\Bmat{0}\Gainmat{}$ have strictly negative real-parts. In practice, such $\Gainmat{}$ can be determined using domain-specific knowledge of the physical system, for instance, through conservative control sequences for aircraft \citep{bosworth1992linearized,ishihara1992design}. It is important to note that while these stabilizing actions are necessary, they are sub-optimal and incur significant regret. Therefore, they are applied only temporarily for the purpose of initial data collection, serving as a stepping stone toward effectively balancing the exploration-exploitation trade-off while maintaining guaranteed stability.

The next step involves determining the posterior belief $\posterior{\episodetime{}}$ using data collected during the time interval $0\leq t \leq \episodetime{}$. Recalling that $\statetwo{t}^{\top}=\left[\state{t}^{\top},\action{t}^{\top}\right]$ denotes the observation vector at time~$t$, we define $\empmean{0}$ and $\empiricalcovmat{0}$ as the mean and precision matrix of a prior normal distribution on $\truth$. Such a prior belief can be objectively derived from previously available data or formed by subjective beliefs about the diffusion process under study, and in either case can accelerate the learning-based control of the system. If no such prior information exists, we simply set $\empmean{0}=0_{\left(\statedim+\controldim\right) \times \statedim}$ and $\empiricalcovmat{0}=I_{\statedim+\controldim}$.
Then, define 
\begin{eqnarray} 
\empiricalcovmat{\episodetime{}} = \empiricalcovmat{0}+ \itointeg{0}{\episodetime{}}{\statetwo{s} \statetwo{s}^{\top}}{s}, \label{RandomLSE1} ~~~~~~~~~~~
\empmean{\episodetime{}} = \empiricalcovmat{\episodetime{}}^{-1}  \left(\empiricalcovmat{0}\empmean{0}+\itointeg{0}{\episodetime{}}{ \statetwo{s} }{\state{s}^{\top}}\right). \label{RandomLSE2}
\end{eqnarray}
Using $\empiricalcovmat{\episodetime{}} \in \R^{\left( \statedim+\controldim\right) \times \left( \statedim+\controldim\right)}$ together with the mean matrix $\empmean{\episodetime{}}$, Algorithm~\ref{algo1} forms the posterior belief
\begin{equation} \label{RandomLSE3}
\posterior{\episodetime{}} = \normaldist{\empmean{\episodetime{}}}{\empiricalcovmat{\episodetime{}}^{-1}},
\end{equation}
about the drift parameter $\truth$. That is, the posterior distribution of every column $i=1, \cdots, \statedim$ of $\truth$, is an independent multivariate normal with the covariance matrix $\empiricalcovmat{\episodetime{}}^{-1}$, while the mean is the column $i$ of $\empmean{\episodetime{}}$. The final step of Algorithm~\ref{algo1} is to sample an output $\estpara{}$ from $\posterior{\episodetime{}}$.
\begin{algorithm}
	\caption{\bf: Stabilization under Uncertainty}  \label{algo1}
	\begin{algorithmic}
		\State{Inputs: initial feedback $\Gainmat{}$, stabilization time length $\episodetime{}$}
		
		\For{$n=0,1, \cdots, \stabstep-1 $}
			\While{$n \episodetime{}\stabstep^{-1} \leq t < (n+1) \episodetime{}\stabstep^{-1}$}
				\State{Apply control action $\action{t}$ in~\eqref{PerturbActionEq}}
			\EndWhile 
		\EndFor
		\State{Calculate $\empiricalcovmat{\episodetime{}}, \empmean{\episodetime{}}$ according to~\eqref{RandomLSE1}}
		\State{Return sample $\estpara{}$ from the distribution $\posterior{\episodetime{}}$ in~\eqref{RandomLSE3}}
	\end{algorithmic}
\end{algorithm}

Next, we proceed towards establishing performance guarantees for Algorithm~\ref{algo1}. For that purpose, recall that $\eigmax{M}$ denotes the largest eigenvalue of a matrix $M$, in magnitude, and let us quantify the \emph{ideal} stability
\begin{equation} \label{StabRadiiEq}
\stabradii=- \log \eigmax{\exp \left[ \Amat{0}-\Bmat{0} \Qu^{-1} \Bmat{0}^{\top} \RiccSol{\truth} \right]}.
\end{equation}

As indicated in the discussions following Definition~\ref{StabDef}, $\stabradii$ is strictly positive and its magnitude characterizes the stabilizability potential of the underlying system, with smaller values making stabilization more challenging. Specifically, $\stabradii$ represents the minimum distance in the complex plane between the imaginary axis and the eigenvalues of the transition matrix $\CLmat{0}=\Amat{0}-\Bmat{0} \Qu^{-1} \Bmat{0}^{\top} \RiccSol{\truth}$ under the optimal policy in~\eqref{OptimalPolicy}. Although $\stabradii$ remains unknown since $\truth$ is unavailable, it plays a crucial role in analyzing stabilization performance, as systems with larger $\stabradii$ allow for faster learning of stabilizing controllers. The precise impact of this quantity, along with other properties of the diffusion process, is formally established in the following result. Informally, the failure probability of Algorithm~\ref{algo1} decreases exponentially with the time length $\episodetime{}$ devoted to its execution.
\begin{thrm}[Stabilization Guarantee] \label{StabThm}
	For the sample $\estpara{}$ given by Algorithm~\ref{algo1}, let $\event{\episodetime{}}$ be the failure event that $\Amat{0} - \Bmat{0} \Qu^{-1} \estB{}^{\top} \RiccSol{\estpara{}}$ has an eigenvalue in the closed right half-plane, i.e., the system is unstable. Then, if $\log \left( \statedim \controldim \stabstep \right) \leq \episodetime{} \leq \stabstep^{1/2}$, we have
	\begin{equation} \label{StabFailProb}
	\log \PP{\event{\episodetime{}}} \lesssim ~~ - ~~ \frac{\eigmin{\BMcov{}} \wedge \dithercoeff{}^2 }{ \eigmax{\BMcov{}} } ~~ \frac{ \stabradii^2 \wedge \stabradii^{2\statedim} }{1 \vee \Mnorm{\Gainmat{}}{2}^2} ~~ {\frac{\episodetime{}}{\statedim^{3} \controldim}}.
	\end{equation}
\end{thrm}
Theorem~\ref{StabThm} offers several important insights into the performance of Algorithm~\ref{algo1}, as follows.

First, the ratio $\eigmax{\BMcov{}}/\eigmin{\BMcov{}}$ represents the disparity between maximum and minimum noise variances across different dimensions. When this ratio is large, the Wiener noise $\BM{t}$ exhibits strong directional bias—having much higher variance in some directions than others—which makes stabilization more difficult.

Second, higher-dimensional systems present greater challenges for learning to stabilize, due to both the increased number of parameters to be estimated and the heightened sensitivity of stability properties to parameter uncertainties in larger systems. This dimensional sensitivity is reflected in the dependence on $\statedim$ and $\controldim$ in the bound in \eqref{StabFailProb}.

Finally, and perhaps most significantly, the failure probability decays exponentially with the time horizon $\episodetime{}$—a result that stems from a new sub-Gaussian tail property we establish for self-normalized continuous-time martingales in Lemma~\ref{SelfNormalizedLem}. This exponential decay rate provides strong theoretical guarantees for the algorithm's stabilization capabilities when sufficient time is allocated.

\begin{remark}
    By~\eqref{StabRadiiEq}, the term $\stabradii^2 \wedge \stabradii^{2 \statedim}$ in \eqref{StabFailProb} reflects that less stable diffusion processes with very small margin $\stabradii <1$ are significantly harder to stabilize under uncertainty. However, we discuss in the proof of Theorem~\ref{StabThm} that such highly sensitive processes can still be stabilized fast by Algorithm~\ref{algo1}, for example if the matrix $\Amat{0}-\Bmat{0} \Qu^{-1} \Bmat{0}^{\top} \RiccSol{\truth}$ is diagonalizable (see Lemma~\ref{EigPerturbLem}).
\end{remark}

A notable feature of Algorithm~\ref{algo1} is its resilience—its failure probability can be reduced even further beyond the bound in \eqref{StabFailProb} through a simple mechanism. If a sample $\estpara{}$ from $\posterior{\episodetime{}}$ results in the failure event $\event{\episodetime{}}$, one can simply repeat the sampling process. Despite the unknown nature of the true drift matrices, failed stabilization is straightforward to detect from short state trajectories. Specifically, when the closed-loop transition matrix $\Amat{0} - \Bmat{0} \Qu^{-1} \estB{}^{\top} \RiccSol{\estpara{}}$, which governs the system dynamics when applying a policy based on $\estpara{}$, has an eigenvalue with non-negative real-part, the magnitude of at least one state variable will exhibit fast growth over time. Consequently, upon observing rapid growth in the state vector's magnitude, the controller can immediately repeat just the final sampling step of the algorithm without restarting the entire procedure. This resilience property allows Algorithm~\ref{algo1} to effectively learn to stabilize the unknown diffusion process with only a state trajectory of length $\episodetime{}$, making it particularly data-efficient in practice.

\subsection{Proof of Theorem \ref{StabThm}}
\label{sec:proof_StabThm}
The proof of this theorem relies on multiple intermediate lemmas. The statements and proofs of the lemmas are provided in the appendices and here we briefly express some results of the lemmas that are directly utilize to establish Theorem~\ref{StabThm}. 

Let us start by focusing on the error for recovering the true dynamics $\truth$ that Algorithm~\ref{algo1} incurs. We decompose the error of the output sample $\estpara{}$ of the algorithm into two components. The first component is the error of the posterior mean $\empmean{\episodetime{}}$ in \eqref{RandomLSE2} for estimating the truth, i.e., $\Mnorm{ \empmean{\episodetime{}} - \truth }{} $, while the second component consists of the additional randomness the sampling procedure imposes, i.e., $\Mnorm{ \estpara{}-\empmean{\episodetime{}} }{}$. In the sequel, we elaborate that both of the above-mentioned errors heavily depend on the precision matrix of the Gaussian distribution in \eqref{RandomLSE3} that is defined in \eqref{RandomLSE2}. 

	Accordingly, we proceed by bounding this matrix from below in Lemma~\ref{MinEigEmpCovLem}. That is, under the conditions $\episodetime{} \gtrsim \left( \statedim+\controldim \right) \log \delta^{-1}$ and $\stabstep \gtrsim \episodetime{} \controldim \log \left( \stabstep \controldim \delta^{-1} \right)$, with probability at least $1-\delta$ it holds that 
	\begin{eqnarray} \label{Lemma3resultEq}
    \eigmin{\empiricalcovmat{\episodetime{}}} \gtrsim \episodetime{} \left(\eigmin{\BMcov{}} \wedge \dithercoeff{}^2\right) \left( 1 + \Mnorm{\Gainmat{}}{2}^2 \right)^{-1}.
	\end{eqnarray}
	
	The next two steps that study the two components of the error $\estpara{} - \truth$ as described above are completely provided and proved in Lemma~\ref{ParaEstimationLem}. Here we discuss the main structure and delegate the fully detailed discussion to the appendix. Since the posterior distribution on the dynamics matrices $\para{}$ is Gaussian, the above inequality immediately leads to the following high-probability error bound for the sampling step of the algorithm:
	\begin{equation} \label{NewAuxIneq1}
		\Mnorm{ \estpara{} - \empmean{\episodetime{}} }{}^2 \lesssim \frac{ \left( 1 + \Mnorm{\Gainmat{}}{2}^2 \right) \statedim \left( \statedim + \controldim \right) }{ \episodetime{} \left(\eigmin{\BMcov{}} \wedge \dithercoeff{}^2\right) } \log \frac{ \statedim \left( \statedim + \controldim \right) }{ \delta }.
	\end{equation}
	
	On the other hand, we establish Lemma~\ref{SelfNormalizedLem} for studying the estimation error. 
	Recall that in \eqref{RandomLSE1}, the vector $\statetwo{t}=\left[\state{t}^\top,\action{t}^\top\right]^\top$ denotes the observation signal. Then, we have a high probability bound for the following stochastic matrix-valued integral:  
	\begin{equation*}
		\Mnorm{ \empiricalcovmat{t}^{-1/2} \itointeg{0}{t}{ \statetwo{s} }{\BM{s}^{\top}} }{}^2 \lesssim  \statedim \left( \statedim+\controldim \right) \eigmax{\BMcov{}} \log \left( \frac{\episodetime{}}{\delta} \right) .
	\end{equation*}
	
	Lemma~\ref{ParaEstimationLem} leverages the above result together with the result of Lemma~\ref{MinEigEmpCovLem} that we expressed above in \eqref{Lemma3resultEq}, in order to show that with probability at least $1-\delta$, we have
	\begin{equation} \label{NewAuxIneq2}
		\Mnorm{\empmean{\episodetime{}} - \truth}{2}^2 \lesssim \frac{ \left( 1 + \Mnorm{\Gainmat{}}{2}^2 \right) \statedim \left( \statedim + \controldim \right) \eigmax{\BMcov{}} }{ \episodetime{} \left(\eigmin{\BMcov{}} \wedge \dithercoeff{}^2\right) } \log \frac{ \episodetime{} }{ \delta } .
	\end{equation}
	Therefore, putting together \eqref{NewAuxIneq1} and \eqref{NewAuxIneq2}, we obtain the following result that is proven in Lemma~\ref{ParaEstimationLem}:
	\begin{equation*} 
		\Mnorm{\estpara{}-\truth}{2}^2 \lesssim \frac{ \statedim \left(\statedim+\controldim\right) }{\episodetime{} } \frac{\eigmax{\BMcov{}}  }{\eigmin{\BMcov{}} \wedge \dithercoeff{}^2} \left( 1 + \Mnorm{\Gainmat{}}{2}^2 \right) \log \left( \frac{\statedim \controldim \episodetime{}}{\delta} \right).
	\end{equation*}
	
	Next, we use the above bound together with Lemma~\ref{LipschitzLemma} to establish a similar inequality for the solutions of \eqref{ARiccEq}. Intuitively speaking, Lemma~\ref{LipschitzLemma} provides a sensitivity analysis for \eqref{ARiccEq} by establishing that the solution matrix $\RiccSol{\para{}}$ is a locally Lipschitz function of the dynamics matrix $\para{}$, and provides its Lipschitz constant in terms of the relevant parameters. So, Lemma~\ref{ParaEstimationLem} and Lemma~\ref{LipschitzLemma} imply that with probability at least $1-\delta$, it holds that
	\begin{equation*}
	\Mnorm{\RiccSol{\estpara{}}-\RiccSol{\truth}}{2}^2 \lesssim  \frac{ \statedim \left(\statedim+\controldim\right) }{\episodetime{} } \frac{\eigmax{\BMcov{}}  }{\eigmin{\BMcov{}} \wedge \dithercoeff{}^2} \left( 1 + \Mnorm{\Gainmat{}}{2}^2 \right) \log \left( \frac{\statedim \controldim \episodetime{}}{\delta} \right).
	\end{equation*}
	
	Note that the right-hand-side above is the bound for both $\Mnorm{\estpara{}-\truth}{2}^2$ as well as $\Mnorm{\RiccSol{\estpara{}}-\RiccSol{\truth}}{2}$. Thus, we get the same upper bound for $\Mnorm{\estB{}^\top \RiccSol{\estpara{}} - \Bmat{0}^\top \RiccSol{\truth}}{2}$. Furthermore, remember that Algorithm~\ref{algo1} uses the sample $\estpara{}$ to employ  the linear feedback $\action{t} = -\Qu^{-1} \estB{}^{\top} \RiccSol{\estpara{}} \state{t}$ for stabilizing the unknown diffusion process in \eqref{dynamics}. Hence, we get a similar error bound for the difference between the ideal closed-loop matrix $\CLmat{0}=\Amat{0}-\Bmat{0} \Qu^{-1} \Bmat{0}^{\top} \RiccSol{\truth}$ under full knowledge of the truth, against the actual closed-loop matrix $\CLmat{}=\Amat{0}-\Bmat{0} \Qu^{-1} \estB{}^{\top} \RiccSol{\estpara{}}$ that governs the dynamics after running Algorithm~\ref{algo1}. That is,  
	\begin{equation} \label{StabThmProofEq1}
	\Mnorm{\CLmat{}-\CLmat{0}}{2}^2 \lesssim  \frac{ \statedim \left(\statedim+\controldim\right) }{\episodetime{} } \frac{\eigmax{\BMcov{}}  }{\eigmin{\BMcov{}} \wedge \dithercoeff{}^2} \left( 1 + \Mnorm{\Gainmat{}}{2}^2 \right) \log \left( \frac{\statedim \controldim \episodetime{}}{\delta} \right),
	\end{equation}
	with probability at least $1-\delta$. 
	
	To proceed, we consider the effect of the above errors on the eigenvalues of $\CLmat{}$ in comparison to that of $\CLmat{0}$. For that purpose, we establish Lemma~\ref{EigPerturbLem} that proves a bound for the sensitivity of eigenvalues of matrices to the perturbation in their entries. Lemma~\ref{EigPerturbLem} expresses that the difference between the leading eigenvalues of $\CLmat{}$ and $\CLmat{0}$ is (apart from a constant factor) at most $\left( 1 \vee \mult{} \Mnorm{\CLmat{}-\CLmat{0} }{2}  \right)^{1/\mult{}}$, where $\mult{}$ is the size of biggest block in the Jordan decomposition of $\CLmat{0}$. The general statement of the lemma containing all exact definitions and full technical details, together with the proof of the lemma, are provided in the appendix. For the sake of simplicity, in the remainder of this proof, we use the fact $\mult{} \leq \statedim$. However, in cases that the difference between $\statedim$ and $\mult{}$ is not small, one can simply get a tighter analysis by replacing the exponent $\statedim$ with $\mult{}$. For example, when the dimension of the process $\statedim$ is not small, while $\CLmat{0}$ is a diagonal matrix and so satisfies $\mult{}=1$. 
	
	Success of the stabilization Algorithm~\ref{algo1} means that all eigenvalues of $\CLmat{}$ are on the open left half-plane of the complex plane. On the other hand, according to the definition of the largest real-part $-\stabradii$ in \eqref{StabRadiiEq}, real-parts of all eigenvalues of $\CLmat{0}$ are at most $-\stabradii$. Therefore, the eigenvalue sensitivity analysis of the previous paragraph indicates that the failure event $\event{\episodetime{}}$ does not occur, as long as
	\begin{equation*}
		\left( 1 \vee \mult{} \Mnorm{\CLmat{}-\CLmat{0} }{2}  \right)^{1/\mult{}} \lesssim \stabradii.
	\end{equation*}
	 That is,  
	 \begin{equation*}
	 	\PP{ \event{\episodetime{}} } \lesssim \PP{ \Mnorm{\CLmat{}-\CLmat{0}}{2} \gtrsim \frac{ \stabradii \wedge \stabradii^{\statedim}  }{\statedim} }.
	 \end{equation*}
	 Finally, the above inequality and \eqref{StabThmProofEq1} yield to the desired result in \eqref{StabFailProb}.
~\hfill~$\blacksquare$
\endproof

\section{Efficiency and Regret Analysis} \label{CostMinSection}
In this section, we analyze Thompson sampling (TS) for minimizing the quadratic cost in \eqref{AveCostDef}, demonstrating that TS efficiently learns optimal control actions through effective management of the trade-off between the exploration and exploitation. This approach achieves a regret that grows with (nearly) the square-root rate, as time grows. Below, we first introduce Algorithm~\ref{algo2} and examine its underlying framework and technical mechanisms. Subsequently, we establish theoretical guarantees by deriving explicit regret bounds parameterized by system characteristics and provide convergence rates for estimating unknown drift matrices. These results contribute to the broader literature on stochastic control by providing performance guarantees for adaptive learning in continuous-time systems.

To formalize the exploration-exploitation dilemma in our diffusion control setting, note that any policy aiming to achieve sub-linear regret must take near-optimal control actions in the long run; that is, 
\begin{equation} \label{NearOptActionEq}
    \action{t} \approx - \Qu^{-1} \Bmat{0}^\top \RiccSol{\truth} \state{t}.
\end{equation}
However, such exploitation-focused policies, despite approximating the optimal policy $\optimalpolicy$, suffer from insufficient information acquisition. From a learning perspective, the resulting observation trajectory $\left\{\statetwo{t}\right\}_{t \geq 0}$ lacks the necessary diversity for accurate parameter estimation. This occurs because in the observation vector $\statetwo{t}^\top=\left[\state{t}^\top,\action{t}^\top\right]$, the control-action component $\action{t}$ becomes approximately collinear with the state signal $\state{t}$ as shown in \eqref{NearOptActionEq}, thereby providing minimal new information about the unknown parameter $\truth$. Conversely, policies that prioritize exploration by deliberately deviating from $\optimalpolicy$ generate more informative data but incur substantial immediate regret. This fundamental trade-off between information acquisition and optimality must be precisely calibrated—a balance that we demonstrate TS achieves.

\begin{algorithm}
	\caption{\bf: Thompson Sampling for Control Policy} \label{algo2}
	\begin{algorithmic}
		\State{Inputs: stabilization time $\episodetime{0}$}
		\State{Calculate sample $\estpara{0}$ by running Algorithm~\ref{algo1} for time $\episodetime{0}$}
		
		\For{$n=1,2, \cdots$}
		\While {$\episodetime{n-1} \leq t < \episodetime{n}$}
		\State{Apply control action $\action{t}=-\Qu^{-1} \estB{n-1}^{\top} \RiccSol{\estpara{n-1}} \state{t}$}
		\EndWhile	
		\State{Letting $\empiricalcovmat{\episodetime{n}},\empmean{\episodetime{n}}$ be as \eqref{RandomLSE1}, sample $\estpara{n}$ from $\posterior{\episodetime{n}}$ given in~\eqref{RandomLSE3}}
		\EndFor
	\end{algorithmic}
\end{algorithm}

Algorithm~\ref{algo2} begins by executing the learning-based stabilization procedure (Algorithm~\ref{algo1}) during the interval $0 \leq t <\episodetime{0}$. As established by Theorem~\ref{StabThm}, when $\episodetime{0}$ is sufficiently large, the optimal feedback policy derived from $\estpara{0}$ stabilizes the system with high probability. If state vector trajectories indicate instability, one can implement repeated sampling from the posterior distribution $\posterior{\episodetime{0}}$ in Algorithm~\ref{algo1} until stability is achieved. Therefore, we can operate under the assumption that the controlled diffusion process maintains stability during the execution of Algorithm~\ref{algo2}. This initialization phase serves a dual purpose: beyond establishing system stability, it also conducts an initial exploration of the state space, generating informative data that Algorithm~\ref{algo2} leverages to construct policies that minimize cumulative regret. This two-phase approach represents an efficient method for addressing the joint challenges of stabilization and optimal decision-making under uncertainty.

The next step of Algorithm~\ref{algo2} is episodic; the parameter estimates $\estpara{n}$ are updated only at discrete episode boundaries $\left\{ \episodetime{n} \right\}_{n=0}^\infty$. During each episode, the controller implements control actions \emph{as if} the sample $\estpara{n}=\left[ \estA{n},\estB{n} \right]^\top$ were the true system parameter $\truth$. Specifically, for all $t \in [\episodetime{n-1}, \episodetime{n})$, using the solution $\RiccSol{\estpara{n}}$ to the algebraic Riccati equation \eqref{ARiccEq}, the control action is given by $\action{t}=-\Qu^{-1} \estB{n}^\top \RiccSol{\estpara{n}} \state{t}$. For each $n =1, 2, \cdots$, at time $\episodetime{n}$, we then use the complete observation history to compute the sufficient statistics $\empiricalcovmat{\episodetime{n}}$ and $\empmean{\episodetime{n}}$ according to \eqref{RandomLSE1}, which are used to construct the posterior distribution $\posterior{\episodetime{n}}$ as defined in \eqref{RandomLSE3}. The parameter estimate for the subsequent episode is obtained by drawing a random sample from this new posterior distribution.

The selection of episode lengths is a design parameter that impacts computational efficiency. Specifically, the episodes in Algorithm~\ref{algo2} are structured to satisfy the following growth condition:
\begin{equation} \label{EpochLengthCond}
0 < \underline{\ssconstant} \leq  \inf_{n \geq 0} \frac{\episodetime{n+1}-\episodetime{n}}{\episodetime{n}} \leq \sup_{n \geq 0} \frac{\episodetime{n+1}-\episodetime{n}}{\episodetime{n}} \leq \overline{\ssconstant} < \infty,
\end{equation}
where the constants $\underline{\ssconstant}$ and $\overline{\ssconstant}$ establish lower and upper bounds on the relative growth rate of consecutive episodes. This balanced growth requirement ensures that episodes are neither too short (which would incur excessive computational overhead from frequent parameter updates) nor too long (which would delay the incorporation of new information into the control policy). This is especially appropriate for continuous-time systems where posterior updates must be sufficiently separated temporally. Frequent updates impose prohibitive computational burden. Moreover, they can degrade system performance by interrupting the natural response time of control actions, preventing policies from meaningfully influencing system dynamics before being replaced.

To understand the balance in episode design, observe that by \eqref{RandomLSE1}, the precision matrix $\empiricalcovmat{\episodetime{}}$ grows with $\episodetime{}$, causing the estimation error $\estpara{n}-\truth$ to decay (at best with a polynomial rate) with $\episodetime{n}$. Consequently, updating the posterior distribution before sufficient new information has accumulated would incur computational costs (posterior sampling and solving \eqref{ARiccEq}) without proportional improvement in control performance. This statistical-computational trade-off motivates setting the episode length $\episodetime{n+1} - \episodetime{n}$ to be at least $\underline{\ssconstant} \episodetime{n}$. Conversely, episodes cannot be excessively long, as this would delay the incorporation of new observations into the parameter estimates; hence the upper bound $\episodetime{n+1} \leq \left(1 + \overline{\ssconstant}\right) \episodetime{n}$. A particularly elegant implementation is achieved by setting $\underline{\ssconstant}=\overline{\ssconstant}$, resulting in geometrically increasing episode lengths $\episodetime{n}=\episodetime{0} \left( 1+\overline{\ssconstant} \right)^n$. 

We now demonstrate that Algorithm~\ref{algo2} efficiently resolves the exploration-exploitation dilemma through an adaptive sampling mechanism. The key insight lies in examining the sequence of posterior distributions $\posterior{\episodetime{n}}$. The exploration intensity induced by sampling $\estpara{n}$ from $\posterior{\episodetime{n}}$ is governed by matrix $\empiricalcovmat{\episodetime{n}}$. When $\eigmin{\empiricalcovmat{\episodetime{n}}}$ is not large enough, the posterior $\posterior{\episodetime{n}}$ exhibits higher variance around its mean $\empmean{\episodetime{n}}$, causing $\estpara{n}$ to likely differ substantially from previous estimates $\left\{\estpara{i}\right\}_{i=1}^{n-1}$. This posterior-induced randomization creates a self-correcting exploration mechanism: the controller implements more diverse control signals $\action{t}$, generating more informative observation data~$\statetwo{t}$. This information-rich data yields two benefits: (1) the subsequent posterior mean $\empmean{\episodetime{n+1}}$ provides a more accurate approximation of $\truth$, and (2) the minimum eigenvalue $\eigmin{\empiricalcovmat{\episodetime{n+1}}}$ increases at an accelerated rate. Consequently, the next posterior $\posterior{\episodetime{n+1}}$ produces samples with reduced estimation error $\|\estpara{n+1}-\truth\|$. Similarly, when a posterior becomes overly concentrated (indicating excessive exploitation), the adaptive mechanism readjusts within a few episodes to restore appropriate exploration levels. This inherent statistical adaptation ensures that TS eventually balances the exploration-exploitation trade-off without explicit parameter tuning. This is formalized below.

\begin{thrm} [Regret and Estimation Rates] \label{RegretThm}
	Parameter estimates and regret of Algorithm~\ref{algo2}, satisfy the following bounds:
	\begin{eqnarray*}
		\Mnorm{\estpara{n}-\truth}{2}^2 &\lesssim&  \frac{\eigmax{\BMcov{}}}{\eigmin{\BMcov{}}} ~~\left( \statedim+\controldim \right) \statedim ~~~\episodetime{n}^{-1/2} \log \episodetime{n}\,,\\
		\regret{T}{} &\lesssim& \left(\eigmax{\BMcov{}} + \dithercoeff{}^2\right) \episodetime{0}+ \frac{\eigmax{\BMcov{}}^2}{\eigmin{\BMcov{}}} \frac{ \Mnorm{\RiccSol{\truth}}{2}^6 }{ \eigmin{Q}^6} ~~~ \left( \statedim+\controldim \right) \statedim ~~~T^{1/2} \log T\,.
	\end{eqnarray*}
\end{thrm}
The regret and estimation rates established in Theorem \ref{RegretThm} reveal several key dependencies. Similar to Theorem \ref{StabThm}, the ratio $\eigmax{\BMcov{}}/\eigmin{\BMcov{}}$ captures how the disparity between maximum and minimum noise variances across different dimensions of the Wiener noise $\BM{t}$ affects learning quality. The factor $\statedim (\statedim+\controldim)$ demonstrates a linear scaling relationship with both estimation error and regret, confirming that larger system dimensions—and consequently more parameters to learn—directly worsen performance. Within the regret bound, the term $\Mnorm{\RiccSol{\truth}}{2}/\eigmin{Q}$ characterizes the effect of the true system parameters $\truth$ and cost matrix $Q$, highlighting how intrinsic system properties influence achievable performance. Finally, the term $\left(\eigmax{\BMcov{}} + \dithercoeff{}^2\right) \episodetime{0}$ quantifies the cost associated with the initial stabilization phase where Algorithm~\ref{algo1} implements the necessarily sub-optimal control actions defined in~\eqref{PerturbActionEq}.

\subsection{Proof of Theorem \ref{RegretThm}.}
\label{sec:proog_RegretThm}
For this proof, we utilize multiple intermediate lemmas whose fully rigorous statements and proofs are delegated to the appendices. 

We proceed to establish the rates at which Algorithm~\ref{algo2} learns the unknown parameters. The estimation error $\estpara{n} - \truth$ can be decomposed into the portion $\empmean{\episodetime{n}} - \truth$ that is caused by the Wiener noise $\BM{t}$ in \eqref{dynamics}, and the portion originating from the randomization that TS employs for ensuring sufficient exploration; $\estpara{n} - \empmean{\episodetime{n}}$. Write the data generation mechanism in \eqref{dynamics} as $\diff \state{t} = \truth^\top \statetwo{t} \diff t + \diff \BM{t}$, and plug in the definition of $\empmean{\episodetime{n}}$ in \eqref{RandomLSE2} to obtain  
\begin{equation*}
	\empmean{\episodetime{n}} = \empiricalcovmat{\episodetime{n}}^{-1} \itointeg{0}{\episodetime{n}}{ \left( \statetwo{t} \statetwo{t}^\top \truth \right) }{t} + \empiricalcovmat{\episodetime{n}}^{-1} \itointeg{0}{\episodetime{n}}{\statetwo{t}}{\BM{t}^\top} = \truth + \empiricalcovmat{\episodetime{n}}^{-1} \left( \itointeg{0}{\episodetime{n}}{\statetwo{t}}{\BM{t}^\top} - \truth \right).
\end{equation*}
That is,
\begin{equation*}
	\Mnorm{\empmean{\episodetime{n}}-\truth}{2} \lesssim \Mnorm{\empiricalcovmat{\episodetime{n}}^{-1} \itointeg{0}{\episodetime{n}}{\statetwo{t}}{\BM{t}^\top} }{2}
	\leq \Mnorm{\empiricalcovmat{\episodetime{n}}^{-1/2}}{2} \Mnorm{\empiricalcovmat{\episodetime{n}}^{-1/2} \itointeg{0}{\episodetime{n}}{\statetwo{t}}{\BM{t}^\top} }{2} .
\end{equation*}
Lemma~\ref{SelfNormalizedLem} establishes an upper-bound for the matrix-valued stochastic integral $ \itointeg{0}{\episodetime{n}}{\statetwo{t}}{\BM{t}^\top}$ based on the dimensions $\statedim,\controldim$, the noise covariance matrix $\BMcov{}$, and the largest eigenvalue of the precision matrix $\empiricalcovmat{\episodetime{n}}$ in \eqref{RandomLSE1}. The result in Lemma~\ref{SelfNormalizedLem} indicates that 
\begin{equation} \label{NewRegretProofEq1}
\Mnorm{\empiricalcovmat{\episodetime{n}}^{-1/2} \itointeg{0}{\episodetime{n}}{\statetwo{t}}{\BM{t}^\top} }{2}^2 \lesssim ~~\statedim (\statedim+\controldim) \eigmax{\BMcov{}} \log \eigmax{\empiricalcovmat{\episodetime{n}}}.
\end{equation}

On the other hand, since by the design of TS in Algorithm~\ref{algo2} the $\statedim+\controldim$ by $\statedim$ matrix $\empiricalcovmat{\episodetime{n}}^{1/2} \left( \estpara{n}-\empmean{\episodetime{n}} \right)$ has a standard normal distribution (i.e., independent columns, as defined in the notation), we have 
\begin{equation} \label{NewRegretProofEq2}
	\Mnorm{\estpara{n}-\empmean{\episodetime{n}}}{2}^2 \lesssim \eigmin{ \empiricalcovmat{\episodetime{n}} }^{-1} ~\statedim (\statedim+\controldim) \log (\statedim \controldim).
\end{equation}

In Lemma~\ref{MinPELem}, we quantify the exploration TS policy in Algorithm~\ref{algo2} performs by establishing the following lower-bound on the precision matrix:
\begin{equation} \label{NewRegretProofEq3}
	\eigmin{\empiricalcovmat{\episodetime{n}}} \gtrsim \episodetime{n}^{1/2} \eigmin{\BMcov{}}.
\end{equation} 
The proof of Lemma~\ref{MinPELem} that is provided in the appendices utilizes Lemma~\ref{OptManifoldLemma}, which specifies the manifold of $ \left( \statedim+\controldim \right) \times \statedim$ matrices that share an optimal linear feedback matrix. Furthermore, it is shown in the proof of Lemma~\ref{MinPELem} (specifically, the inequality in \eqref{StateEmpCovBound2}) that for the largest eigenvalue of $\empiricalcovmat{\episodetime{n}}$ we have a linear growth rate with time, i.e., $\log \eigmax{\empiricalcovmat{\episodetime{n}}} \lesssim \log \episodetime{n}$. Therefore, by putting \eqref{NewRegretProofEq1}, \eqref{NewRegretProofEq2}, and \eqref{NewRegretProofEq3} together, we obtain
\begin{equation*}
	\Mnorm{ \estpara{n}-\truth}{2}^2 \lesssim ~~ \frac{\eigmax{\BMcov{}}}{\eigmin{\BMcov{}}} \left( \statedim+\controldim \right) \statedim ~\episodetime{n}^{-1/2} \log \episodetime{n},
\end{equation*} 
which is the desired result in the first statement of the theorem. 

To proceed towards establishing the regret bound, Lemma~\ref{GeneralRegretLem} shows that we need to integrate $\Mnorm{\action{t}+\Qu^{-1} \estB{n}^\top \RiccSol{\estpara{n}}\state{t}}{2}^2$ over the period $\episodetime{0} \leq t \leq T$: 
\begin{equation*}
	\regret{T}{} \lesssim \left(\eigmax{\BMcov{}} + \dithercoeff{}^2\right) \episodetime{0} 
	+ \itointeg{\episodetime{0}}{T}{ \norm{ \action{t} + \Qu^{-1} \Bmat{0}^{\top} \RiccSol{\truth} \state{t} }{}^2 }{t}.
\end{equation*}
Note that during the time period $\episodetime{n-1} < t \leq \episodetime{n}$ the control action taken by Algorithm~\ref{algo2} is $\action{t}= \Optgain{\estpara{n-1}} \state{t}$, for the linear feedback matrix $\Optgain{\estpara{n-1}} = -\Qu^{-1} \estB{n-1}^{\top} \RiccSol{\estpara{n-1}}$. Further, denote the optimal feedback matrix as $\Optgain{\truth} = -\Qu^{-1} \Bmat{0}^{\top} \RiccSol{\truth}$. Since the algorithm is episodic in the sense that the feedback matrices $\Optgain{\estpara{n}}$ remain unaltered during each episode, we can break the above integral into a sum of integrals over the episodes. To that end, let $n$ be the number of episodes by time $T$; that is, $\episodetime{n-1} \leq T < \episodetime{n}$. Then, 
\begin{equation*}
	\regret{T}{} \lesssim \left(\eigmax{\BMcov{}} + \dithercoeff{}^2\right) \episodetime{0} 
	+ \sum\limits_{i=0}^{n-1} \itointeg{\episodetime{i}}{\episodetime{i+1}}{ \norm{ \left( \Optgain{\estpara{i}} - \Optgain{\truth} \right) \state{t} }{}^2 }{t}.
\end{equation*}

Next, we bound the integrals above using the circular property of the trace of matrices:
\begin{equation*}
	\itointeg{\episodetime{i}}{\episodetime{i+1}}{ \norm{ \left( \Optgain{\estpara{i}} - \Optgain{\truth} \right) \state{t} }{}^2 }{t} = \tr{ \left( \Optgain{\estpara{i}} - \Optgain{\truth} \right)^{\top} \left( \Optgain{\estpara{i}} - \Optgain{\truth} \right)  \itointeg{\episodetime{i}}{\episodetime{i+1}}{  \state{t} \state{t}^\top  }{t} }.
\end{equation*}
Now, we use the fact that the stability of the dynamics in Algorithm~\ref{algo2} renders the growth of the matrix $\itointeg{\episodetime{i}}{\episodetime{i+1}}{  \state{t} \state{t}^\top  }{t}$ linear with time. Specifically, the bound in \eqref{StateEmpCovBound2} in the proof of Lemma~\ref{GeneralRegretLem} yields to 
$$\eigmax{\itointeg{\episodetime{i}}{\episodetime{i+1}}{  \state{t} \state{t}^\top  }{t}} \lesssim \left( \episodetime{i+1} - \episodetime{i} \right) \eigmax{\BMcov{}} .$$ 

Thus, the regret bound becomes
\begin{equation*}
\regret{T}{} \lesssim \left(\eigmax{\BMcov{}} + \dithercoeff{}^2\right) \episodetime{0} 
+ \eigmax{\BMcov{}} \sum\limits_{i=0}^{n-1} \left( \episodetime{i+1} - \episodetime{i} \right) \Mnorm{\Optgain{\estpara{i}}-\Optgain{\truth}}{2}^2.
\end{equation*}

We continue the proof by finding an upper-bound for $\Mnorm{\Optgain{\estpara{i}}-\Optgain{\truth}}{2}$. For that purpose, we leverage the Lipschitz continuity of $\Optgain{\para{}}$ that is proven in Lemma~\ref{LipschitzLemma}, to obtain 
\begin{equation*}
	\Mnorm{\Optgain{\estpara{i}} - \Optgain{\truth}}{2} \lesssim \frac{\Mnorm{\RiccSol{\truth}}{2}^3}{\eigmin{\Qx} \eigmin{\Qu}^2} \Mnorm{\estpara{i}-\truth}{2}.
\end{equation*}
Thus, the first result of the theorem on the decay rate of $\estpara{i} - \truth$ leads to 
\begin{equation*}
\regret{T}{} \lesssim \left(\eigmax{\BMcov{}} + \dithercoeff{}^2\right) \episodetime{0} 
+ \frac{\eigmax{\BMcov{}}^2}{\eigmin{\BMcov{}}} \frac{\Mnorm{\RiccSol{\truth}}{2}^6}{\eigmin{\Qx}^2 \eigmin{\Qu}^4} \statedim (\statedim+\controldim) \sum\limits_{i=0}^{n-1} \left( \episodetime{i+1} - \episodetime{i} \right) \frac{\log \episodetime{i}}{\episodetime{i}^{1/2}}.
\end{equation*}
Finally, since according to the condition \eqref{EpochLengthCond} for the lengths of the episodes we have $\sum\limits_{i=0}^{n-1} \left( \episodetime{i+1} - \episodetime{i} \right) \episodetime{i}^{-1/2} \log \episodetime{i} \lesssim T^{1/2} \log T$, the above upper-bound for the regret leads to the desired regret bound of Theorem~\ref{RegretThm}. ~\hfill~$\blacksquare$
\endproof

\section{Numerical Experiments} \label{NumericalSection}
We empirically evaluate our theoretical results and proposed method across three different real-world systems: blood-glucose control~\citep{zhou2018autoregressive,gondhalekar2016periodic}, flight control of X-29A airplane at 2000~ft altitude \citep{bosworth1992linearized}, and flight control of Boeing 747 airplanes at 20000 ft altitude~\citep{ishihara1992design}.

For implementation, we parameterize our framework using the dynamics in \eqref{dynamics} and the cost function in \eqref{AveCostDef} with the specific drift matrices appropriate for each system. We select $\BMcov{} =0.25~I_\statedim$, $\Qx =I_\statedim$, and $\Qu =0.1~I_\controldim$ where $I_n$ is the $n$ by $n$ identity matrix. To numerically simulate the diffusion process $\state{t}$ in~\eqref{dynamics}, we employ time-steps of length $10^{-3}$, ensuring computational accuracy. For Algorithm~\ref{algo1}, we set $\dithercoeff{}=5$ and $\stabstep = \lfloor \episodetime{}^{3/2}\rfloor$, with $\episodetime{}$ varied across different systems. The initial feedback matrix $\Gainmat{}$ in \eqref{PerturbActionEq} is generated randomly.

We benchmark Algorithm~\ref{algo2} against the Randomized Estimate policy~\citep{faradonbeh2023online} across all systems. For Algorithm~\ref{algo1}, we conduct $1,000$ independent replications. For Algorithm~\ref{algo2}, we implement a geometric progression of episode lengths with $\episodetime{n} = 20\times 1.1^{n}$ and perform $100$ replications, allowing us to characterize both average-case and worst-case performance—critical considerations for operational reliability.

\subsubsection*{Blood-glucose control.}
Blood-glucose control system is characterized by the following drift matrices, with Algorithm~\ref{algo1} applied for $\episodetime{}$ varying from 10 to 45 and Algorithm~\ref{algo2} for $50 \leq T \leq 600$:
\begin{align*}
	\Amat{0} = \begin{bmatrix}
		1.91\;\;\;    &-2.82\;\;\;		&0.91
		\\
		1.00\;\;\;    &-1.00\;\;\;		&0.00
		\\
		0.00\;\;\;    &1.00\;\;\;		&-1.00
	\end{bmatrix}, ~~~~
	\Bmat{0} = \begin{bmatrix}
		-0.0992	\\
		0.0000
		\\
		0.0000
	\end{bmatrix}.
\end{align*}
\begin{figure}[H]
    \centering
    \begin{minipage}[c]{0.45\linewidth}
        \centering
        \includegraphics[width=\linewidth]{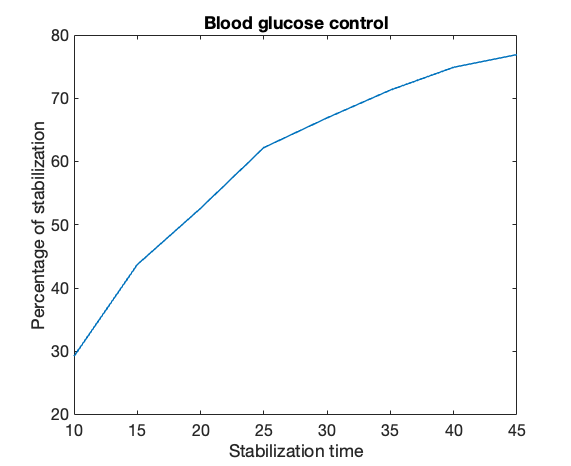}
    \end{minipage}%
    \hfill
    \begin{minipage}[c]{0.5\linewidth}
        \centering
        \includegraphics[width=\linewidth]{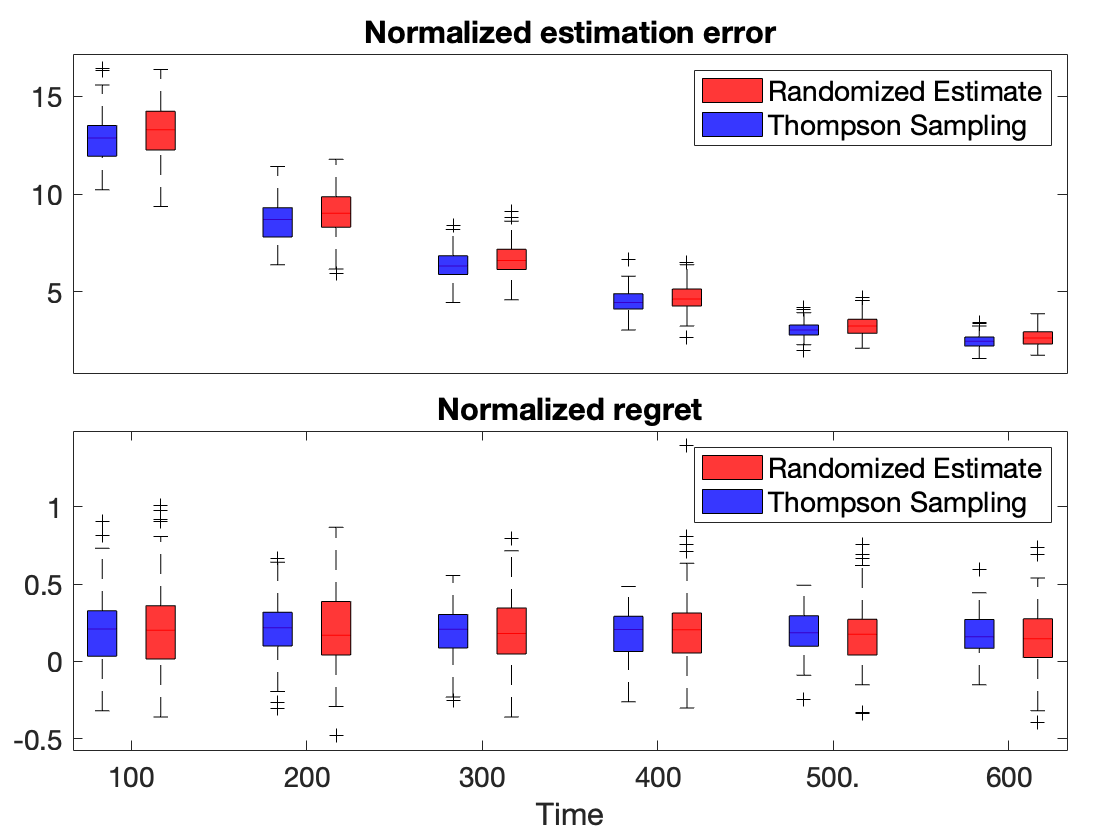}
    \end{minipage}
    \centering\includegraphics[width=.60\linewidth]{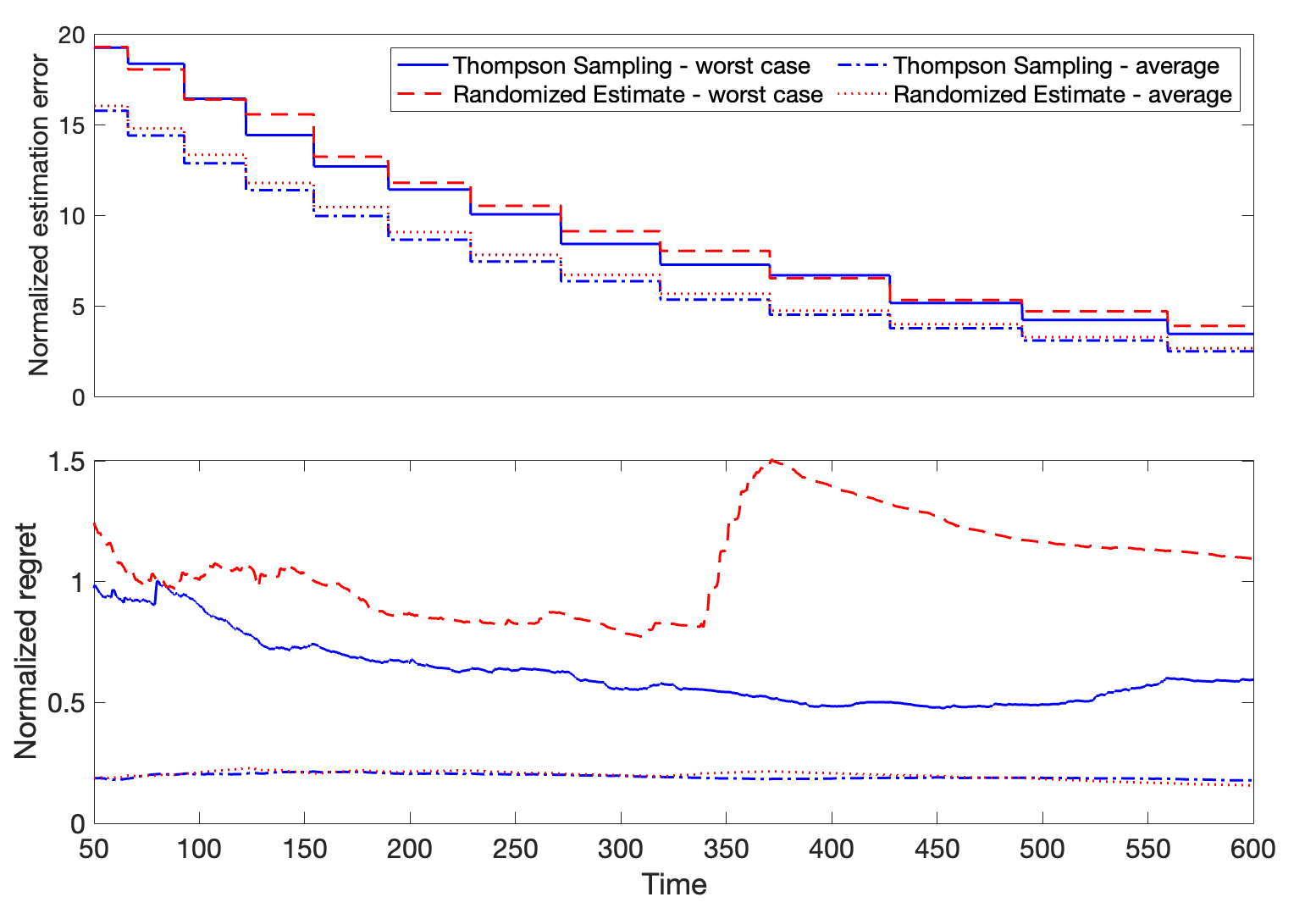}
    \caption{Analysis of blood glucose control system. \textbf{Left-top:} Percentage of successful system stabilization across 1000 runs of Algorithm~\ref{algo1}. \textbf{Right-top:} Comparative performance between Algorithm~\ref{algo2} (\textcolor{blue}{blue}) and Randomized Estimate policy (\textcolor{red}{red}). Upper panel shows normalized squared estimation error ($\Mnorm{\estpara{n}-\truth}{}^2$ divided by $\statedim (\statedim+\controldim) \episodetime{n}^{-1/2} \log \episodetime{n}$) at time points 100, 200, 300, 400, 500, and 600 seconds across 100 replications. Lower panel displays normalized regret ($\regret{T}{}$ divided by $\statedim (\statedim+\controldim) T^{1/2} \log T$). \textbf{Bottom:} Extended performance comparison between Algorithm~\ref{algo2} (\textcolor{blue}{blue}) and Randomized Estimate policy (\textcolor{red}{red}), with normalized squared estimation error (top) and normalized regret (bottom) for 100 replications.}
    \label{fig:BGC}
\end{figure}

\begin{samepage}
\subsubsection*{X-29A Airplane Control.}
The X-29A represents an advanced experimental aircraft system. The system is characterized by the following drift matrices, with Algorithm~\ref{algo1} applied for $\episodetime{}$ varying from 4 to 20 and Algorithm~\ref{algo2} for an operational horizon of $50 \leq T \leq 600$:
\begin{equation*}
	\Amat{0} = \begin{bmatrix}
		-0.16\;\;\;    & 0.07\;\;\;		& -1.00\;\;\;    &0.04
		\\
		-15.20\;\;\;    & -2.60\;\;\;		&1.11\;\;\;    &0.00
		\\
		6.84\;\;\;    &-0.10\;\;\;		&-0.06\;\;\;    &0.00
		\\
		0.00\;\;\;    &1.00\;\;\;		&0.07\;\;\;    &0.00
	\end{bmatrix}, ~~~~
	\Bmat{0} = \begin{bmatrix}
		-0.0006\;\;\;    &0.0007	\\
		1.3430 \;\;\;    &0.2345
		\\
		0.0897\;\;\;    &-0.0710
		\\
		0.0000\;\;\;    &0.0000
	\end{bmatrix}.
\end{equation*}%
\begin{figure}[ht]
    \centering
    \begin{minipage}[c]{0.45\linewidth}
        \centering
        \includegraphics[width=\linewidth]{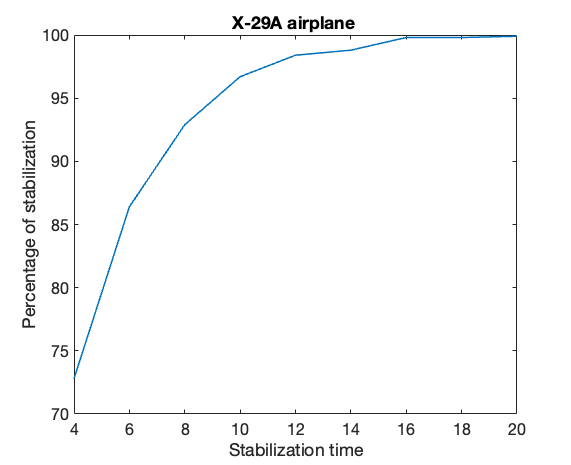}
    \end{minipage}%
    \hfill
    \begin{minipage}[c]{0.5\linewidth}
        \centering
        \includegraphics[width=\linewidth]{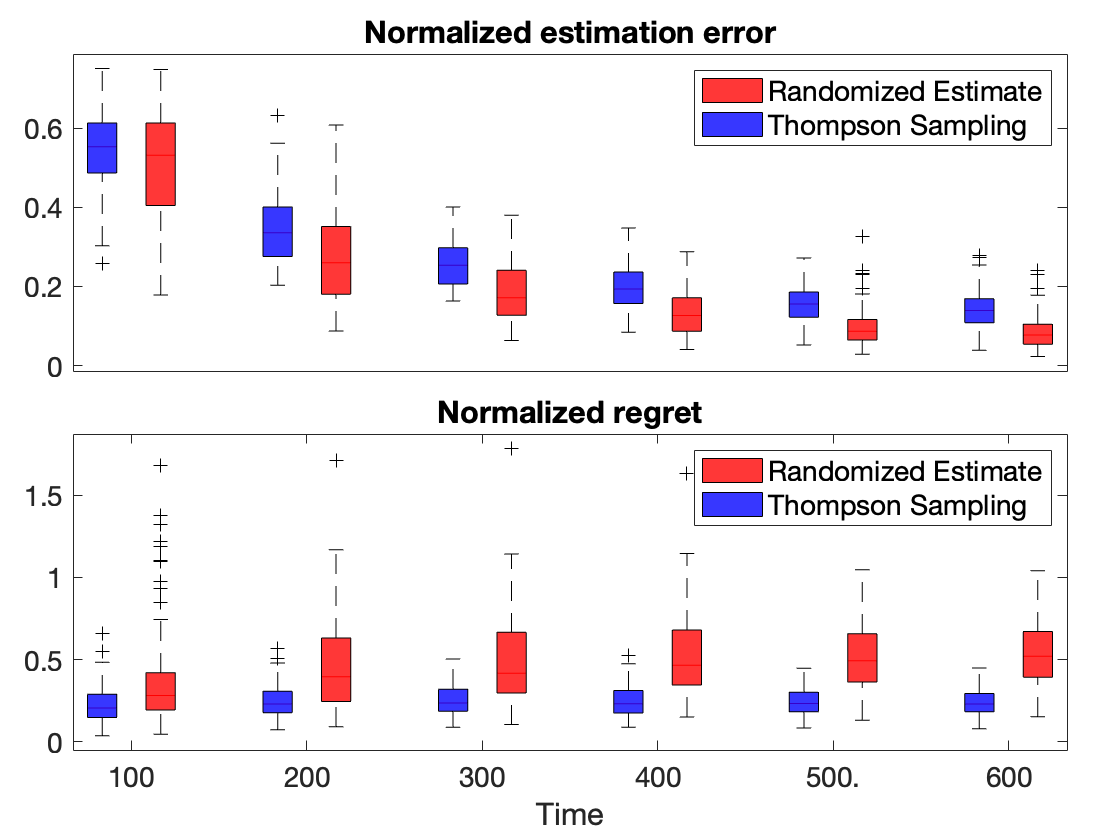}
    \end{minipage}
    \centering\includegraphics[width=.6\linewidth]{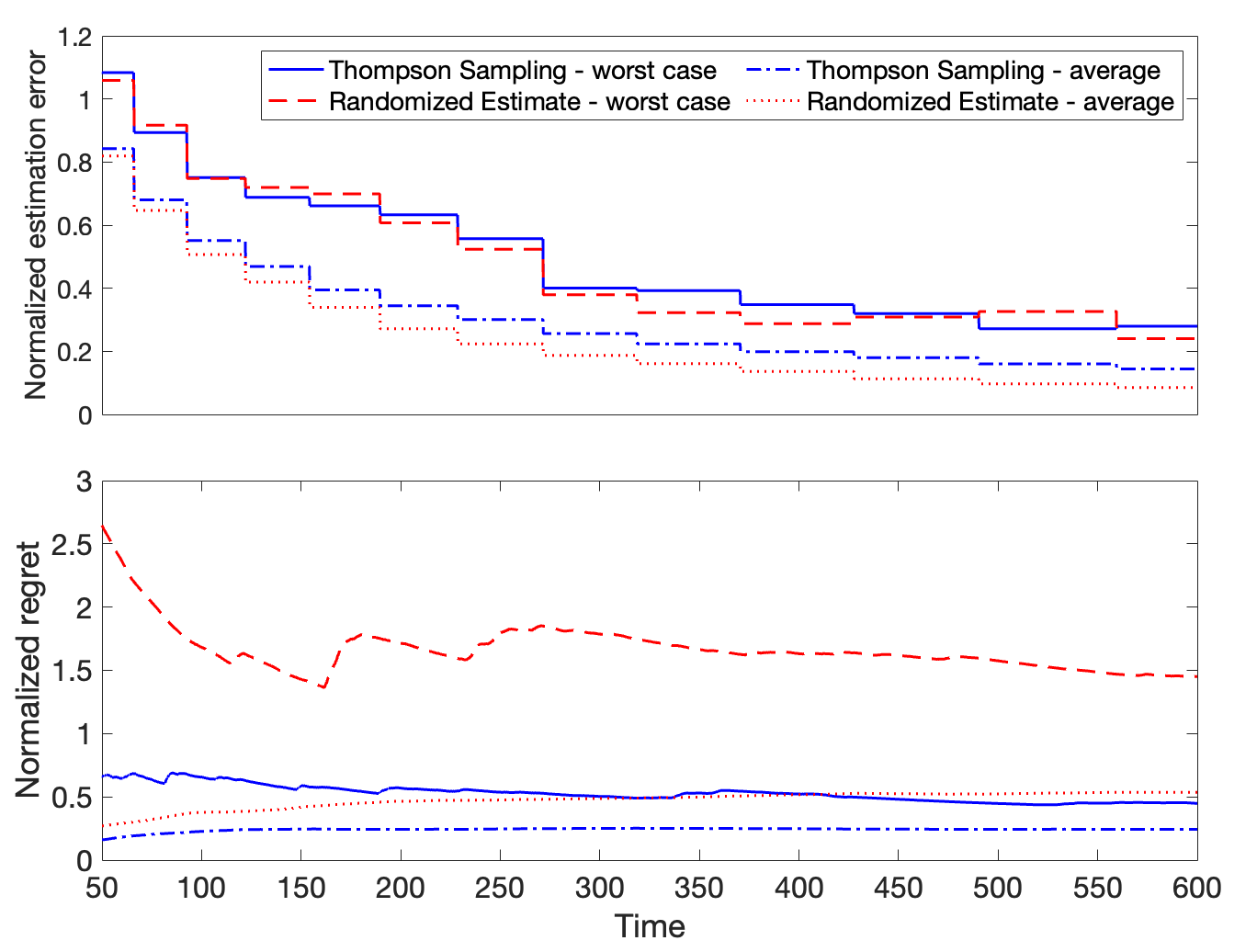}
    \caption{Analysis of X-29A airplane control system. \textbf{Left-top:} Percentage of successful system stabilization across 1000 runs of Algorithm~\ref{algo1}. \textbf{Right-top:} Comparative performance between Algorithm~\ref{algo2} (\textcolor{blue}{blue}) and Randomized Estimate policy (\textcolor{red}{red}). Upper panel shows normalized squared estimation error ($\Mnorm{\estpara{n}-\truth}{}^2$ divided by $\statedim (\statedim+\controldim) \episodetime{n}^{-1/2} \log \episodetime{n}$) at time points 100, 200, 300, 400, 500, and 600 seconds across 100 replications. Lower panel displays normalized regret ($\regret{T}{}$ divided by $\statedim (\statedim+\controldim) T^{1/2} \log T$). \textbf{Bottom:} Extended performance comparison between Algorithm~\ref{algo2} (\textcolor{blue}{blue}) and Randomized Estimate policy (\textcolor{red}{red}), with normalized squared estimation error (top) and normalized regret (bottom) for 100 replications.}
    \label{fig:X29}
\end{figure}
\end{samepage}

\subsubsection*{Boeing 747 Airplane Control.}
The Boeing 747 represents a commercial aviation system where operational efficiency and safety must be jointly optimized. The system dynamics are governed by the following drift matrices, with Algorithm~\ref{algo1} applied for $\episodetime{}$ varying from 10 to 45 and Algorithm~\ref{algo2} for an operational horizon of $50 \leq T \leq 600$: 
\begin{align*}
	\Amat{0} = \begin{bmatrix}
		-0.199\;\;\;    &0.003\;\;\;		&-0.980\;\;\;    &0.038
		\\
		-3.868\;\;\;    &-0.929\;\;\;		&0.471\;\;\;    &-0.008
		\\
		1.591\;\;\;    &-0.015\;\;\;		&-0.309\;\;\;    &0.003
		\\
		-0.198\;\;\;    &0.958\;\;\;		&0.021\;\;\;    &0.000
	\end{bmatrix}, ~~~~
	\Bmat{0} = \begin{bmatrix}
		-0.001\;\;\;    &0.058	\\
		0.296 \;\;\;    &0.153
		\\
		0.012\;\;\;    &-0.908
		\\
		0.015\;\;\;    &0.008
	\end{bmatrix}.
\end{align*}
\begin{figure}[ht]
    \centering
    \begin{minipage}[c]{0.45\linewidth}
        \centering
        \includegraphics[width=\linewidth]{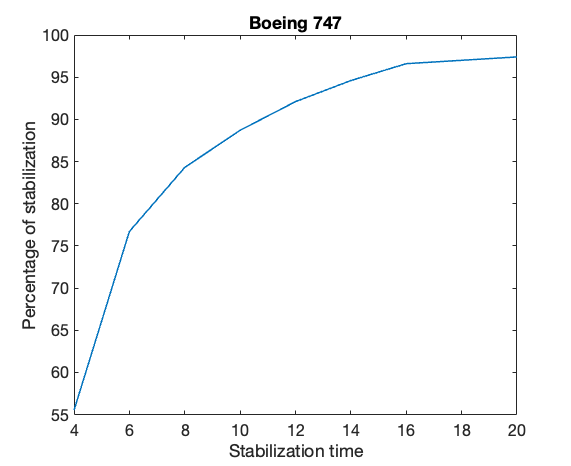}
    \end{minipage}%
    \hfill
    \begin{minipage}[c]{0.50\linewidth}
        \centering
        \includegraphics[width=\linewidth]{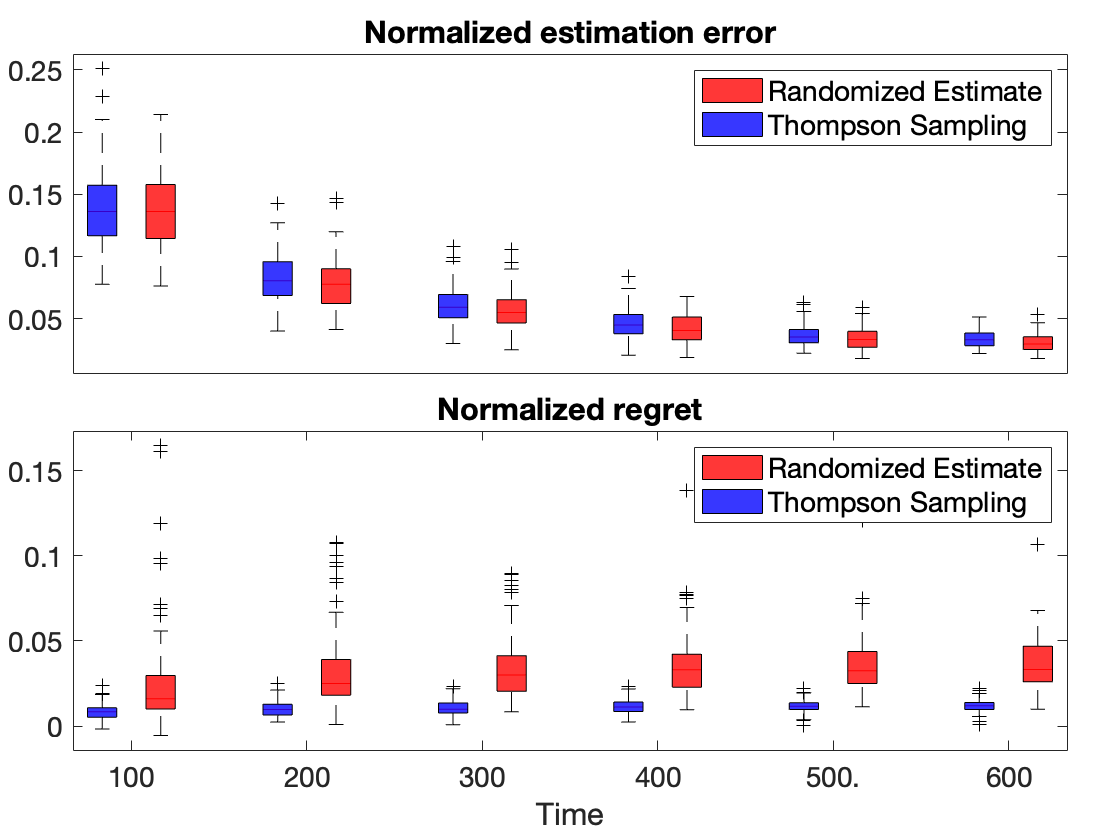}
    \end{minipage}
    \centering\includegraphics[width=.60\linewidth]{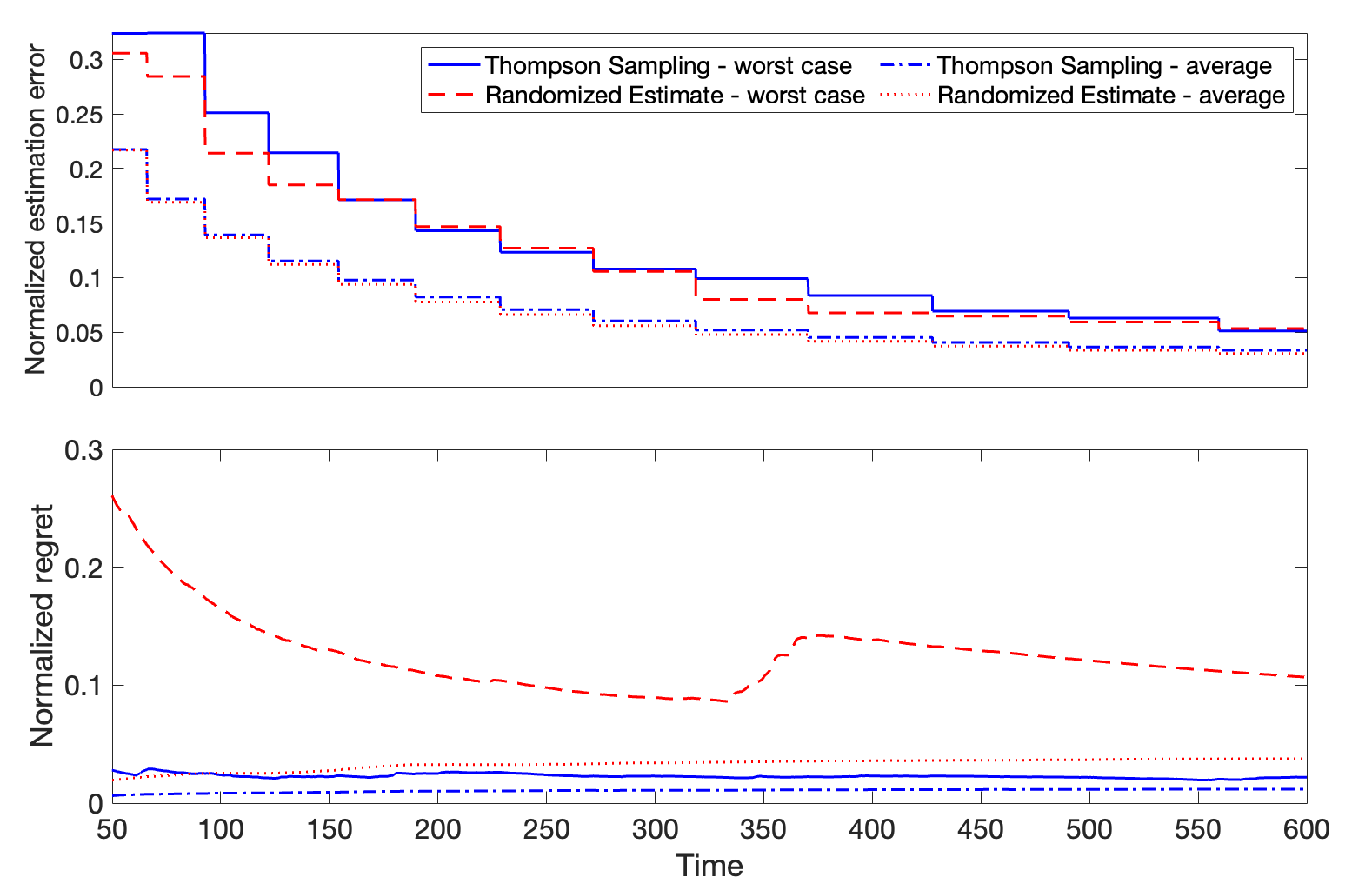}
    \caption{Analysis of Boeing 747 airplane control system. \textbf{Left-top:} Percentage of successful system stabilization across 1000 runs of Algorithm~\ref{algo1}. \textbf{Right-top:} Comparative performance between Algorithm~\ref{algo2} (\textcolor{blue}{blue}) and Randomized Estimate policy (\textcolor{red}{red}). Upper panel shows normalized squared estimation error ($\Mnorm{\estpara{n}-\truth}{}^2$ divided by $\statedim (\statedim+\controldim) \episodetime{n}^{-1/2} \log \episodetime{n}$) at time points 100, 200, 300, 400, 500, and 600 seconds across 100 replications. Lower panel displays normalized regret ($\regret{T}{}$ divided by $\statedim (\statedim+\controldim) T^{1/2} \log T$). \textbf{Bottom:} Extended performance comparison between Algorithm~\ref{algo2} (\textcolor{blue}{blue}) and Randomized Estimate policy (\textcolor{red}{red}), with normalized squared estimation error (top) and normalized regret (bottom) for 100 replications.}
    \label{fig:B747}
\end{figure}
\subsection{Results and Discussions}
Figures~\ref{fig:BGC}-\ref{fig:B747} depict the results of our simulation across all three settings, demonstrating several key insights across all three applications:

First, the probability of successful system stabilization increases exponentially with the running time of Algorithm~\ref{algo1}, confirming the theoretical guarantees established in Theorem~\ref{StabThm}. This exponential improvement in stabilization probability provides practical guidance for determining minimum operational time requirements in safety-critical applications.

Second, our comparative analysis reports both average- and worst-case performance metrics for estimation error and regret, normalized by their theoretical scaling with time and dimensionality as established in Theorem~\ref{RegretThm}. The results clearly demonstrate that Thompson sampling substantially outperforms the benchmark Randomized Estimate policy, particularly in worst-case scenarios. This suggests that our proposed approach explores the action space in a more robust fashion—a critical advantage in operational settings where reliability constraints dominate.

This performance advantage is consistent across all three systems. Notably, in the blood-glucose control application, the superior worst-case performance of our method has particularly significant implications, as healthcare operations typically prioritize predictable outcomes and worst-case guarantees over average-case performance. Similarly, the aircraft control applications benefit from improved worst-case guarantees, aligning with aviation safety requirements.

Overall, our empirical results validate our theoretical findings while demonstrating the practical advantages of TS for reinforcement learning in continuous-time stochastic control systems.

\section{Concluding Remarks}
\label{concluding_rem}
We studied TS reinforcement learning policies to control a diffusion process with unknown drift matrices. We first proposed a stabilization algorithm for linear diffusion processes and established that the failure probability of the algorithm decays exponentially with time. Furthermore, we demonstrated the efficiency of TS in balancing exploration versus exploitation for minimizing a quadratic cost function. More precisely, for Algorithm~\ref{algo2}, we established regret bounds growing as square-root of time and quadratically with dimensions. Our empirical studies across multiple applications—including aircraft control and blood-glucose regulation—demonstrate the superiority of TS over state-of-the-art methods.

Beyond the theoretical guarantees, our work also offers several insights for decision-making under uncertainty in continuous-time systems. The two-phase approach we present—first ensuring stability and then optimizing performance—provides a practical framework for implementing reinforcement learning in safety-critical operations where unbounded trajectories are unacceptable. Our analysis reveals a fundamental trade-off: systems with smaller stability margins (smaller $\zeta_0$) require significantly more exploration to learn to stabilize reliably, yet excessive exploration generates substantial operational costs. This trade-off is particularly relevant for practitioners, who must balance short-term performance with learning requirements. Notably, our findings suggest that the Bayesian approach of TS algorithm allows for more judicious exploration than traditional methods, resulting in substantially lower worst-case regret—a crucial consideration in operational settings where risk management is paramount.

As the first theoretical analysis of TS for control of continuous-time systems, this work opens several important directions for future research at the intersection of stochastic control and reinforcement learning. Establishing minimax regret lower-bounds for diffusion process control remains an open challenge that would provide valuable insights into the fundamental limitations of learning-based control methods. Additionally, extending TS for robust control of diffusion processes to simultaneously minimize cost functions across a family of drift matrices would address important practical concerns in operation contexts where model uncertainty is significant. Another promising research direction lies in analyzing the efficiency of TS for learning to control under partial observation scenarios, where the state is not directly observed but rather a noisy linear function of the state is available as the output signal—a common setting in many real-world operations and manufacturing systems.

\bibliographystyle{informs2014}
\bibliography{Ref}

\begin{APPENDICES}
	\newpage
\section*{Organization of Appendices}
This paper has three appendices.
Appendix~\ref{append1} contains the intermediate lemmas that are used in the proof of Theorem~\ref{StabThm} together with their proofs. Similarly, auxiliary results for the proof of Theorem~\ref{RegretThm} are presented in Appendix~\ref{append2}. Then, Appendix~\ref{append3} consists of statements and proofs of other technical results that are used for establishing both theorems.

\section{Auxiliary Lemmas in the Proof of Theorem~\ref{StabThm}} \label{append1} 

In the remainder of this section, the above-mentioned technical lemmas are stated and their proofs are provided in the corresponding subsections.

\subsection{Bounding cross products of state and randomization} 
\begin{deffn}
	For a set $\paraspace{}$, let $\indicator{\paraspace{}}$ be the indicator function that is $1$ on $\paraspace{}$, and vanishes outside of $\paraspace{}$.
\end{deffn}
\begin{lemm} \label{CrossTermLem1}
	In Algorithm~\ref{algo1}, for $t \geq 0$, define the piecewise-constant signal $v(t)$ below according to the randomization sequence $\dither{n}$:
	\begin{equation} \label{DitherTermDef}
	v(t)= \sum\limits_{n=0}^{\stabstep-1} \indicator{\frac{n\episodetime{}}{\stabstep} \leq t < \frac{(n+1)\episodetime{}}{\stabstep}} \dither{n}.
	\end{equation}
	Then, with probability at least $1-\delta$, we have
	\begin{eqnarray*}	
	&& {\Mnorm{ \itointeg{0}{\episodetime{}}{\state{s}v(s)^{\top} }{s} - \frac{\episodetime{}^2}{2 \stabstep^2} \Bmat{0} \sum\limits_{n=0}^{\stabstep-1} \dither{n} \dither{n}^{\top} }{2}} \\
	&\lesssim& {\left(\dithercoeff{}^2 + \eigmax{\BMcov{}}\right) \left( 1+ \itointeg{0}{\episodetime{}}{ \Mnorm{e^{\CLmat{}t}}{2} }{t} \right)}  \left( \statedim \controldim^{1/2} \episodetime{}^{1/2} \log \frac{\statedim \controldim}{\delta} + \controldim \left[ 1+ \frac{\episodetime{}^2}{ \stabstep^{2}} \right] \frac{\episodetime{}}{ \stabstep^{1/2}}  \log^{3/2} \frac{\stabstep \controldim}{\delta} \right).
	\end{eqnarray*}
\end{lemm}
\begin{prf}
	First, after plugging the control signal $\action{t}$ in \eqref{dynamics} and solving the resulting stochastic differential equation, we obtain
	\begin{equation} \label{GeneralOUEq}
	\state{t}=e^{\CLmat{}t} \state{0} + \itointeg{0}{t}{e^{\CLmat{}(t-s)}}{\BM{s}} + \itointeg{0}{t}{e^{\CLmat{}(t-s)} \Bmat{0}v(s) }{s},
	\end{equation}
	where $\CLmat{}=\Amat{0}+\Bmat{0}\Gainmat{}$ is the closed-loop transition matrix. This implies that
	\begin{equation*}
		\itointeg{0}{\episodetime{}}{\state{t}v(t)^\top}{t} = \randommatrix{1}+ \randommatrix{2} + \randommatrix{3},
	\end{equation*}
	where the three terms on the right-hand-side above are as follows:
	\begin{eqnarray*}
		\randommatrix{1} &=& \itointeg{0}{\episodetime{}}{ e^{\CLmat{}t} \state{0} v(t)^\top}{t} = \sum\limits_{n=0}^{\stabstep-1} \left( \itointeg{n \episodetime{} \stabstep^{-1}}{(n+1)\episodetime{} \stabstep^{-1}}{ e^{\CLmat{}t} }{t} \right) \state{0} \dither{n}^\top ,\\
		\randommatrix{2} &=& \itointeg{0}{\episodetime{}}{ \itointeg{0}{t}{e^{\CLmat{}(t-s)}}{\BM{s}} v(t)^\top}{t} = \sum\limits_{n=0}^{\stabstep-1} \left( \itointeg{n \episodetime{} \stabstep^{-1}}{(n+1)\episodetime{} \stabstep^{-1}}{ \itointeg{0}{t}{e^{\CLmat{}(t-s)}}{\BM{s}}}{t} \right) \dither{n}^\top ,\\
		\randommatrix{3} &=& \itointeg{0}{\episodetime{}}{\itointeg{0}{t}{e^{\CLmat{}(t-s)} \Bmat{0}v(s)~ }{s} ~v(t)^\top}{t}.
	\end{eqnarray*}
	Next, we focus on bounding these three matrices. To analyze $\randommatrix{1}$, we use the fact that every entry of $\randommatrix{1}$ is a normal random variable with mean zero, while the variance is at most
	\begin{equation*}
	\dithercoeff{}^2 \sum\limits_{n=0}^{\stabstep-1} \Mnorm{\itointeg{n \episodetime{} \stabstep^{-1}}{(n+1)\episodetime{} \stabstep^{-1}}{ e^{\CLmat{}t} }{t}}{2}^2 \norm{\state{0}}{2}^2 \leq \dithercoeff{}^2 \left(\itointeg{0}{\episodetime{}}{\Mnorm{e^{\CLmat{}t}}{2} }{t} \right)^2 \norm{\state{0}}{2}^2.
	\end{equation*}
	Therefore, with probability at least $1-\delta$, it holds that
	\begin{equation} \label{CrossTermLemAuxEq1}
		\Mnorm{\randommatrix{1}}{2} \lesssim \dithercoeff{} \left(\itointeg{0}{\episodetime{}}{\Mnorm{e^{\CLmat{}t}}{2} }{t} \right) \norm{\state{0}}{2} \sqrt{\statedim \controldim \log \left( \frac{\statedim \controldim}{\delta}\right)}.
	\end{equation}
	
	Now, to study $\randommatrix{2}$, Fubini Theorem~\citep{karatzas2012brownian} gives
	\begin{eqnarray*}
		\itointeg{n \episodetime{} \stabstep^{-1}}{(n+1)\episodetime{} \stabstep^{-1}}{ \itointeg{0}{t}{e^{\CLmat{}(t-s)}}{\BM{s}}}{t} &=&
		\itointeg{0}{n \episodetime{} \stabstep^{-1}}{ \left( \itointeg{n \episodetime{} \stabstep^{-1}}{(n+1)\episodetime{} \stabstep^{-1}}{e^{\CLmat{}(t-n \episodetime{} \stabstep^{-1})} }{t} \right) e^{\CLmat{}(n \episodetime{} \stabstep^{-1}-s)} }{\BM{s}} \\
		&+&
		\itointeg{n \episodetime{} \stabstep^{-1}}{(n+1)\episodetime{} \stabstep^{-1}}{ \left( \itointeg{s}{(n+1)\episodetime{} \stabstep^{-1}}{ e^{\CLmat{}(t-s)}}{t} \right) }{ \BM{s}}.
	\end{eqnarray*}
	To proceed, define the the matrix 
	$$F =\itointeg{n \episodetime{} \stabstep^{-1}}{(n+1)\episodetime{} \stabstep^{-1}}{e^{\CLmat{}(t-n \episodetime{} \stabstep^{-1})} }{t} =\itointeg{0}{\episodetime{} \stabstep^{-1}}{e^{\CLmat{}t} }{t}$$ 
	that does not depend on $s$ or $n$, as well as the matrix 
	$$G_s=\itointeg{s}{(n+1)\episodetime{} \stabstep^{-1}}{ e^{\CLmat{}(t-s)}}{t}$$ 
	for $n \episodetime{}\stabstep^{-1} \leq s \leq (n+1) \episodetime{} \stabstep^{-1}$, i.e., $G_s$ does not depend on $n$. Therefore, using $F, G_s$ we can write
\begin{eqnarray*}
	\itointeg{n \episodetime{} \stabstep^{-1}}{(n+1)\episodetime{} \stabstep^{-1}}{ \itointeg{0}{t}{e^{\CLmat{}(t-s)}}{\BM{s}}}{t} 
	&=&  F \sum\limits_{m=1}^{n} e^{\CLmat{}( n -m )  \episodetime{} \stabstep^{-1} } \itointeg{ (m-1) \episodetime{} \stabstep^{-1}}{m \episodetime{} \stabstep^{-1}}{ e^{\CLmat{}( m \episodetime{} \stabstep^{-1}-s)} }{\BM{s}} \\
	&+&
	\itointeg{n \episodetime{} \stabstep^{-1}}{(n+1)\episodetime{} \stabstep^{-1}}{ G_s }{ \BM{s}}.
	\end{eqnarray*}
	
	So, letting $\basis{i}$ for $i=1, \cdots, \statedim$ be the standard basis of the Euclidean space $\R^{\statedim}$, conditioned on the Wiener process $\left\{ \BM{s} \right\}_{s \geq 0}$, for every $j=1,\cdots, \controldim$, the coordinate $j$ of $\basis{i}^\top \randommatrix{2}$ is a mean zero normal random variable. Thus, given $\left\{\BM{s}\right\}_{s \geq 0}$, with probability at least $1-\delta$, it holds that
	\begin{equation*}
	\left(\basis{i}^\top \randommatrix{2} \basis{j}\right)^2 \lesssim \mathrm{var \left( \basis{i}^\top \randommatrix{2} \basis{j} \Big| \sigfield{\BM{0:\episodetime{}}} \right) } \log \frac{1}{\delta}.
	\end{equation*}
	
	Now, to calculate the conditional variance, we can write
	\begin{equation*}
		\frac{\mathrm{var \left( \basis{i}^\top \randommatrix{2} \basis{j} \Big| \sigfield{\BM{0:\episodetime{}}} \right) }}{\dithercoeff{}^2} = \sum\limits_{n=0}^{\stabstep-1} \left[ \basis{i}^\top \itointeg{n \episodetime{} \stabstep^{-1}}{(n+1)\episodetime{} \stabstep^{-1}}{ \itointeg{0}{t}{e^{\CLmat{}(t-s)}}{\BM{s}}}{t} \right]^2 \lesssim \sum\limits_{n=1}^{\stabstep-1} \left[\left( \sum\limits_{m=1}^n \regterm{m,n} \right)^2 + \ssconstant_n^2\right],
	\end{equation*}
	where
	\begin{eqnarray*}
		\regterm{m,n} &=& \basis{i}^\top F e^{\CLmat{}( n -m )  \episodetime{} \stabstep^{-1} } \itointeg{ (m-1) \episodetime{} \stabstep^{-1}}{m \episodetime{} \stabstep^{-1}}{ e^{\CLmat{}( m \episodetime{} \stabstep^{-1}-s)} }{\BM{s}} , \\
		\ssconstant_n &=& \basis{i}^\top \itointeg{n \episodetime{} \stabstep^{-1}}{(n+1)\episodetime{} \stabstep^{-1}}{ G_s }{ \BM{s}}.
 	\end{eqnarray*}
 	To proceed, define the matrix $H = \left[ H_{n,m} \right]$, where for $1 \leq m,n \leq \stabstep-1$, every block $H_{n,m} \in \R^{1 \times \statedim}$ is
 	\begin{equation*} \label{ToeplitzMatrixDefEq}
	 	H_{n,m} = \basis{i}^\top F e^{\CLmat{}( n -m )  \episodetime{} \stabstep^{-1} },
 	\end{equation*}
 	for $m \leq n$, and is $0$ for $m >n$. Then, denote
 	\begin{equation*}
	 	\Gamma = \begin{bmatrix}
	 		\itointeg{ 0 }{ \episodetime{} \stabstep^{-1}}{ e^{\CLmat{}( \episodetime{} \stabstep^{-1}-s)} }{\BM{s}} \\
	 		\itointeg{ \episodetime{} \stabstep^{-1}}{2 \episodetime{} \stabstep^{-1}}{ e^{\CLmat{}( 2\episodetime{} \stabstep^{-1}-s)} }{\BM{s}} \\
	 		\vdots \\
	 		\itointeg{ (\stabstep-1) \episodetime{} \stabstep^{-1}}{ \episodetime{} }{ e^{\CLmat{}( \episodetime{}-s)} }{\BM{s}}
	 	\end{bmatrix} \in \R^{\statedim (\stabstep-1) \times 1 },
 	\end{equation*}
 	to get
 	\begin{equation*}
	 	\sum\limits_{n=0}^{\stabstep-1} \left( \sum\limits_{m=1}^n \regterm{m,n}^2 \right) = \norm{H \Gamma}{2}^2 \leq \eigmax{H^\top H} \norm{\Gamma}{2}^2.
 	\end{equation*}
 	Now, for the matrix $H$, we have~\citep{hartman1954spectra,reichel1992eigenvalues}:
 	\begin{equation*} \label{ToeplitzMatrixIneq}
	 	\eigmax{H^\top H} \lesssim \left( \sum\limits_{n=1}^{\stabstep-1} \norm{H_{n,1}}{2} \right)^2 \lesssim \left( \episodetime{}\stabstep^{-1} \sum\limits_{n=1}^{\stabstep} e^{\CLmat{}n \episodetime{}\stabstep^{-1} } \right)^2 \lesssim \left(\itointeg{0}{\episodetime{}}{\Mnorm{ e^{\CLmat{}t} }{2} }{t}\right)^2.
 	\end{equation*}
 	Note that thanks to the independent increments of the Wiener process, the blocks of $\Gamma$ are statistically independent. Further, by Ito Isometry~\citep{karatzas2012brownian}, every block of $\Gamma$ is a mean-zero normally distributed vector with the covariance matrix
 	\begin{equation*}
 		\itointeg{ 0 }{ \episodetime{} \stabstep^{-1}}{ e^{\CLmat{}( \episodetime{} \stabstep^{-1}-s)} \BMcov{} e^{\CLmat{}^\top ( \episodetime{} \stabstep^{-1}-s)} }{s}.
 	\end{equation*}
 	So, according to the exponential inequalities for quadratic forms of normally distributed random variables~\citep{laurent2000adaptive}, it holds with probability at least $1-\delta$, that
 	\begin{equation*}
	 	\norm{\Gamma}{2}^2 \lesssim \statedim \stabstep \eigmax{\BMcov{}} \left( \episodetime{} \stabstep^{-1} \right) \log \frac{1}{\delta}.
 	\end{equation*}
 	Thus, with probability at least $1-\delta$, we have
 	\begin{equation*}
	 	\sum\limits_{n=0}^{\stabstep-1} \left( \sum\limits_{m=1}^n \regterm{m,n}^2 \right) \lesssim \left(\itointeg{0}{\episodetime{}}{\Mnorm{ e^{\CLmat{}t} }{2} }{t}\right)^2 \statedim \eigmax{\BMcov{}} \episodetime{} \log \frac{1}{\delta}.
 	\end{equation*}
 	Similarly, the bound above can be shown for $\sum\limits_{n=1}^{\stabstep-1} \ssconstant_n^2$. Hence, we obtain the corresponding high probability bound for a single entry $\basis{i}^\top \randommatrix{2} \basis{j}$ of $\randommatrix{2}$, which together with a union bound, implies that
 	\begin{equation} \label{CrossTermLemAuxEq2}
	 	\Mnorm{\randommatrix{2}}{2} \lesssim \dithercoeff{} \statedim \controldim^{1/2} \left(\itointeg{0}{\episodetime{}}{\Mnorm{ e^{\CLmat{}t} }{2} }{t}\right) \eigmax{\BMcov{}}^{1/2} \episodetime{}^{1/2} \log \left(\frac{\statedim \controldim }{\delta}\right),
 	\end{equation}
 	with probability at least $1-\delta$.
 	
	Next, according to Fubini Theorem, $\randommatrix{3}$ can also be written as
	\begin{eqnarray*}
		\randommatrix{3} &=& \itointeg{0}{\episodetime{}}{\itointeg{0}{s}{e^{\CLmat{}(s-t)} \Bmat{0}v(t) v(s)^\top }{t} }{s} = \itointeg{0}{\episodetime{}}{\itointeg{t}{\episodetime{}}{e^{\CLmat{}(s-t)} \Bmat{0}v(t) v(s)^\top }{s} }{t}.
	\end{eqnarray*}
	Thus, we have
	\begin{equation*}
		2 \randommatrix{3}  = \itointeg{0}{\episodetime{}}{\itointeg{0}{\episodetime{}}{e^{\CLmat{}\left| t-s \right| } \Bmat{0} v(t \wedge s) v(s \vee t)^\top }{t} }{s}.
	\end{equation*}
	Recall that the signal $v(t)$ in \eqref{DitherTermDef} is piecewise-constant, and the values of this signal are determined by the randomization sequence $\dither{n}$. So, the above double integral can be written as a double sum
	\begin{eqnarray*}
		2 \randommatrix{3} &=& \sum\limits_{n=0}^{\stabstep-1} \sum\limits_{m=0}^{\stabstep-1} \left(\itointeg{n \episodetime{} \stabstep^{-1}}{(n+1)\episodetime{} \stabstep^{-1}}{ \itointeg{m \episodetime{} \stabstep^{-1}}{(m+1)\episodetime{} \stabstep^{-1}}{ e^{\CLmat{}\left| t-s \right| }  }{s}  }{t}\right) \Bmat{0} \dither{m \wedge n } \dither{m \vee n}^{\top} \\
		&=& \sum\limits_{n=0}^{\stabstep-1} \sum\limits_{m=0}^{\stabstep-1} \left( e^{\CLmat{}\left| m-n \right| \episodetime{}\stabstep^{-1} } \itointeg{0}{\episodetime{} \stabstep^{-1}}{ \itointeg{0}{\episodetime{} \stabstep^{-1}}{ e^{\CLmat{}\left| t-s \right| }  }{s}  }{t}\right) \Bmat{0} \dither{m \wedge n } \dither{m \vee n}^{\top}.
	\end{eqnarray*}
 	Thus, we have
 	\begin{equation} \label{CrossTermLemAuxEq3}
  		2 \randommatrix{3} - \frac{\episodetime{}^2}{ \stabstep^2} \Bmat{0} \sum\limits_{n=0}^{\stabstep-1} \dither{n} \dither{n}^{\top} = \randommatrix{4} + \randommatrix{5} ,
 	\end{equation}
 	for
 	\begin{eqnarray*}
	 	\randommatrix{4} &=& \left( \itointeg{0}{\episodetime{} \stabstep^{-1}}{ \itointeg{0}{\episodetime{} \stabstep^{-1}}{ e^{\CLmat{}\left| t-s \right| }  }{s}  }{t} - \episodetime{}^2 \stabstep^{-2} I_{\controldim} \right) \Bmat{0} \sum\limits_{n=0}^{\stabstep-1} \dither{ n } \dither{ n }^{\top} , \\
	 	\randommatrix{5} &=& 2 \left( \itointeg{0}{\episodetime{} \stabstep^{-1}}{ \itointeg{0}{\episodetime{} \stabstep^{-1}}{ e^{\CLmat{}\left| t-s \right| }  }{s}  }{t} \right) \sum\limits_{n=0}^{\stabstep-1} \sum\limits_{m=n+1}^{\stabstep-1} \left( e^{\CLmat{}\left( m-n \right) \episodetime{}\stabstep^{-1} } \Bmat{0} \dither{ n } \dither{ m }^{\top} \right).
 	\end{eqnarray*}
 	To proceed, we use the following concentration inequality for random matrices with martingale difference structures, titled as Matrix Azuma inequality~\citep{tropp2012user}.
 	\begin{thrm} \label{Azuma}
 		Let $\left\{ \RandMat{n} \right\}_{n=1}^k$ be a $d_1 \times d_2$ martingale difference sequence. That is, for some filtration $\left\{\filter{n} \right\}_{n=0}^k$, the matrix $\RandMat{n}$ is $\filter{n}$-measurable, and $\E{\RandMat{n}\Big| \filter{n-1}}=0$. Suppose that $\Mnorm{\RandMat{n}}{2} \leq \sigma_n$, for some fixed sequence $\left\{ \sigma_n \right\}_{n=1}^k$. Then, with probability at least $1-\delta$, we have
 		\begin{equation*}
 		\Mnorm{\sum\limits_{n=1}^{k}\RandMat{n}}{}^2 \lesssim \left( \sum\limits_{n=1}^k \sigma_n^2 \right) \log \frac{d_1+d_2}{\delta}.
 		\end{equation*}
 	\end{thrm}
 	
 	So, to study $\randommatrix{4}$, we apply Theorem~\ref{Azuma} to the random matrices $\RandMat{n} = \dither{n}\dither{n}^\top - \dithercoeff{}^2 I_{\controldim}$, using the trivial filtration and the high probability upper-bounds for $\Mnorm{\RandMat{n}}{2} \leq \norm{\dither{n}}{2}^2 + \dithercoeff{}^2$;
 	\begin{equation*}
	 	\Mnorm{\RandMat{n}}{2} \leq \sigma_n = \dithercoeff{}^2 \left( 1 + \controldim \log \frac{\controldim \stabstep }{\delta} \right),
 	\end{equation*}
 	as well as the fact
 	\begin{equation*}
 		\Mnorm{\itointeg{0}{\episodetime{} \stabstep^{-1}}{ \itointeg{0}{\episodetime{} \stabstep^{-1}}{ \left(e^{\CLmat{}\left| t-s \right| } - I_{\controldim}\right)  }{s}  }{t}}{2} \lesssim \episodetime{}^3 \stabstep^{-3},
 	\end{equation*}
 	to obtain the following bound, which holds with probability at least $1-\delta$:
 	\begin{equation} \label{CrossTermLemAuxEq4}
	 	\Mnorm{\randommatrix{4}}{2} \lesssim \Mnorm{\Bmat{0}}{2} \dithercoeff{}^2 \episodetime{}^3 \stabstep^{-2} \left( 1 + \frac{\controldim}{\stabstep^{1/2}} \log^{3/2} \frac{\stabstep \controldim}{\delta} \right) .
 	\end{equation}

 	On the other hand, to establish an upper-bound for $\randommatrix{5}$, consider the random matrices
 	\begin{equation*}
	 	\RandMat{n}= \sum\limits_{m=n+1}^{\stabstep-1} \left( e^{\CLmat{}\left( m-n \right) \episodetime{}\stabstep^{-1} } \Bmat{0} \dither{ n } \dither{ m }^{\top} \right),
 	\end{equation*}
 	subject to the natural filtration they generate, and apply Theorem~\ref{Azuma}, using the bounds
 	\begin{equation*}
	 	\Mnorm{\RandMat{n}}{2} \le \sigma_n \lesssim \episodetime{}^{-1} \stabstep \left( \itointeg{0}{\episodetime{}}{ \Mnorm{e^{\CLmat{}t}}{2} }{t} \right) \Mnorm{\Bmat{0}}{2} \dithercoeff{}^2 \controldim \log \frac{\stabstep \controldim}{\delta},
 	\end{equation*}
 	together with
 	\begin{equation*}
 	\Mnorm{\itointeg{0}{\episodetime{} \stabstep^{-1}}{ \itointeg{0}{\episodetime{} \stabstep^{-1}}{ e^{\CLmat{}\left| t-s \right| }  }{s}  }{t}}{2} \lesssim \episodetime{}^2 \stabstep^{-2}.
 	\end{equation*}
 	Therefore, Theorem~\ref{Azuma} indicates that with probability at least $1-\delta$, it holds that
 	\begin{equation} \label{CrossTermLemAuxEq5}
	 	\randommatrix{5} \lesssim \frac{\episodetime{}}{\stabstep^{1/2}} \left( \itointeg{0}{\episodetime{}}{ \Mnorm{e^{\CLmat{}t}}{2} }{t} \right) \Mnorm{\Bmat{0}}{2} \dithercoeff{}^2 \controldim \log^{3/2} \frac{\stabstep \controldim}{\delta}.
 	\end{equation}
 	
 	Now, plug in the inequalities in \eqref{CrossTermLemAuxEq4} and \eqref{CrossTermLemAuxEq5} that bound $\randommatrix{4}, \randommatrix{5}$ into the equation in \eqref{CrossTermLemAuxEq3}, to obtain a high probability concentration inequality for $\randommatrix{3}$. The latter result together with \eqref{CrossTermLemAuxEq1}, \eqref{CrossTermLemAuxEq2} that provide the corresponding upper-bounds for $\randommatrix{1}, \randommatrix{2}$, yield to the desired result.
 	
\end{prf}

\subsection{Bounding cross products of state and Wiener process}
\begin{lemm} \label{CrossTermLem2}
	In Algorithm~\ref{algo1}, with probability at least $1-\delta$, we have
	\begin{equation*}
	\Mnorm{\itointeg{0}{t}{\state{s}}{\BM{s}^{\top}}}{2}
	\lesssim \left(\itointeg{0}{\episodetime{}}{\Mnorm{ e^{\CLmat{}t} }{2} }{t}\right) \left( \eigmax{\BMcov{}} + \dithercoeff{}^2 \right) \statedim \left( \statedim+ \controldim \right)^{1/2} \episodetime{}^{1/2} \log \left(\frac{\statedim \controldim }{\delta}\right).
	\end{equation*}
\end{lemm}
\begin{prf}
	First, according to \eqref{GeneralOUEq}, we can write
	\begin{equation*}
		\itointeg{0}{\episodetime{}}{\state{t} }{\BM{t}^\top} = \randommatrix{1} + \randommatrix{2} + \randommatrix{3},
	\end{equation*}
	where the following random matrices are denoted:
	\begin{eqnarray*}
		\randommatrix{1} &=& \itointeg{0}{\episodetime{}}{ e^{\CLmat{}t} \state{0} }{\BM{t}^\top} ,\\
		\randommatrix{2} &=& \itointeg{0}{\episodetime{}}{ \itointeg{0}{t}{ e^{\CLmat{}(t-s)} \Bmat{0} v(s) }{s} }{\BM{t}^\top}, \\
		\randommatrix{3} &=& \itointeg{0}{\episodetime{}}{ \itointeg{0}{t}{ e^{\CLmat{}(t-s)} }{\BM{s}} }{\BM{t}^\top}.
	\end{eqnarray*}
	Now, according to Ito Isometry \citep{karatzas2012brownian}, similar to \eqref{CrossTermLemAuxEq1}, we have
	\begin{equation} \label{CrossLem2Eq1}
	\Mnorm{\randommatrix{1}}{2} \lesssim \eigmax{\BMcov{}}^{1/2} \left(\itointeg{0}{\episodetime{}}{\Mnorm{e^{\CLmat{}t}}{2} }{t} \right) \norm{\state{0}}{2} \sqrt{\statedim \controldim \log \left( \frac{\statedim \controldim}{\delta}\right)},
	\end{equation}
	with probability at least $1 -\delta$. Moreover, in a procedure similar to the one that leads to \eqref{CrossTermLemAuxEq2}, one can show that with probability at least $1-\delta$, it holds that
	\begin{equation} \label{CrossLem2Eq2}
		\Mnorm{\randommatrix{2}}{2} \lesssim  \left(\itointeg{0}{\episodetime{}}{\Mnorm{ e^{\CLmat{}t} }{2} }{t}\right) \eigmax{\BMcov{}}^{1/2} \dithercoeff{} \statedim \controldim^{1/2} \episodetime{}^{1/2} \log \left(\frac{\statedim \controldim }{\delta}\right).
	\end{equation}
	
	Therefore, we need to find a similar upper-bound for $\randommatrix{3}$. To that end, Ito formula provides
	\begin{equation*}
		\diff \left( e^{-\CLmat{}s} \BM{s} \right) = -\CLmat{} e^{-\CLmat{}s} \BM{s} \diff s + e^{-\CLmat{}s} \diff \BM{s}.
	\end{equation*}
	Therefore, integration gives
	\begin{equation*}
		\itointeg{0}{t}{ e^{-\CLmat{}s} }{\BM{s}} = e^{-\CLmat{}t} \BM{t} + \CLmat{} \itointeg{0}{t}{ e^{-\CLmat{}s} \BM{s} }{s},
	\end{equation*}
	which after rearranging and letting $\RandMat{t}=\itointeg{0}{t}{ e^{\CLmat{}(t-s)} }{\BM{s}}$, leads to
	\begin{equation*}
		\RandMat{t} \BM{t}^\top = \left(\itointeg{0}{t}{ e^{\CLmat{}(t-s)} }{\BM{s}}\right) {\BM{t}^\top} = { \BM{t} }{\BM{t}^\top} + \CLmat{} { \left(\itointeg{0}{t}{ e^{\CLmat{}(t-s)} \BM{s} }{s}\right) }{\BM{t}^\top}.
	\end{equation*}
	Now, since $\diff \RandMat{t} = \diff \BM{t}$, Ito Isometry~\citep{karatzas2012brownian} implies that $\diff \RandMat{t} \diff \BM{t}^\top = \BMcov{} \diff t$. So, apply integration by part and use the above equation to get
	\begin{eqnarray*}
	\randommatrix{3} = \itointeg{0}{\episodetime{}}{ \RandMat{t} }{\BM{t}^\top} &=& \itointeg{0}{\episodetime{}}{  }{ \left( \RandMat{t} \BM{t}^\top \right) } - \left(\itointeg{0}{\episodetime{}}{ \BM{t} }{ \RandMat{t}^\top }\right)^\top - \itointeg{0}{\episodetime{}}{ \diff \RandMat{t} }{ \BM{t}^\top } \\
	&=& { \RandMat{\episodetime{}} }{\BM{\episodetime{}}^\top} - \left(\itointeg{0}{\episodetime{}}{ \BM{t} }{ \BM{t}^\top }\right)^\top - \BMcov{} \episodetime{} \\
	&=& { \BM{\episodetime{}} }{\BM{\episodetime{}}^\top} + \CLmat{} { \left(\itointeg{0}{\episodetime{}}{ e^{\CLmat{}(\episodetime{}-s)} \BM{s} }{s}\right) }{\BM{\episodetime{}}^\top}  - \left(\itointeg{0}{\episodetime{}}{ \BM{t} }{ \BM{t}^\top }\right)^\top - \BMcov{} \episodetime{} .
	\end{eqnarray*}
	Above, in the last line, we used the fact that
	\begin{equation*}
		\BM{\episodetime{}} \BM{\episodetime{}}^\top = \itointeg{0}{\episodetime{}}{}{\left( \BM{t} \BM{t}^\top \right)} = \left(\itointeg{0}{\episodetime{}}{ \BM{t} }{ \BM{t}^\top }\right)^\top + \left(\itointeg{0}{\episodetime{}}{ \BM{t} }{ \BM{t}^\top }\right) + \BMcov{} \episodetime{}.
	\end{equation*}
	
	This shows that every entry of $\randommatrix{3}$ is a quadratic function of the normally distributed random vectors $\BM{\episodetime{}}$ and $\itointeg{0}{\episodetime{}}{ e^{\CLmat{}(\episodetime{}-s)} \BM{s} }{s}$. Thus, exponential inequalities for quadratic forms of normal random vectors~\citep{laurent2000adaptive} imply that for all $i,j =1, \cdots, \statedim$, it holds that
	\begin{equation} \label{CrossLem2Eq3}
		\left(\basis{i}^\top \randommatrix{3} \basis{j}\right)^2 \lesssim \statedim \E{ \left( \basis{i}^\top \randommatrix{3} \basis{j} \right)^2} \log^2 \frac{1}{\delta},
	\end{equation}
	since $\E{\basis{i}^\top \randommatrix{3} \basis{j}}=0$. So, it suffices to find the expectation in \eqref{CrossLem2Eq3}. For that purpose, we use Ito Isometry~\citep{karatzas2012brownian} to obtain:
	\begin{eqnarray*}
		\E{\left(\basis{i}^\top \randommatrix{3} \basis{j}\right)^2} &=& \E{ \left(\itointeg{0}{\episodetime{}}{ \basis{i}^\top \RandMat{t} \basis{j}^\top \BMcov{}^{1/2} }{\left( \BMcov{}^{-1/2} \BM{t} \right)}\right)^2}
		= \E{ \itointeg{0}{\episodetime{}}{ \norm{\basis{i}^\top \RandMat{t} \BMcov{}^{1/2} \basis{j}}{2}^2 }{t}} \\
		&\leq& \basis{j}^\top \BMcov{} \basis{j} \E{ \itointeg{0}{\episodetime{}}{ \left(\basis{i}^\top \RandMat{t}\right)^2 }{t}}
		= \basis{j}^\top \BMcov{} \basis{j} \E{ \itointeg{0}{\episodetime{}}{ \left(\basis{i}^\top \itointeg{0}{t}{ e^{\CLmat{}(t-s)} }{\BM{s}} \right)^2 }{t}}.
	\end{eqnarray*}
	
	To proceed with the above expression, apply Fubini Theorem~\citep{karatzas2012brownian} to interchange the expected value with the integral, and then use Ito Isometry again:
	\begin{eqnarray*}
		\E{ \itointeg{0}{\episodetime{}}{ \left(\basis{i}^\top \itointeg{0}{t}{ e^{\CLmat{}(t-s)} }{\BM{s}} \right)^2 }{t}}
		&=& \itointeg{0}{\episodetime{}}{ \E{\left(\basis{i}^\top \itointeg{0}{t}{ e^{\CLmat{}(t-s)} \BMcov{}^{1/2} }{\left( \BMcov{}^{-1/2} \BM{s}\right)} \right)^2} }{t} \\
		&=& \itointeg{0}{\episodetime{}}{ \basis{i}^\top \left(\itointeg{0}{t}{ e^{\CLmat{}(t-s)} \BMcov{} e^{\CLmat{}^\top(t-s)} }{s}\right) \basis{i} }{t} \\
		&\leq&  \basis{i}^\top \left(\itointeg{0}{\episodetime{}}{ e^{\CLmat{}s} \BMcov{} e^{\CLmat{}^\top s} }{s} \right) \basis{i} \episodetime{}.
	\end{eqnarray*}
	Therefore, \eqref{CrossLem2Eq3} yields to
	\begin{eqnarray}
		\Mnorm{\randommatrix{3}}{2}^2 \leq \sum\limits_{i,j=1}^{\statedim} \left(\basis{i}^\top \randommatrix{3} \basis{j}\right)^2
		&\lesssim& \sum\limits_{i,j=1}^{\statedim} \left[ \basis{j}^\top \BMcov{} \basis{j} \basis{i}^\top \left(\itointeg{0}{\episodetime{}}{ e^{\CLmat{}s} \BMcov{} e^{\CLmat{}^\top s} }{s} \right) \basis{i} \right] \episodetime{} \statedim \log^2 \frac{\statedim}{\delta} \notag \\
		&=& \tr{\BMcov{}} \tr{\itointeg{0}{\episodetime{}}{ e^{\CLmat{}s} \BMcov{} e^{\CLmat{}^\top s} }{s}} \statedim \episodetime{} \log^2 \frac{\statedim}{\delta} \notag \\
		&\lesssim& \tr{\BMcov{}}^2 \left( \itointeg{0}{\episodetime{}}{ \Mnorm{e^{\CLmat{}s}}{2} }{s} \right)^2 \statedim \episodetime{} \log^2 \frac{\statedim}{\delta} . \label{CrossLem2Eq4}
	\end{eqnarray}
	Finally, putting \eqref{CrossLem2Eq1}, \eqref{CrossLem2Eq2}, and \eqref{CrossLem2Eq4} together, we obtain the desired result.
\end{prf}

\subsection{Concentration of normal posterior distribution in Algorithm~\ref{algo1}} 

\begin{lemm} \label{MinEigEmpCovLem}
	In Algorithm~\ref{algo1}, letting $\CLmat{}=\Amat{0}+\Bmat{0}\Gainmat{}$, suppose that
	\begin{eqnarray}
	\episodetime{} &\gtrsim& \left(\itointeg{0}{\episodetime{}}{\Mnorm{\exp (\CLmat{}s)}{}^2}{s}\right) \left( \eigmax{\BMcov{}} + \dithercoeff{}^2 \Mnorm{\Bmat{0}}{2}^2\right) (\statedim+\controldim) \log \frac{1}{\delta}, \label{MinEigEmpCovLemEq1} \\
	\frac{\stabstep}{\episodetime{}} &\gtrsim& \frac{\dithercoeff{}^2}{\dithercoeff{}^2 \wedge \eigmin{\BMcov{}}} ~~ \Mnorm{\Bmat{0}}{2} \left( 1 \vee \Mnorm{\Gainmat{}}{2} \right) \controldim \log \frac{\stabstep \controldim}{\delta}. \label{MinEigEmpCovLemEq2}
	\end{eqnarray}
	Then, for the matrix $\empiricalcovmat{\episodetime{}}$ in~\eqref{RandomLSE1}, with probability at least $1-\delta$ we have
	\begin{eqnarray*}
	\eigmin{\empiricalcovmat{\episodetime{}}} \gtrsim \episodetime{} \left(\eigmin{\BMcov{}} \wedge \dithercoeff{}^2\right) \left( 1 + \Mnorm{\Gainmat{}}{2}^2 \right)^{-1}.
	\end{eqnarray*}
\end{lemm}
\begin{prf}
	First, we can write the control action in~\eqref{PerturbActionEq} as $\action{t} = \Gainmat{} \state{t} + v(t)$, for the piecewise-constant continuous-time signal $v(t)$ that we defined in~\eqref{DitherTermDef}. Then, the dynamics in~\eqref{dynamics} provides
	\begin{equation*}
	\diff \state{t} = \left(\CLmat{} \state{t} + \Bmat{0} v(t)\right) \diff t + \diff \BM{t}.
	\end{equation*}
	Therefore, similar to \eqref{GeneralOUEq}, one can solve the above stochastic differential equation to get
	\begin{equation*}
		\state{t}=e^{\CLmat{}t} \state{0} + \itointeg{0}{t}{e^{\CLmat{}(t-s)}}{\BM{s}} + \itointeg{0}{t}{e^{\CLmat{}(t-s)} \Bmat{0}v(s) }{s}.
	\end{equation*}
	So, using the exponential inequalities for quadratic forms~\citep{laurent2000adaptive}, with probability at least $1-\delta$, it holds that
	\begin{equation} \label{EndStateBoundEq}
		\norm{\state{\episodetime{}} - e^{\CLmat{}\episodetime{}} \state{0} }{2}^2 \lesssim \eigmax{ \itointeg{0}{\episodetime{}}{ e^{\CLmat{}s} \BMcov{} e^{\CLmat{}^\top s} }{s} + \dithercoeff{}^2 \sum\limits_{n=0}^{\stabstep-1} J_n \Bmat{0} \Bmat{0}^\top J_n^\top } \left(\statedim+\statedim^{1/2} \log \frac{1}{\delta}\right),
	\end{equation}
	where the matrix 
	\begin{equation*}
		J_n = \itointeg{n \episodetime{} \stabstep^{-1}}{(n+1)\episodetime{}\stabstep^{-1}}{e^{\CLmat{}s}}{s}
	\end{equation*}
	is employed for upper-bounding the quadratic form $\norm{ \state{\episodetime{}}- e^{\CLmat{}\episodetime{}} \state{0} }{2}^2$ of the normally distributed random matrix $\state{\episodetime{}}- e^{\CLmat{}\episodetime{}} \state{0}$. Furthermore, an application of Ito calculus~\citep{karatzas2012brownian} leads to 
	$$\diff \state{t} \diff \state{t}^\top = \diff \BM{t} \diff \BM{t}^\top = \BMcov{} \diff t .$$ 
	Now, by defining the matrix valued processes
	\begin{equation*}
	\randommatrix{t} = \itointeg{0}{t}{\state{s}\state{s}^{\top}}{s}, ~~~~~~~~~~~~~~~ M_t = \itointeg{0}{t}{\state{s}}{\BM{s}^{\top}} + \itointeg{0}{t}{\state{s}v(s)^{\top} \Bmat{0}^{\top}}{s},
	\end{equation*}
	we obtain
	\begin{eqnarray*}
		\diff \left(\state{t}\state{t}^{\top} \right)
		&=&  \state{t} \diff \state{t}^{\top} + \diff \state{t}\state{t}^{\top} + \diff \state{t} \diff \state{t}^{\top} \\
		&=& \state{t} \left(\left(\CLmat{} \state{t} + \Bmat{0} v(t)\right) \diff t + \diff \BM{t}\right)^{\top} \\
		&+& \left(\left(\CLmat{} \state{t} + \Bmat{0} v(t)\right) \diff t + \diff \BM{t}\right)\state{t}^{\top} + \BMcov{}\diff t \\
		&=& \diff \randommatrix{t} \CLmat{}^{\top} + \CLmat{} \diff \randommatrix{t} + \diff M_t + \diff M_t^{\top} + \BMcov{} \diff t.
	\end{eqnarray*}
	Thus, after integrating both sides of the above equality, we obtain
	\begin{equation*}
	\randommatrix{t} \CLmat{}^{\top} + \CLmat{} \randommatrix{t} + M_t+M_t^{\top} + t \BMcov{} + \state{0}\state{0}^{\top} - \state{t}\state{t}^{\top}=0.
	\end{equation*}
	Because all eigenvalues of $\CLmat{}$ are in the open left half-plane, we can solve the above equation for $\randommatrix{t}$, to get
	\begin{equation} \label{StateEmpCovEq}
	\randommatrix{t} = \itointeg{0}{\infty}{ \exp \left( \CLmat{}s \right) \left[M_t+M_t^{\top} + t \BMcov{} + \state{0}\state{0}^{\top} - \state{t}\state{t}^{\top}\right] \exp \left( \CLmat{}^{\top} s \right) }{s}.
	\end{equation}
	
	Recall that Lemma~\ref{CrossTermLem1} and Lemma~\ref{CrossTermLem2} establish high-probability concentration bounds for the two integrals in the definition of $M_{\episodetime{}}$. Broadly speaking, according to these two lemmas, the difference $$M_{\episodetime{}} - \frac{\episodetime{}^2}{2 \stabstep^2} \Bmat{0} \sum\limits_{n=0}^{\stabstep-1} \dither{n} \dither{n}^{\top} \Bmat{0} $$ 
	grows at most as large as $\episodetime{}^{1/2}$, while the term $\episodetime{} \BMcov{}$ grows at least as fast as $\episodetime{}$ (and also note that the second matrix in the difference above is positive-semidefinite). Hence, if $\episodetime{}$ is large enough, then the latter term $\episodetime{} \BMcov{}$ dominates the former $M_{\episodetime{}}+M_{\episodetime{}}^{\top}$. In addition, the inequality in \eqref{EndStateBoundEq} indicates that the growth rate $\episodetime{}$ dominates $\state{\episodetime{}}\state{\episodetime{}}^{\top}$ as well, since the maximum eigenvalue of the matrix on the right-hand-side of \eqref{EndStateBoundEq} is bounded thanks to stability of the matrix $\CLmat{}$. All in all, these lead to the desired growth rate for the matrix $M_{\episodetime{}}+M_{\episodetime{}}^{\top} + \episodetime{} \BMcov{} + \state{0}\state{0}^{\top} - \state{\episodetime{}}\state{\episodetime{}}^{\top}$, as follows. 
	
	Formally, Lemma~\ref{CrossTermLem1}, Lemma~\ref{CrossTermLem2}, and \eqref{EndStateBoundEq} all together, show that as long as \eqref{MinEigEmpCovLemEq1} holds,
	with probability at least $1-\delta$ we have
	\begin{equation*}
		\eigmin{M_{\episodetime{}}+M_{\episodetime{}}^{\top} + \episodetime{} \BMcov{} + \state{0}\state{0}^{\top} - \state{\episodetime{}}\state{\episodetime{}}^{\top}} \gtrsim \episodetime{} \eigmin{\BMcov{}}.
	\end{equation*}
	Thus, \eqref{StateEmpCovEq} implies that $\eigmin{\randommatrix{\episodetime{}}} \gtrsim \episodetime{} \eigmin{\BMcov{}}$. 
	
	In the rest of this proof, we extend the above lower-bound for $\eigmin{\randommatrix{\episodetime{}}}$ to a similar one for $\eigmin{\empiricalcovmat{\episodetime{}}}$. To proceed towards that end, note that the matrix $\empiricalcovmat{\episodetime{}}$ in~\eqref{RandomLSE1} comprises two signals $\state{t},v(t)$. The empirical covariance matrix of the state signal is $\randommatrix{\episodetime{}}$ that is studied in the above arguments, while for the piecewise-constant randomization signal $v(t)$ in \eqref{DitherTermDef}, we have
	\begin{equation*}
		\itointeg{0}{t}{v(s)v(s)^\top}{s} = \sum\limits_{n=0}^{\stabstep-1} \itointeg{n \episodetime{} \stabstep^{-1}}{(n+1)\episodetime{}\stabstep^{-1}}{\dither{n}\dither{n}^\top}{s} = \frac{ \episodetime{} }{ \stabstep } \sum\limits_{n=0}^{\stabstep-1} \dither{n}\dither{n}^\top.
	\end{equation*}
	Thus, according to Theorem~\ref{Azuma}, similar to \eqref{CrossTermLemAuxEq4} we have
	\begin{equation*}
		\Mnorm{\sum\limits_{n=0}^{\stabstep-1} \dither{n}\dither{n}^\top - \stabstep \dithercoeff{}^2 I_{\controldim}}{2} \lesssim \stabstep^{1/2} \dithercoeff{}^2 \controldim \log^{3/2} \frac{\stabstep \controldim}{\delta},
	\end{equation*}
	with probability at least $1-\delta$, which by denoting
	\begin{equation*}
	H_{\episodetime{}} = \itointeg{0}{\episodetime{}}{\begin{bmatrix} 0_{\statedim} \\ v(s) \end{bmatrix} \begin{bmatrix} 0_{\statedim} \\ v(s) \end{bmatrix}^{\top} }{s} - \episodetime{} \dithercoeff{}^2 \begin{bmatrix} 0_{\statedim \times \statedim} & 0_{\statedim \times \controldim} \\ 0_{\controldim \times \statedim} & I_{\controldim} \end{bmatrix},
	\end{equation*}
	leads to 
	\begin{equation} \label{HboundEq}
		\Mnorm{H_{\episodetime{}}}{2} ~~\lesssim~~ \frac{\episodetime{}}{\stabstep^{1/2}} \dithercoeff{}^2 \controldim \log^{3/2} \frac{\stabstep \controldim}{\delta} ~~\lesssim~~  \dithercoeff{}^2 \controldim \log^{3/2} \frac{\stabstep \controldim}{\delta},
	\end{equation}
	where in the last inequality above, we use $\stabstep \gtrsim \episodetime{}^2$.
	
	Next, using $\statetwo{s}=\left[\state{s}^{\top}, \state{s}^{\top} \Gainmat{}^{\top}+v(s)^{\top}\right]^{\top}$, the matrix $\empiricalcovmat{\episodetime{}}$ can be written as
	\begin{equation} \label{EmpCovDecomposEq}
	\empiricalcovmat{\episodetime{}} = \begin{bmatrix} I_{\statedim} \\ \Gainmat{} \end{bmatrix} \randommatrix{\episodetime{}} \begin{bmatrix} I_{\statedim} \\ \Gainmat{} \end{bmatrix}^{\top} + \episodetime{} \dithercoeff{}^2 \begin{bmatrix} 0_{\statedim \times \statedim} & 0_{\statedim \times \controldim} \\ 0_{\controldim \times \statedim} & I_{\controldim} \end{bmatrix} + F_{\episodetime{}} + H_{\episodetime{}},
	\end{equation}
	where
	\begin{equation*}
	F_{\episodetime{}} =  \itointeg{0}{\episodetime{}}{ \left(\begin{bmatrix} I_{\statedim} \\ \Gainmat{} \end{bmatrix} \state{s} \begin{bmatrix} 0_{\statedim} \\ v(s) \end{bmatrix}^{\top} +\begin{bmatrix} 0_{\statedim} \\ v(s) \end{bmatrix} \state{s}^{\top} \begin{bmatrix} I_{\statedim} \\ \Gainmat{} \end{bmatrix}^{\top}\right) }{s}.
	\end{equation*}
	However, Lemma~\ref{CrossTermLem1} and $\stabstep \gtrsim \episodetime{}^2$ give a high probability upper-bound for the above matrix:
	\begin{equation} \label{FboundEq}
		\Mnorm{F_{\episodetime{}}}{2} \lesssim \left( 1 \vee \Mnorm{\Gainmat{}}{2} \right) \left( \eigmax{\BMcov{}} + \dithercoeff{}^2 \right) \left( \statedim \controldim^{1/2} \episodetime{}^{1/2} \log \frac{\statedim \controldim}{\delta} + \controldim \log^{3/2} \frac{\stabstep \controldim}{\delta} \right).
	\end{equation}
	
	In the sequel, we show that with probability at least $1-\delta$, it holds that
	\begin{equation*}
		\eigmin{\begin{bmatrix} I_{\statedim} \\ \Gainmat{} \end{bmatrix} \randommatrix{\episodetime{}} \begin{bmatrix} I_{\statedim} \\ \Gainmat{} \end{bmatrix}^{\top} + \episodetime{} \dithercoeff{}^2 \begin{bmatrix} 0_{\statedim \times \statedim} & 0_{\statedim \times \controldim} \\ 0_{\controldim \times \statedim} & I_{\controldim} \end{bmatrix}} \gtrsim \episodetime{} \left(\eigmin{\BMcov{}} \wedge \dithercoeff{}^2\right) \left( 1 + \Mnorm{\Gainmat{}}{2}^2 \right)^{-1},
	\end{equation*}
	which, according to \eqref{HboundEq}, \eqref{EmpCovDecomposEq}, and \eqref{FboundEq}, implies the desired result because of the assumptions \eqref{MinEigEmpCovLemEq2} and $\log \left( \statedim \controldim \stabstep \right) \lesssim \episodetime{}^{1/2}$. To show the above least eigenvalue inequality, we use $\eigmin{\randommatrix{\episodetime{}}} \gtrsim \episodetime{} \eigmin{\BMcov{}}$ to obtain
	\begin{equation*}
		\eigmin{\begin{bmatrix} I_{\statedim} \\ \Gainmat{} \end{bmatrix} \randommatrix{\episodetime{}} \begin{bmatrix} I_{\statedim} \\ \Gainmat{} \end{bmatrix}^{\top} + \episodetime{} \dithercoeff{}^2 \begin{bmatrix} 0_{\statedim \times \statedim} & 0_{\statedim \times \controldim} \\ 0_{\controldim \times \statedim} & I_{\controldim} \end{bmatrix}} \gtrsim \episodetime{} \left(\eigmin{\BMcov{}} \wedge \dithercoeff{}^2\right) \eigmin{\begin{bmatrix}
			I_\statedim & ~~\Gainmat{}^\top \\ \Gainmat{} & ~~\Gainmat{} \Gainmat{}^\top + I_\controldim
			\end{bmatrix}}.
	\end{equation*}
	However, the formula for block matrix inversion gives
	\begin{equation*}
		\eigmin{\begin{bmatrix}
			I_\statedim & \Gainmat{}^\top \\ \Gainmat{} & \Gainmat{} \Gainmat{}^\top + I_\controldim
			\end{bmatrix}} = 
		\eigmax{\begin{bmatrix}
			I_\statedim & \Gainmat{}^\top \\ \Gainmat{} & \Gainmat{} \Gainmat{}^\top + I_\controldim
			\end{bmatrix}^{-1}}^{-1} =
		\eigmax{\begin{bmatrix}
			\Gainmat{}^\top \Gainmat{} + I_\statedim & -\Gainmat{}^\top \\ -\Gainmat{} & I_\controldim
			\end{bmatrix}}^{-1},
	\end{equation*}
	that is clearly at least $\left( 1+ \Mnorm{\Gainmat{}}{2}^2 \right)^{-1}$, apart from a constant factor. This completes the proof.
\end{prf}

\subsection{Approximation of true drift parameter by Algorithm~\ref{algo1}} 
\begin{lemm} \label{ParaEstimationLem}
	Suppose that $\estpara{}$ is given by Algorithm~\ref{algo1}. Then, with probability at least $1-\delta$, we have
	\begin{equation} \label{PosteriorBiasEq}
	\Mnorm{\estpara{}-\truth}{2}^2 \lesssim \frac{ \statedim \left(\statedim+\controldim\right) }{\episodetime{} } \frac{\eigmax{\BMcov{}}  }{\eigmin{\BMcov{}} \wedge \dithercoeff{}^2} \left( 1 + \Mnorm{\Gainmat{}}{2}^2 \right) \log \left( \frac{\statedim \controldim \episodetime{}}{\delta} \right).
	\end{equation}
\end{lemm}
\begin{prf}
	First, consider the mean matrix of the Gaussian posterior distribution. Using the data generation mechanism $\diff \state{t} = \truth^\top \statetwo{t} \diff t + \diff \BM{t}$, as well as the definition of $\empiricalcovmat{\episodetime{}}$ in \eqref{RandomLSE2}, we have
	\begin{equation*}
		\empmean{\episodetime{}} = \empiricalcovmat{\episodetime{}}^{-1} \itointeg{0}{\episodetime{}}{\statetwo{s}}{\state{s}^\top} = \empiricalcovmat{\episodetime{}}^{-1} \left( \itointeg{0}{\episodetime{}}{\statetwo{s}\statetwo{s}^\top}{s} \truth +  \itointeg{0}{\episodetime{}}{\statetwo{s}}{\BM{s}^\top} \right) = \truth - \empiricalcovmat{\episodetime{}}^{-1} \left( \truth -  \itointeg{0}{\episodetime{}}{\statetwo{s}}{\BM{s}^\top} \right),
	\end{equation*}
	Now, since the sampling is performed from a Gaussian distribution, the sample $\estpara{}$ from $\posterior{\episodetime{}}$ can be written as $\estpara{\episodetime{}}=\empmean{\episodetime{}}+ \empiricalcovmat{\episodetime{}}^{-1/2} \randommatrix{}$, where $\randommatrix{} \sim \normaldist{0_{(\statedim+\controldim)\times \statedim}}{I_{\statedim+\controldim}}$ is a standard normal random matrix, as defined in the notation. So, for the error matrix, it holds that
	\begin{equation*}
		\Mnorm{\estpara{}-\truth}{2} \leq \Mnorm{\empiricalcovmat{\episodetime{}}^{-1/2}}{2} \left( \Mnorm{\empiricalcovmat{\episodetime{}}^{-1/2} \truth}{2} +  \Mnorm{ \empiricalcovmat{\episodetime{}}^{-1/2} \itointeg{0}{\episodetime{}}{\statetwo{s}}{\BM{s}^\top}}{2} + \Mnorm{\randommatrix{}}{2} \right) .
	\end{equation*}
	The rest of this proof proceeds as follows. We employ Lemma~\ref{MinEigEmpCovLem} to establish useful upper-bounds for $\Mnorm{\empiricalcovmat{\episodetime{}}^{-1/2}}{2}$ and $\Mnorm{\empiricalcovmat{\episodetime{}}^{-1/2} \truth}{2}$, and utilize Lemma~\ref{SelfNormalizedLem} to handle the term $\Mnorm{ \empiricalcovmat{\episodetime{}}^{-1/2} \itointeg{0}{\episodetime{}}{\statetwo{s}}{\BM{s}^\top}}{2}$. So, the proof will go through since the last term consisting of $\randommatrix{}$ is straightforward to upper-bound, thanks to the standard Gaussian distribution that every coordinate of this $\statedim \times \left( \statedim+ \controldim \right)$ random matrix possesses. That is, with probability at least $1-\delta$, we have 
	\begin{equation} \label{PosteriorBiasProof1}
		\Mnorm{\randommatrix{}}{2}^2 ~~\lesssim~~ \statedim (\statedim+\controldim) \log \frac{\statedim (\statedim+\controldim)}{\delta}.
	\end{equation} 
	
	Moreover, according to Lemma~\ref{MinEigEmpCovLem}, with probability at least $1-\delta$ it holds that 
	\begin{equation} \label{PosteriorBiasProof2}
		\Mnorm{\empiricalcovmat{\episodetime{}}^{-1/2}}{2}^2  ~~ \lesssim ~~\frac{ 1 + \Mnorm{\Gainmat{}}{2}^2 }{ \episodetime{} \left(\eigmin{\BMcov{}} \wedge \dithercoeff{}^2\right) } .
	\end{equation}
	
	Next, Lemma~\ref{SelfNormalizedLem} provides the following
	\begin{equation} \label{PosteriorBiasProof3}
		\Mnorm{ \empiricalcovmat{\episodetime{}}^{-1/2} \itointeg{0}{\episodetime{}}{\statetwo{s}}{\BM{s}^\top}}{2}^2 ~~\lesssim~~ \statedim \eigmax{\BMcov{}} \left[ \left( \statedim + \controldim \right) \log \eigmax{\empiricalcovmat{\episodetime{}}} - \log \delta \right].
	\end{equation} 
	
	Finally, by putting the results of \eqref{PosteriorBiasProof1}, \eqref{PosteriorBiasProof2}, and \eqref{PosteriorBiasProof3} together, we obtain the desired result  in \eqref{PosteriorBiasEq}. To see that, note that according to \eqref{EmpCovDecomposEq}, we have 
	$$ \log \eigmax{\empiricalcovmat{\episodetime{}}} \lesssim \log \eigmax{\itointeg{0}{\episodetime{}}{ \state{s} \state{s}^{\top} }{s}} + \log \episodetime{},$$
	while the above quantity grows with the rate $\log \episodetime{}$, according to \eqref{StateEmpCovEq}.	
\end{prf}

\subsection{Eigenvalue perturbation bound for sum of two matrices} 
\begin{lemm} \label{EigPerturbLem}
	Suppose that $M, \erterm{}$ are $\statedim \times \statedim$ matrices, and let $M= \JordanMat^{-1} \Lambda \JordanMat$ be the Jordan diagonalization of $M$. That is, for some positive integer $k$, we have the ${\statedim \times \statedim}$ block-diagonal matrix $\Lambda=\diag{\Lambda_1, \cdots, \Lambda_k}$, where the blocks ${\Lambda_1,\cdots, \Lambda_k}$ are Jordan matrices of the form
	\begin{eqnarray*}
		\Lambda_i= \begin{bmatrix}
			\lambda_i & 1 & 0 & \cdots & 0 & 0 \\
			0 & \lambda_i & 1 & 0 & \cdots & 0 \\
			\vdots & \vdots & \vdots & \vdots & \vdots & \vdots \\
			0 & 0 & \cdots & 0 & \lambda_i & 1 \\
			0 & 0 & 0 & \cdots & 0 & \lambda_i
		\end{bmatrix} \in \C^{\mult{i} \times \mult{i}}.
	\end{eqnarray*}
	Let $\mult{}=\max\limits_{1 \leq i \leq k} \mult{i} \leq \statedim$, and define $\Delta_M (\erterm{})$ as the difference between the largest real-part of the eigenvalues of $M+\erterm{}$, and that of $M$. Then, it holds that
	\begin{equation} \label{EigPerturb}
	 \Delta_M (\erterm{}) \leq \left(1 \vee \mult{} \Mnorm{\erterm{}}{} \cond{\JordanMat}\right)^{1/\mult{}},
	\end{equation}
	where $\cond{\JordanMat}$ is the condition number of $\JordanMat$: $$\cond{\JordanMat}= \frac { \left| \eigmax{\JordanMat^\top \JordanMat} \right|^{1/2} }{ \left| \eigmin{\JordanMat^\top \JordanMat} \right|^{-1/2} }.$$
\end{lemm}
\begin{prf}
	Since the expression on the right-hand-side of \eqref{EigPerturb} is positive, it is enough to consider an eigenvalue $\lambda$  of $M+\erterm{}$ which is not an eigenvalue of $M$, and show that $\Re \left(\lambda\right) - \mosteig{M}$ is less than the expression on the RHS of $\eqref{EigPerturb}$. So, for such $\lambda$, the matrix $ M - \lambda I_{\statedim} $ is non-singular, while $M+\erterm{}-\lambda I_{\statedim}$ is singular. Let the vector $v\neq 0$ be such that $\left( M + \erterm{} - \lambda I_{\statedim} \right)v = 0$, which by Jordan diagonalization above implies that
	\begin{equation} \label{StabPerturbProofEq1}
	v = -\JordanMat^{-1} \left( \Lambda - \lambda I \right)^{-1} \JordanMat \erterm{} v.
	\end{equation}
	Then, $\Lambda=\diag{\Lambda_1, \cdots, \Lambda_k}$ indicates that $\Lambda - \lambda I$ and $\left( \Lambda - \lambda I \right)^{-1}$ are block diagonal, the latter consisting of the blocks $\diag{\left( \Lambda_1 - \lambda I_{\mult{1}} \right)^{-1}, \cdots, \left( \Lambda_k - \lambda I_{\mult{k}} \right)^{-1}}$.
	
	Now, multiplications show that 
	\begin{equation*}
	\left( \Lambda_i - \lambda I_{\mult{i}} \right)^{-1} = - \begin{bmatrix}
	\left( \lambda -\lambda_i \right)^{-1} & \left( \lambda -\lambda_i \right)^{-2}  & \cdots & \left( \lambda -\lambda_i \right)^{-\mult{i}} \\
	0 & \left( \lambda -\lambda_i \right)^{-1} & \cdots & \left( \lambda -\lambda_i \right)^{-\mult{i} +1} \\
	\vdots & \vdots & \vdots & \vdots \\
	0 & \cdots & 0 & \left( \lambda -\lambda_i \right)^{-1}
	\end{bmatrix}.
	\end{equation*}
	Therefore, according to the definition of matrix operator norms in Section \ref{intro}, we obtain
	\begin{equation*}
	\Mnorm{\left( \Lambda_i - \lambda I_{\mult{i}} \right)^{-1}}{2} \leq \mult{}  \left( 1 \vee \left| \lambda - \lambda_i \right|^{-\mult{}}  \right).
	\end{equation*}
	Putting these bounds for the blocks of $\left( \Lambda - \lambda I \right)^{-1}$ together, \eqref{StabPerturbProofEq1} leads to
	\begin{eqnarray*}
		1 &\leq& \Mnorm{\left( \Lambda - \lambda I \right)^{-1}}{2} \Mnorm{\JordanMat}{2} \Mnorm{\JordanMat^{-1}}{2} \Mnorm{\erterm{}}{2} \\
		&\leq& \mult{} \cond{\JordanMat} \Mnorm{\erterm{}}{2} \max\limits_{1 \leq i \leq k} \left( 1 \wedge \left| \lambda - \lambda_i \right|^{\mult{}}  \right)^{-1}\\
		&\leq& \mult{} \cond{\JordanMat} \Mnorm{\erterm{}}{2} \left( 1 \wedge \left( \Re (\lambda) - \mosteig{M} \right)^{\mult{}}  \right)^{-1}.
	\end{eqnarray*}
	To see the last inequality above, note that if $\Re (\lambda) - \mosteig{M}$ is positive, then it is larger than all the terms $\left| \lambda - \lambda_i\right|$, for $i=1,\cdots, k$. Thus, for 
	\begin{equation*}
	\Re \left(\lambda\right)=\mosteig{M+\erterm{}},
	\end{equation*}
	we obtain~\eqref{EigPerturb}.
\end{prf}

\newpage
\section{Lemmas in the Proof of Theorem~\ref{RegretThm}}  \label{append2} 

\subsection{Geometry of drift parameters and optimal policies}
\begin{lemm} \label{OptManifoldLemma}
	For the drift parameter $\para{1}$, and for $X \in \R^{\statedim\times\statedim}, Y \in \R^{\statedim\times\controldim}$, define
	\begin{eqnarray*}
		\boldsymbol{\Delta}_{\para{1}} (X,Y) = \RiccSol{\para{1}} Y + \itointeg{0}{\infty}{ e^{\CLmat{1}^{\top}t} \left[ M \left( X,Y \right)^{\top} \RiccSol{\para{1}} + \RiccSol{\para{1}} M \left( X, Y \right) \right] e^{\CLmat{1}t} \Bmat{1} }{t},
	\end{eqnarray*}
	where $\CLmat{1}=\Amat{1}-\Bmat{1} \Qu^{-1} \Bmat{1}^{\top} \RiccSol{\para{1}}$ and $M\left(X,Y\right)=X - Y \Qu^{-1} \Bmat{1}^{\top} \RiccSol{\para{1}}$. Then, $\boldsymbol{\Delta}_{\para{1}} (X,Y)$ is the directional derivative of $\Bmat{}^{\top}\RiccSol{\para{}}$ at $\para{1}$ in the direction $\left[X,Y\right]$. Importantly, the tangent space of the manifold of matrices $\para{} \in \R^{\statedim \times (\statedim+\controldim)}$ that satisfy $\Bmat{}^{\top}\RiccSol{\para{}}=\Bmat{1}^{\top}\RiccSol{\para{1}}$ at $\para{1}$ consists of all matrices $X,Y$ that $\boldsymbol{\Delta}_{\para{1}} (X,Y) =0$.
\end{lemm}
\begin{prf}
	First, note that according to the Lipschitz continuity of $\RiccSol{\para{}}$ in Lemma~\ref{LipschitzLemma}, the directional derivative exists and is well-defined, as long as $\Mnorm{\RiccSol{\para{1}}}{2} < \infty$. However, Lemma~\ref{ApproxParaLem} provides that $\RiccSol{\para{1}}$ is finite in a neighborhood of $\truth$, and so the required condition holds. Below, we start by establishing the second result to identify the tangent space, and then prove the general result on the directional derivative. 
	
	To proceed, let $\para{} = \para{1} + \epsilon \left[X,Y\right]^{\top}$ be such that $\Bmat{}^{\top} \RiccSol{\para{}} = \Bmat{1}^{\top} \RiccSol{\para{1}}$, and denote $\Optgain{\para{1}}=-\Qu^{-1} \Bmat{1}^{\top} \RiccSol{\para{1}}$. So, the directional derivative of $\RiccSol{\para{1}}$ along the matrix $\left[X,Y\right]^\top$ can be found as follows. First, denoting the closed-loop transition matrix by $\CLmat{}=\Amat{}-\Bmat{} \Qu^{-1} \Bmat{}^{\top} \RiccSol{\para{}}$, since
	\begin{equation*}
	\CLmat{}^{\top}\RiccSol{\para{}} + \RiccSol{\para{}} \CLmat{} + \Qmat + \Optgain{\para{}}^{\top} \Rmat \Optgain{\para{}} =0,
	\end{equation*}
	we have
	\begin{eqnarray*}
		&&\left(\CLmat{1}+\epsilon X+\epsilon Y\Optgain{\para{1}}\right)^{\top} \RiccSol{\para{}} + \RiccSol{\para{}} \left(\CLmat{1}+\epsilon X+\epsilon Y\Optgain{\para{1}}\right) \\
		&=& - \Qx - \Optgain{\para{1}}^{\top} \Qu \Optgain{\para{1}} 
		= \CLmat{1}^{\top} \RiccSol{\para{1}} + \RiccSol{\para{1}} \CLmat{1} .
	\end{eqnarray*}
	For the matrix $\erterm{}=\lim\limits_{\epsilon \to 0} \epsilon^{-1} \left( \RiccSol{\para{}} - \RiccSol{\para{1}} \right)$, the latter result implies that
	\begin{equation*}
	\CLmat{1}^{\top} \erterm{} + \erterm{} \CLmat{1} + \left(X+ Y\Optgain{\para{1}}\right)^{\top} \RiccSol{\para{1}} + \RiccSol{\para{1}} \left(X+ Y\Optgain{\para{1}}\right) =0.
	\end{equation*}
	Then, since all eigenvalues of ${\CLmat{1}}$ are in the open left half-plane, the above Lyapunov equation for $\erterm{}$ leads to the integral form
	\begin{equation*}
	\erterm{} = \itointeg{0}{\infty}{ e^{\CLmat{1}^{\top}t} \left( \left(X+ Y\Optgain{\para{1}}\right)^{\top} \RiccSol{\para{1}} + \RiccSol{\para{1}} \left(X+ Y\Optgain{\para{1}}\right)  \right) e^{\CLmat{1}t} }{t}.
	\end{equation*}
	On the other hand, $\Optgain{\para{}}=-\Qu^{-1} \Bmat{}^{\top} \RiccSol{\para{}}$ gives
	\begin{eqnarray*}
		0 = \lim\limits_{\epsilon \to 0} \frac{1}{\epsilon} \left(\Bmat{}^{\top} \RiccSol{\para{}} - \Bmat{1}^{\top} \RiccSol{\para{1}}\right) 
		= \lim\limits_{\epsilon \to 0} \frac{1}{\epsilon} \left[ \left(\Bmat{}^{\top} - \Bmat{1}^{\top}\right) \RiccSol{\para{}} - \Bmat{1}^{\top} \left(\RiccSol{\para{1}} - \RiccSol{\para{}}\right)\right],
	\end{eqnarray*}
	which, according to the definitions of $\erterm{},M(X,Y)$, implies the desired result about the tangent space of the manifold under consideration.
	
	Next, to establish the more general result on the directional derivative, we use the directional derivative of $\RiccSol{\para{}}$ in \eqref{LipschitzLemProofEq2}:
	\begin{equation*}
		\itointeg{0}{\infty}{ e^{\CLmat{1}^{\top}t} \left( \RiccSol{\para{1}} \left[X + Y \Optgain{\para{1}}\right] + \left[ X + Y \Optgain{\para{1}}\right]^{\top}  \RiccSol{\para{1}} \right) e^{\CLmat{1} t} }{t}.
	\end{equation*}
	Finally, since the directional derivative for $\Bmat{}^\top$ is $Y$, for $\Bmat{}^\top \RiccSol{\para{}}$, by the product rule it is $\boldsymbol{\Delta}_{\para{1}}(X,Y)$, which finishes the proof. 
\end{prf}

\subsection{Regret bounds in terms of deviations in control actions} 
\begin{lemm} \label{GeneralRegretLem}
	Let $\action{t}$ be the action that Algorithm~\ref{algo2} takes at time $t$. Then, for the regret of Algorithm~\ref{algo2}, it holds that
	\begin{eqnarray*}
		\regret{T}{} &\lesssim& \left(\eigmax{\BMcov{}} + \dithercoeff{}^2\right) \episodetime{0} \Mnorm{ \Gainmat{} + \Qu^{-1} \Bmat{0}^\top \RiccSol{\truth} }{2}^2 \\
		&+& \itointeg{\episodetime{0}}{T}{ \norm{ \action{t} + \Qu^{-1} \Bmat{0}^{\top} \RiccSol{\truth} \state{t} }{}^2 }{t} + {x^*_T}^\top \RiccSol{\truth} x^*_T,
	\end{eqnarray*}
	where $x^*_T$ is the terminal state under the optimal trajectory $\optimalpolicy$ in \eqref{OptimalPolicy}.
\end{lemm}
\begin{prf}
	First, denote the optimal linear feedback of $\optimalpolicy$ in \eqref{OptimalPolicy} by $\action{t}=\Optgain{\truth} \state{t}$, where $\Optgain{\truth}=-\Qu^{-1} \Bmat{0} \RiccSol{\truth}$. According to the episodic structure of Algorithm~\ref{algo2}, for $\episodetime{n} \leq t  < \episodetime{n+1}$, denote
	\begin{equation*}
	\Gainmat{t}=-\Qu^{-1} \estB{n}^{\top} \RiccSol{\estpara{n}}.
	\end{equation*}
	We first consider the regret of Algorithm~\ref{algo2} after finishing stabilization by running Algorithm~\ref{algo1}; i.e., for $\episodetime{0} \leq t \leq T$. Fix some small $\epsilon>0$, that we will let decay later. We proceed by finding an approximation of the regret through sampling at times $\episodetime{0}+k\epsilon$, for non-negative integers $k$. To do that, denote $N=\lceil (T-\episodetime{0}) / \eps \rceil$, and define the sequence of policies $\left\{ \policy_i \right\}_{i=0}^N$ according to
	\begin{equation*}
	\policy_i = \begin{cases}
	\action{t} = \Gainmat{t} \state{t} & t < \episodetime{0}+i\eps \\
	\action{t} = \Optgain{\truth} \state{t} & t \geq \episodetime{0}+i\eps
	\end{cases}.
	\end{equation*}
	That is, the policy $\policy_{i}$ switches to the optimal feedback at time $\episodetime{0}+i\epsilon$. So, the zeroth policy $\policy_0$ corresponds to applying the optimal policy $\optimalpolicy$ after stabilization at time $\episodetime{0}$, while the last one $\policy_N$ is nothing but the one in Algorithm~\ref{algo2}, that we denote by $\policy$, for the sake of brevity. As such, we have $\regret{T}{\policy_0}=0$, and the telescopic summation below holds true:
	\begin{equation} \label{SumRegretTelescopEq}
	\regret{T}{\policy} = \sum\limits_{i=0}^{N-1} \left( \regret{T}{\policy_{i+1}} - \regret{T}{\policy_i} \right).
	\end{equation}
	
	Now, to consider the difference $\regret{T}{\policy_{i+1}} - \regret{T}{\policy_i}$, for a fixed $i$ in the range $0 \leq i <N$, denote $t_1=\episodetime{0}+i\epsilon$ and let $\state{t}^{\policy_i},\state{t}^{\policy_{i+1}}$ be the state trajectories under $\policy_i, \policy_{i+1}$, respectively. By definition, we have $\state{t}^{\policy_{i}}=\state{t}^{\policy_{i+1}}$, for all $t \leq t_1$. So, we drop the policy superscript and use $\state{t_1}$ to refer to the states of both of them at time $t_1$. Therefore, as long as $t_1 \leq t < t_1 + \epsilon$, similar to \eqref{GeneralOUEq}, the solutions of the stochastic differential equation are
	\begin{eqnarray*}
		\state{t}^{\policy_{i}} &=& e^{\CLmat{0}(t-t_1)} \state{t_1} + \itointeg{t_1}{t}{e^{\CLmat{0}(t-s)}}{\BM{s}}, \\
		\state{t}^{\policy_{i+1}} &=& e^{\CLmat{}(t-t_1)} \state{t_1} + \itointeg{t_1}{t}{e^{\CLmat{}(t-s)} }{\BM{s}},
	\end{eqnarray*}
	where $\CLmat{0}=\Amat{0}+\Bmat{0}\Optgain{\truth}$ and $\CLmat{}=\Amat{0}+\Bmat{0}\Gainmat{t_1}$ are the closed-loop transition matrices under $\policy_{i}$ and $\policy_{i+1}$, respectively. To work with the above two state trajectories, we define some notations for convenience:
	\begin{eqnarray*}
		M_0 &=& \Qmat+ \Optgain{\truth}^{\top} \Rmat \Optgain{\truth}, \\
		M_1 &=& \Qmat + \Gainmat{t_1}\Rmat \Gainmat{t_1}, \\
		\statethree{t} &=& \state{t}^{\policy_{i+1}} - \state{t}^{\policy_{i}},\\
		\erterm{t} &=& e^{\CLmat{}(t-t_1)} - e^{\CLmat{0}(t-t_1)}.
	\end{eqnarray*}
	Thus, letting 
	$$Z_t = \itointeg{t_1}{t}{ \left[ e^{\CLmat{}(t-s)} - e^{\CLmat{0}(t-s)} \right] }{\BM{s}},$$ 
	it holds that 
	$\statethree{t} = \erterm{t} \state{t_1} + Z_t + \order{\eps^2}$. Further, for the observation signal $\statetwo{t}$ and the cost matrix $Q$ defined in Section~\ref{ProblemSection}, we have
	\begin{eqnarray}
	&&\itointeg{t_1}{t_1+\epsilon}{\left( \instantcost{t}{\policy_{i+1}} - \instantcost{t}{\policy_{i}} \right)}{t} \notag \\
	&=& \itointeg{t_1}{t_1+\epsilon}{ \left[ \left( \state{t}^{\policy_{i}}+\statethree{t} \right)^{\top} M_1 \left( \state{t}^{\policy_{i}}+\statethree{t} \right) - {\state{t}^{\policy_{i}}}^{\top} M_0 \state{t}^{\policy_{i}} \right] }{t} \notag \\
	&=& \itointeg{t_1}{t_1+\epsilon}{ \left[ {\state{t}^{\policy_{i}}}^{\top} S \state{t}^{\policy_{i}} + 2 \statethree{t}^{\top} M_1 \state{t}^{\policy_{i}} + \statethree{t}^{\top} M_1 \statethree{t} \right] }{t}, \label{LocalRegretLemProofEq1}
	\end{eqnarray}
	where $S = M_1-M_0 = \Gainmat{t_1}^{\top} \Rmat \Gainmat{t_1} - \Optgain{\truth}^{\top} \Rmat \Optgain{\truth}$.
	
	On the other hand, for $t \geq t_1+\epsilon$, the evolutions of the state vectors are the same for the two policies and we have
	\begin{equation*}
	\state{t}^{\policy_i} = e^{\CLmat{0}(t-t_1-\epsilon)} \state{t_1+\epsilon}^{\policy_i} + \itointeg{t_1+\epsilon}{t}{e^{\CLmat{0}(t-s)} }{\BM{s}}.
	\end{equation*}
	Therefore, the difference signal becomes 
	\begin{equation*}
		\statethree{t} = e^{\CLmat{0}(t-t_1+\epsilon)} \left[ \state{t_1+\epsilon}^{\policy_{i+1}} - \state{t_1+\epsilon}^{\policy_{i}} \right] = e^{\CLmat{0}(t-t_1+\epsilon)} \statethree{t_1+\epsilon} = e^{\CLmat{0}(t-t_1+\epsilon)} \left[ \erterm{t_1+\epsilon} \state{t_1} + Z_{t_1+\epsilon} \right],
	\end{equation*}
	and we obtain 
	\begin{eqnarray}
	&& \itointeg{t_1+\epsilon}{T}{\left( \instantcost{t}{\policy_{i+1}} - \instantcost{t}{\policy_{i}} \right)}{t} \notag \\
	&=& \itointeg{t_1+\epsilon}{T}{ \left[ \left( \state{t}^{\policy_{i}}+\statethree{t} \right)^{\top} M_0 \left( \state{t}^{\policy_{i}}+\statethree{t} \right) - {\state{t}^{\policy_{i}}}^{\top} M_0 \state{t}^{\policy_{i}} \right] }{t} \notag \\
	&=& \itointeg{t_1+\epsilon}{T}{ \left[ 2 \statethree{t}^{\top} M_0 \state{t}^{\policy_{i}} + \statethree{t}^{\top} M_0 \statethree{t} \right] }{t}. \label{LocalRegretLemProofEq2}
	\end{eqnarray}
	Now, after doing some algebra, the expressions in \eqref{LocalRegretLemProofEq1} and \eqref{LocalRegretLemProofEq2} lead ro the following for small $\epsilon$:
	\begin{equation*}
		 \regret{T}{\policy_{i+1}} - \regret{T}{\policy_i} = \left(\state{t_1}^{\top} {F}_{t_1} \state{t_1} + 2\state{t_1}^{\top} {g}_{t_1} \right) \epsilon + \order{\epsilon^2},
	\end{equation*}
	where
		\begin{eqnarray*}
		{F}_{t_1} &=& S_t + \itointeg{t_1}{T}{ \left(2 H_{t_1}^{\top} e^{\CLmat{0}^{\top}(s-t_1)} \left( \Qx+ \Optgain{\truth}^{\top} \Qu \Optgain{\truth} \right) e^{\CLmat{0}(s-t_1)} \right) }{s} + \order{\epsilon}, \\
		{g}_{t_1} &=& \itointeg{t_1}{T}{ \left(  H_{t_1}^{\top} e^{\CLmat{0}^{\top}(s-t_1)} \left( \Qx+ \Optgain{\truth}^{\top} \Qu \Optgain{\truth} \right) \itointeg{t_1}{s}{ e^{\CLmat{0}(s-u)} }{\BM{u}} \right) }{s} + \order{\epsilon}, \\
		S_{t_1} &=& \Gainmat{t_1}^{\top} \Qu \Gainmat{t_1} - \Optgain{\truth}^{\top} \Qu \Optgain{\truth}, \\
		H_{t_1} &=& \Bmat{0} \left(\Gainmat{t_1}-\Optgain{\truth}\right).
	\end{eqnarray*}
	
	Thus, as $\epsilon$ tends to zero, by \eqref{SumRegretTelescopEq}, we have
	\begin{equation} \label{GenRegThmProofEq0}
	\regret{T}{\policy} - \regret{\episodetime{0}}{\policy} = \itointeg{\episodetime{0}}{T}{ \left(\state{t}^{\top} {F}_{t} \state{t} + 2\state{t}^{\top} {g}_{t} \right) }{t},
	\end{equation}
	where $F_{t},g_{t}$ are the above expressions, without the $\order{\epsilon}$ terms. 
	
	Next, by \eqref{LyapInteg}, the quadratic expression in terms of the matrix ${F}_t$ can be equivalently written with
	\begin{equation*} \label{GenRegThmProofEq1}
		{F}_t = S_t + H_t^{\top} \RiccSol{\truth} + \RiccSol{\truth} H_t - H_t^{\top} \erterm{t} - \erterm{t} H_t,
	\end{equation*}
	where
	\begin{equation*}
		\erterm{t} = \itointeg{T}{\infty}{e^{\CLmat{0}^\top (s-t)} M_0 e^{\CLmat{0}(s-t)} }{s} = e^{\CLmat{0}^\top (T-t)} \RiccSol{\truth} e^{\CLmat{0}(T-t)}.
	\end{equation*}
	Note that in the last equality above, we again used \eqref{LyapInteg}.
	Now, after doing some algebra similar to the expression in \eqref{LyapAuxEq}, we have
	\begin{equation*}
		S_t + H_t^{\top} \RiccSol{\truth} + \RiccSol{\truth} H_t = \left(\Gainmat{t}-\Optgain{\truth}\right)^{\top} \Qu \left(\Gainmat{t}-\Optgain{\truth}\right),
	\end{equation*}
	which in turn implies that
	\begin{equation} \label{GenRegretProofEq2}
		\itointeg{\episodetime{0}}{T}{ \state{t}^\top F_t \state{t} }{t} = \itointeg{\episodetime{0}}{T}{ \norm{\Qu^{1/2} \left( \Gainmat{t}-\Optgain{\truth} \right) \state{t}}{2}^2 }{t} - 2 \itointeg{\episodetime{0}}{T}{ \state{t}^\top \erterm{t} H_t \state{t} }{t} .
	\end{equation}
	
	To study the latter integral, suppose that $x^*_{t}$ is the state trajectory under the optimal policy $\optimalpolicy$ in \eqref{OptimalPolicy}, and define $\xi_t = \state{t}-x^*_t$. 
	Note that \eqref{dynamics} gives $\diff \state{t} = \left( \CLmat{0} + H_t \right) \state{t} \diff t + \diff \BM{t}$, as well as $\diff x^*_t = \CLmat{0} x^*_t \diff t + \diff \BM{t}$. Thus, we get $\diff \xi_t = H_t \state{t} \diff t + \CLmat{0} \xi_t \diff t$, using which, we have the following for $\varphi_t =  e^{-\CLmat{0}t} \xi_t $:
	\begin{equation*}
		\diff \varphi_t =  \diff \left( e^{-\CLmat{0}t} \xi_t \right) = e^{-\CLmat{0}t} \diff \xi_t - \CLmat{0} e^{-\CLmat{0}t} \xi_t \diff t = e^{-\CLmat{0}t} H_t \state{t} \diff t.
	\end{equation*}
	Above, we used the fact that the matrices $e^{-\CLmat{0}t}, {\CLmat{0}}$ commute. So, it holds that
	\begin{eqnarray*}
		\state{t}^\top E_t H_t \state{t} \diff t &=& \state{t}^\top e^{\CLmat{0}^\top (T-t)} \RiccSol{\truth} e^{\CLmat{0}T} \diff \varphi_t \\
		&=& {x^*_t}^\top e^{\CLmat{0}^\top (T-t)} \RiccSol{\truth} e^{\CLmat{0}T} \diff \varphi_t + \varphi_t^\top e^{\CLmat{0}^\top T} \RiccSol{\truth} e^{\CLmat{0}T} \diff \varphi_t \\
		&=& {x^*_t}^\top e^{\CLmat{0}^\top (T-t)} \RiccSol{\truth} e^{\CLmat{0}T} \diff \varphi_t + \frac{1}{2} \diff \left[ \varphi_t^\top e^{\CLmat{0}^\top T} \RiccSol{\truth} e^{\CLmat{0}T} \varphi_t \right].
	\end{eqnarray*}
	In the above expression, writing the solution of the stochastic differential equation as in \eqref{GeneralOUEq}, we have
	\begin{equation*}
		e^{\CLmat{0}(T-t)}x^*_t = x^*_T - \itointeg{t}{T}{e^{\CLmat{0}(T-s)}}{\BM{s}},
	\end{equation*}
	which gives
	\begin{eqnarray*}
		2 \state{t}^\top E_t H_t \state{t} \diff t &=& 2 {x^*_t}^\top e^{\CLmat{0}^\top (T-t)} \RiccSol{\truth} e^{\CLmat{0}T} \diff \varphi_t + \diff \left[ \varphi_t^\top e^{\CLmat{0}^\top T} \RiccSol{\truth} e^{\CLmat{0}T} \varphi_t \right] \\
		&=& - 2 \left( \itointeg{t}{T}{e^{\CLmat{0}(T-s)}}{\BM{s}} \right)^\top \RiccSol{\truth} e^{\CLmat{0}T} \diff \varphi_t \\
		&+& 2 {x^*_T}^\top \RiccSol{\truth} e^{\CLmat{0}T} \diff \varphi_t + \diff \left[ \varphi_t^\top e^{\CLmat{0}^\top T} \RiccSol{\truth} e^{\CLmat{0}T} \varphi_t \right] \\
		&=& - 2 \left( \itointeg{t}{T}{e^{\CLmat{0}(T-s)}}{\BM{s}} \right)^\top \RiccSol{\truth} e^{\CLmat{0}T} \diff \varphi_t \\
		&+& \diff \left[ \left( x^*_T + e^{\CLmat{0}T} \varphi_t \right)^\top \RiccSol{\truth} \left( x^*_T + e^{\CLmat{0}T} \varphi_t \right) \right],
	\end{eqnarray*}
	where the latest equality holds since the differential of the constant term $x^*_T \RiccSol{\truth} x^*_T$ is zero. Next, integration by part yields to
	\begin{eqnarray*}
		\itointeg{\episodetime{0}}{T}{ \left( \itointeg{t}{T}{e^{\CLmat{0}(T-s)}}{\BM{s}} \right)^\top \RiccSol{\truth} e^{\CLmat{0}T} }{\varphi_t} &=& - \left( \itointeg{\episodetime{0}}{T}{e^{\CLmat{0}(T-s)}}{\BM{s}} \right)^\top \RiccSol{\truth} e^{\CLmat{0}T} \varphi_{\episodetime{0}} \\
		&+&\itointeg{\episodetime{0}}{T}{ \varphi_t^\top e^{\CLmat{0}^\top T} \RiccSol{\truth} e^{\CLmat{0}(T-t)} }{\BM{t}} 
	\end{eqnarray*}
	Now, note the following simplifying expressions: First, by definition, we have $x^*_T + e^{\CLmat{0}T} \varphi_T =x^*_T +\xi_T= \state{T}$ and
	\begin{equation*}
		x^*_T + e^{\CLmat{0}T} \varphi_{\episodetime{0}} = x^*_T + e^{\CLmat{0}(T-t)}  (\state{\episodetime{0}}-x^*_{\episodetime{0}}) = e^{\CLmat{0}(T-t)} \state{\episodetime{0}} + \itointeg{\episodetime{0}}{T}{ e^{\CLmat{0}(T-s)} }{\BM{s}},
	\end{equation*}
	is the terminal state vector under the policy $\policy_0$ that switches to the optimal policy $\optimalpolicy$ after the time $\episodetime{0}$, because $ \itointeg{\episodetime{0}}{T}{e^{\CLmat{0}(T-s)}}{\BM{s}}  = x^*_T - e^{\CLmat{0}(T-\episodetime{0})} x^*_{\episodetime{0}}$. Finally, according to Lemma~\ref{SelfNormalizedLem}, we have
	\begin{equation*}
		\Mnorm{\itointeg{\episodetime{0}}{T}{ \varphi_t^\top e^{\CLmat{0}^\top T} \RiccSol{\truth} e^{\CLmat{0}(T-t)} }{\BM{t}}}{2} \lesssim \left( \itointeg{\episodetime{0}}{T}{ \norm{ e^{\CLmat{0}(T-t)} \xi_t  }{2}^2 }{t} \right)^{1/2} \log \itointeg{\episodetime{0}}{T}{ \norm{ e^{\CLmat{0}(T-t)} \xi_t  }{2}^2 }{t}.
	\end{equation*}
	Putting the above bounds together, we obtain
	\begin{equation} \label{GenRegretProofEq3}
		- 2 \itointeg{\episodetime{0}}{T}{ \state{t}^\top \erterm{t} H_t \state{t} }{t} - {x^*_T}^\top \RiccSol{\truth} x^*_T \lesssim \itointeg{\episodetime{0}}{T}{ \norm{ e^{\CLmat{0}(T-t)} \left( \state{t} \right)  }{2}^2 }{t}.
	\end{equation}

	To proceed toward working with the integration of $\state{t}^\top g_t$, employ Fubini Theorem~\citep{karatzas2012brownian} to obtain
	\begin{eqnarray}
	\itointeg{\episodetime{0}}{T}{\state{t}^{\top}\widetilde{g}_t}{t} &=& \itointeg{\episodetime{0}}{T}{\itointeg{t}{T}{   \itointeg{t}{s}{ \left(\state{t}^{\top} H_t^{\top} e^{\CLmat{0}^{\top}(s-t)} M e^{\CLmat{0}(s-u)} \right) }{\BM{u}} }{s} }{t} \notag \\
	&=& \itointeg{\episodetime{0}}{T}{\itointeg{\episodetime{0}}{u}{ \itointeg{u}{T}{ \left(\state{t}^{\top} H_t^{\top} e^{\CLmat{0}^{\top}(s-t)} M e^{\CLmat{0}(s-u)} \right) }{s} }{t} }{\BM{u}}. \notag
	\end{eqnarray}
	
	Now, denote the inner double integral by $\statethree{u}^{\top}$:
	\begin{equation*}
		\statethree{u}^{\top} = \itointeg{\episodetime{0}}{u}{ \itointeg{u}{T}{ \left(\state{t}^{\top} H_t^{\top} e^{\CLmat{0}^{\top}(s-t)} M_0 e^{\CLmat{0}(s-u)} \right) }{s} }{t} = \itointeg{0}{u}{ \left(\state{t}^{\top} \left( \Gainmat{t}-\Optgain{\truth} \right)^{\top} P_{t,u}^{\top} \right) }{t},
	\end{equation*}
	where
	\begin{equation*}
	P_{t,u}^{\top} = \Bmat{0}^{\top} \itointeg{u}{T}{  e^{\CLmat{0}^{\top}(s-t)} M_0 e^{\CLmat{0}(s-u)} }{s}.
	\end{equation*}
	
	Now, let $\regterm{T}=\itointeg{\episodetime{0}}{T}{\norm{\statethree{u}}{2}^2}{u}$, and employ Lemma~\ref{SelfNormalizedLem} to get
	\begin{equation} \label{GenRegretProofEq4}
	\itointeg{\episodetime{0}}{T}{\state{t}^{\top}\widetilde{g}_t}{t} = \itointeg{\episodetime{0}}{T}{ \statethree{u}^{\top} }{\BM{u}} = \order{\regterm{T}^{1/2} \log^{1/2} \regterm{T}}.
	\end{equation}
	Thus, we can work with $\regterm{T}$ to bound the portion of the regret the integral of $\state{t}^\top g_t$ captures. For that purpose, the triangle inequality and Fubini Theorem~\citep{karatzas2012brownian} lead to
	\begin{eqnarray*}
		\regterm{T} &\leq& \itointeg{0}{T}{ \itointeg{0}{u}{ \norm{P_{t,u}  \left(\Gainmat{t}-\Optgain{\truth}\right) \state{t}}{2}^2 }{t} }{u} \\
		&=& \itointeg{0}{T}{ \left( \state{t}^{\top} \left(\Gainmat{t}-\Optgain{\truth}\right)^{\top} \left[\itointeg{t}{T}{ P_{t,u}^{\top} P_{t,u} }{u}\right] \left(\Gainmat{t}-\Optgain{\truth}\right) \state{t} \right) }{t} \\
		&\leq& \eigmax{\itointeg{t}{T}{ P_{t,u}^{\top} P_{t,u} }{u}} \itointeg{0}{T}{ \norm{\left(\Gainmat{t}-\Optgain{\truth}\right) \state{t}}{2}^2 }{t}.
	\end{eqnarray*}
	We can show that $\eigmax{\itointeg{t}{T}{ P_{t,u}^{\top} P_{t,u} }{u}} \lesssim 1 $: 
	\begin{eqnarray*}
		\eigmax{\itointeg{t}{T}{ P_{t,u}^{\top} P_{t,u} }{u}} &\leq& \itointeg{t}{T}{ \Mnorm{P_{t,u}^{\top}}{2}^2 }{u} \\
		&\lesssim& \itointeg{t}{T}{ \Mnorm{ \itointeg{u}{T}{ e^{\CLmat{0}^{\top}(s-t)} M_0 e^{\CLmat{0}(s-u)} }{s} }{2}^2 }{u} \\
		&\leq& \itointeg{t}{T}{ \Mnorm{ e^{\CLmat{0}^{\top}(u-t)}}{2}^2 \Mnorm{ \itointeg{u}{T}{ e^{\CLmat{0}^{\top}(s-u)} M_0 e^{\CLmat{0}(s-u)} }{s} }{2}^2  }{u} \\
		&\leq& \Mnorm{\RiccSol{\truth}}{2}^2 \itointeg{t}{T}{ \Mnorm{ e^{\CLmat{0}^{\top}(u-t)}}{2}^2  }{u} \lesssim 1.
	\end{eqnarray*}
	Above, in the last inequality we use~\eqref{LyapInteg}. Note that the last expression is a bounded constant, since all eigenvalues of $\CLmat{0}$ are in the open left half-plane.
	
	Thus, according to \eqref{GenRegretProofEq4}, it is enough to consider  
	\begin{equation} \label{GenRegretProofEq5}
		\regterm{T} \lesssim \itointeg{0}{T}{ \norm{\left(\Gainmat{t}-\Optgain{\truth}\right) \state{t}}{2}^2 }{t},
	\end{equation}
	in order to bound the portion of the regret that the integration of $\state{t}^\top g_t$ contributes.
	
	While the above discussions apply to the regret during the time interval $\episodetime{0} \leq t \leq T$, we can similarly bound the regret during the stabilization period $0 \leq t \leq \episodetime{0}$. The difference is in the randomization sequence $\dither{n}, n=0,1, \cdots$, which is reflected through the piece-wise constant signal $v(t)$ in \eqref{DitherTermDef}. Therefore, it suffices to add the effect of $v(t)$ to the one of the Wiener process $\BM{t}$, and so $\BMcov{}$ will be replaced with $\left( \BMcov{}+\dithercoeff{}^2 I \right)$:
	\begin{equation} \label{GenRegretProofEq6}
		\regret{\episodetime{0}}{\policy} \leq \left( \eigmax{\BMcov{}}+\dithercoeff{}^2 \right) \episodetime{0} \Mnorm{\Gainmat{}-\Optgain{\truth}}{2}^2.
	\end{equation}
	
	Finally, putting \eqref{GenRegThmProofEq0}, \eqref{GenRegretProofEq2}, \eqref{GenRegretProofEq3}, \eqref{GenRegretProofEq4}, \eqref{GenRegretProofEq5}, and \eqref{GenRegretProofEq6} together, we get the desired result.
\end{prf}

\subsection{Anti-concentration of the posterior precision matrix in Algorithm~\ref{algo2}} 

\begin{lemm} \label{MinPELem}
	In Algorithm~\ref{algo2}, we have the following for the matrix $\empiricalcovmat{\episodetime{n}}$ that is defined in~\eqref{RandomLSE1}:
	\begin{equation*}
	\liminf\limits_{n \to \infty}  \frac{ \eigmin{\empiricalcovmat{\episodetime{n}}} }{ \episodetime{n}^{1/2} } \gtrsim \eigmin{\BMcov{}}.
	\end{equation*}
\end{lemm}
\begin{prf}
	First, we define some notation. Recall that during the time interval $\episodetime{i} \leq t < \episodetime{i+1}$ corresponding to episode $i$, Algorithm \ref{algo2} uses a single parameter estimate $\estpara{i}$. So, for $i=0,1,\cdots$, we use $\randommatrix{i}, \Gainmat{i}, \CLmat{i}$ to denote the sample covariance matrix of the state vectors of episode $i$, and the feedback and closed-loop matrices during episode $i$:
	\begin{eqnarray*}
		\randommatrix{i} &=& \itointeg{\episodetime{i}}{\episodetime{i+1}}{\state{t}\state{t}^\top}{t} , \\
		\Gainmat{i} &=& - \Qu^{-1} \estB{i}^\top \RiccSol{\estpara{i}},\\
		\CLmat{i} &=& \Amat{0}+ \Bmat{0} \Gainmat{i}.
	\end{eqnarray*}
	So, it holds that 
	\begin{equation} \label{EmpCovsRelationEq}
		\empiricalcovmat{\episodetime{n}} = \empiricalcovmat{\episodetime{0}} + \sum\limits_{i=0}^{n-1} \Fmatrix{i} \randommatrix{i} \Fmatrix{i}^\top,
	\end{equation}
	where $\Fmatrix{i} = \begin{bmatrix}
		I_{\statedim} \\
		\Gainmat{i}
		\end{bmatrix}$.
	
	Now, consider the matrix $\randommatrix{i}$. Note that according to the bounded growth rates of the episode (from both above and below) in \eqref{EpochLengthCond}, both $\episodetime{i+1}-\episodetime{i} $ and $\episodetime{i}$ tend to infinity as $i$ grows. Thus, in the sequel, we suppose that the indices $n,i,j,k$ that are used for denoting the episodes, are large enough. Similar to \eqref{StateEmpCovEq}, we have
	\begin{equation*}
		\randommatrix{i} = \itointeg{0}{\infty}{e^{\CLmat{i}s} \left[(\episodetime{i+1}-\episodetime{i}) \BMcov{} + M_i + M_i^\top + \state{\episodetime{i}} \state{\episodetime{i}}^\top - \state{\episodetime{i+1}} \state{\episodetime{i+1}}^\top \right] e^{\CLmat{i}^\top s} }{s},
	\end{equation*}
	where 
	\begin{equation*}
		M_i = \itointeg{\episodetime{i}}{\episodetime{i+1}}{\state{t} }{\BM{t}^\top}.
	\end{equation*}
	So, using the fact that the real-parts of all eigenvalues of $\CLmat{i}$ are negative and so $\state{\episodetime{i+1}}$ can be bounded with $\exp \left( \CLmat{i} (\episodetime{i+1}-\episodetime{i}) \right) \state{\episodetime{i}}$ similar to \eqref{EndStateBoundEq}, as well as Lemma~\ref{CrossTermLem2}, we obtain the following bounds for the largest and smallest eigenvalues of $\randommatrix{i}$
	\begin{eqnarray} 
		\eigmin{\randommatrix{i}} &\gtrsim&  (\episodetime{i+1}-\episodetime{i}) \eigmin{\BMcov{}} \eigmin{\itointeg{0}{\infty}{ e^{\CLmat{i}s} e^{\CLmat{i}^\top s} }{s}}, \label{StateEmpCovBound1}\\ 
		\eigmax{\randommatrix{i}} &\lesssim& (\episodetime{i+1}-\episodetime{i}) \eigmax{\BMcov{}} \itointeg{0}{\infty}{\Mnorm{e^{\CLmat{i}s}}{2}^2}{s}. \label{StateEmpCovBound2}
	\end{eqnarray}
	
	On the other hand, for the parameter estimates at the end of episodes, similar to \eqref{PosteriorBiasEq}, we have
	\begin{equation*}
		\empiricalcovmat{\episodetime{i}}^{1/2} \left(\estpara{i} - \truth\right) = \empiricalcovmat{\episodetime{i}}^{1/2} \left( \estpara{i} - \empmean{\episodetime{i}} \right) + \empiricalcovmat{\episodetime{i}}^{-1/2} \left(  - \truth + \itointeg{0}{\episodetime{i}}{\statetwo{s}}{\BM{s}^\top} \right). 
	\end{equation*}
	Note that by the construction of the posterior $\posterior{\episodetime{i}}$ in \eqref{RandomLSE3}, for the first term we have $\empiricalcovmat{\episodetime{i}}^{1/2} \left( \estpara{i} - \empmean{\episodetime{i}} \right) \sim \normaldist{0}{I_{\statedim+\controldim}}$. Further, for the second term, Lemma~\ref{SelfNormalizedLem} together with \eqref{StateEmpCovBound2} lead to 
	\begin{equation*}
		\Mnorm{\empiricalcovmat{\episodetime{i}}^{-1/2} \left(  - \truth + \itointeg{0}{\episodetime{i}}{\statetwo{s}}{\BM{s}^\top} \right)}{} \lesssim (\statedim+\controldim) \log^{1/2} \episodetime{i}.
	\end{equation*}
	Therefore, we have
	\begin{equation*}
	\Mnorm{\empiricalcovmat{\episodetime{i}}^{1/2} \left(\estpara{i} - \truth\right)}{2} \lesssim (\statedim+\controldim) \log^{1/2} \episodetime{i} .
	\end{equation*}
	However, using the relationship between $\empiricalcovmat{\episodetime{i}}$ and $\randommatrix{0}, \cdots, \randommatrix{i-1}$ in \eqref{EmpCovsRelationEq}, we can write
	\begin{equation*}
		(\statedim+\controldim)^2 \log \episodetime{i} \gtrsim \left(\estpara{i} - \truth\right)^\top \empiricalcovmat{\episodetime{i}} \left(\estpara{i} - \truth\right) \geq \left(\estpara{i} - \truth\right)^\top \left[ \sum\limits_{j=0}^{i-1} \Fmatrix{j} \randommatrix{j} \Fmatrix{j}^\top \right] \left(\estpara{i} - \truth\right),
	\end{equation*}
	which according to the bound in \eqref{StateEmpCovBound1} implies that
	\begin{equation*}
		\eigmin{\BMcov{}} \sum\limits_{j=0}^{i-1} (\episodetime{j+1}-\episodetime{j}) \Mnorm{ \Fmatrix{j}^\top \left( \estpara{i} - \truth \right)}{2}^2 \lesssim (\statedim+\controldim)^2 \log \episodetime{i}.
 	\end{equation*}
 	Clearly, the above result indicates that for $j<i$, it holds that
 	\begin{equation} \label{TangentSpaceProjEq}
 	\Mnorm{ \left( \estpara{i} - \truth \right)^{\top} \Fmatrix{j}}{2}^2 \lesssim \frac{(\statedim+\controldim)^2 \log \episodetime{i}}{\eigmin{\BMcov{}} (\episodetime{j+1}-\episodetime{j})}.
 	\end{equation}

 	Next, we employ Lemma~\ref{OptManifoldLemma} to study hoe Algorithm~\ref{algo2} utilizes Thompson sampling to diversify the matrices $\Fmatrix{1},\Fmatrix{2}, \cdots$. To do so, we consider the randomization the posterior $\posterior{\episodetime{i}}$ applies to the sub-matrix of the parameter estimate corresponding to the input matrix $\estB{i}$. That is, we aim to find the distribution of the random $\statedim \times \controldim$ matrix $\left( \estpara{i} - \empmean{\episodetime{i}} \right)^\top \begin{bmatrix}
	 	0_{\statedim \times \controldim} \\ I_{\controldim}
	 	\end{bmatrix}$. Since $\estpara{i} - \empmean{\episodetime{i}} \sim \normaldist{0}{\empiricalcovmat{\episodetime{i}}^{-1}}$, we have
	 \begin{equation} \label{EffectiveRandomizationEq}
		 \erterm{i} = \begin{bmatrix}
		 0_{\statedim \times \controldim} \\ I_{\controldim}
		 \end{bmatrix}^\top \left( \estpara{i} - \empmean{\episodetime{i}} \right) \sim \normaldist{0}{ \begin{bmatrix}
		 	0_{\statedim \times \controldim} \\ I_{\controldim}
		 	\end{bmatrix}^\top \empiricalcovmat{\episodetime{i}}^{-1} \begin{bmatrix}
		 	0_{\statedim \times \controldim} \\ I_{\controldim}
		 	\end{bmatrix} } = \normaldist{0}{  \left[\empiricalcovmat{\episodetime{i}}^{-1}\right]_{22} },
	 \end{equation}	
	 where $\left[\empiricalcovmat{\episodetime{i}}^{-1}\right]_{22}$ is the $\controldim \times \controldim$ lower-left block in $\empiricalcovmat{\episodetime{i}}^{-1}$:
	 \begin{equation*}
		 \empiricalcovmat{\episodetime{i}}^{-1} = \begin{bmatrix}
		 \left[\empiricalcovmat{\episodetime{i}}^{-1}\right]_{11} & \left[\empiricalcovmat{\episodetime{i}}^{-1}\right]_{12} \\ \left[\empiricalcovmat{\episodetime{i}}^{-1}\right]_{21} &
		 \left[\empiricalcovmat{\episodetime{i}}^{-1}\right]_{22}
		 \end{bmatrix}.
	 \end{equation*}
	 
	 Note that $\empiricalcovmat{\episodetime{0}}$ is a positive semi-definite matrix. Therefore, it suffices to show the desired result for $\empiricalcovmat{\episodetime{n}}-\empiricalcovmat{\episodetime{0}}$, and so in the sequel we remove the effect of $\empiricalcovmat{\episodetime{0}}$ by treating $\episodetime{0}$ as $0$. So, to calculate the inverse $\empiricalcovmat{\episodetime{i}}^{-1}$, we apply block matrix inversion to 
	 \begin{equation*}
		 \empiricalcovmat{\episodetime{i}}
		 = \begin{bmatrix}
		 \left[\empiricalcovmat{\episodetime{i}}\right]_{11} & 
		 \left[\empiricalcovmat{\episodetime{i}}\right]_{12} \\
		 \left[\empiricalcovmat{\episodetime{i}}\right]_{21} & 
		 \left[\empiricalcovmat{\episodetime{i}}\right]_{22} 
		 \end{bmatrix}= \begin{bmatrix}
		 \sum\limits_{j=0}^{i-1} \randommatrix{j} & 
		 \sum\limits_{j=0}^{i-1} \randommatrix{j} \Gainmat{j}^\top \\
		 \sum\limits_{j=0}^{i-1} \Gainmat{j} \randommatrix{j} & 
		 \sum\limits_{j=0}^{i-1} \Gainmat{j} \randommatrix{j} \Gainmat{j}^{\top}
		 \end{bmatrix},
	 \end{equation*}
	 and denote 
     \begin{equation*}
         \Omega_i =  \left[\empiricalcovmat{\episodetime{i}}\right]_{22} - \left[\empiricalcovmat{\episodetime{i}}\right]_{21} \left[\empiricalcovmat{\episodetime{i}}\right]_{11}^{-1} \left[\empiricalcovmat{\episodetime{i}}\right]_{12}, 
     \end{equation*}
     to obtain
	 \begin{eqnarray*}
		 \left[\empiricalcovmat{\episodetime{i}}^{-1}\right]_{11} &=& \left[\empiricalcovmat{\episodetime{i}}\right]_{11}^{-1} + \left[\empiricalcovmat{\episodetime{i}}\right]_{11}^{-1} \left[\empiricalcovmat{\episodetime{i}}\right]_{12} 
		 \Omega_i^{-1}
		 \left[\empiricalcovmat{\episodetime{i}}\right]_{21} \left[\empiricalcovmat{\episodetime{i}}\right]_{11}^{-1} ,\\ 
		 \left[\empiricalcovmat{\episodetime{i}}^{-1}\right]_{12} &=& - \left[\empiricalcovmat{\episodetime{i}}\right]_{11}^{-1} \left[\empiricalcovmat{\episodetime{i}}\right]_{12} 
		 \Omega_i^{-1},\\
		 \left[\empiricalcovmat{\episodetime{i}}^{-1}\right]_{22} &=& \Omega_i^{-1}. 
	 \end{eqnarray*}
 	The smallest eigenvalue of $\empiricalcovmat{\episodetime{i}}$ is related to that of $\Omega_i$. On one hand, since $\Omega_i^{-1}$ is a sub-matrix of $\empiricalcovmat{\episodetime{i}}^{-1}$, it holds that $\eigmax{\Omega_i^{-1}} \leq \eigmax{\empiricalcovmat{\episodetime{i}}^{-1}}$, which implies that $\eigmin{\Omega_i} \geq \eigmin{\empiricalcovmat{\episodetime{i}}}$. Now, we show that the inequality holds in the opposite direction as well, modulo a constant factor. Suppose that $\nu \in \R^{\statedim+\controldim}$ is a unit vector, $\nu = [\nu_1^\top, \nu_2^\top]^\top$, $\nu_1 \in \R^\statedim$, and $\nu_2 \in \R^\controldim$. So, after doing some algebra as follows, we have
 	\begin{eqnarray*}
	 	\nu^\top \empiricalcovmat{\episodetime{i}} \nu 
	 	&=& \nu_1^\top \left[\empiricalcovmat{\episodetime{i}}\right]_{11} \nu_1 
	 	+ 2 \nu_1^\top \left[\empiricalcovmat{\episodetime{i}}\right]_{12} \nu_2
	 	+ \nu_2^\top \left[\empiricalcovmat{\episodetime{i}}\right]_{22} \nu_2  \\
	 	&=& \nu_1^\top \left[\empiricalcovmat{\episodetime{i}}\right]_{11} \nu_1 
	 	+ 2 \nu_1^\top \left[\empiricalcovmat{\episodetime{i}}\right]_{11} \left[\empiricalcovmat{\episodetime{i}}\right]_{11}^{-1} \left[\empiricalcovmat{\episodetime{i}}\right]_{12} \nu_2 \\
	 	&+& \nu_2^\top \left[\empiricalcovmat{\episodetime{i}}\right]_{21} \left[\empiricalcovmat{\episodetime{i}}\right]_{11}^{-1}
	 	\left[\empiricalcovmat{\episodetime{i}}\right]_{11}
	 	\left[\empiricalcovmat{\episodetime{i}}\right]_{11}^{-1} \left[\empiricalcovmat{\episodetime{i}}\right]_{12} \nu_2 \\
	 	&+& \nu_2^\top \left[\empiricalcovmat{\episodetime{i}}\right]_{22} \nu_2 
	 	- \nu_2^\top \left[\empiricalcovmat{\episodetime{i}}\right]_{21} \left[\empiricalcovmat{\episodetime{i}}\right]_{11}^{-1} \left[\empiricalcovmat{\episodetime{i}}\right]_{12} \nu_2 \\
	 	&=& \left(\nu_1 + \left[\empiricalcovmat{\episodetime{i}}\right]_{11}^{-1} \left[\empiricalcovmat{\episodetime{i}}\right]_{12}  \nu_2\right)^\top \left[\empiricalcovmat{\episodetime{i}}\right]_{11} 
	 	\left(\nu_1 + \left[\empiricalcovmat{\episodetime{i}}\right]_{11}^{-1} \left[\empiricalcovmat{\episodetime{i}}\right]_{12}  \nu_2\right) \\
	 	&+& \nu_2 \Omega_i \nu_2.
	 	\end{eqnarray*}	
 	
 	For the matrix $\left[\empiricalcovmat{\episodetime{i}}\right]_{11} = \sum\limits_{j=0}^{i-1} \randommatrix{j}$, the smallest eigenvalue lower bounds in \eqref{StateEmpCovBound1} lead to $\eigmin{\left[\empiricalcovmat{\episodetime{i}}\right]_{11}} \gtrsim \episodetime{i} \eigmin{\BMcov{}}$. Thus, in order to show the desired smallest eigenvalue result for $\empiricalcovmat{\episodetime{n}}$, it suffices to consider unit vectors $\nu$ for which $\norm{\nu_1 + \left[\empiricalcovmat{\episodetime{i}}\right]_{11}^{-1} \left[\empiricalcovmat{\episodetime{i}}\right]_{12}  \nu_2}{2} \lesssim \episodetime{i}^{-1/4}$ holds. For such unit vectors $\nu$, the expressions $\left[\empiricalcovmat{\episodetime{i}}\right]_{11}=\sum\limits_{j=0}^{i-1} \randommatrix{j}$ and $\left[\empiricalcovmat{\episodetime{i}}\right]_{12}=\sum\limits_{j=0}^{i-1} \randommatrix{j} \Gainmat{j}^\top$,   as well as Lemma~\ref{ApproxParaLem} that indicates that the matrices $\Gainmat{j}$ are bounded, $\norm{\nu_2}{2}$ needs to be bounded away from zero since $\norm{\nu_1}{2}^2 + \norm{\nu_2}{2}^2 = \norm{\nu}{2}^2=1$. Thus, we have
 	\begin{equation} \label{OmegaEigenvaluesEq}
	 	\eigmin{\Omega_i} \geq \eigmin{\empiricalcovmat{\episodetime{i}}} \gtrsim \eigmin{\Omega_i}.
 	\end{equation} 
 	Otherwise, the desired result about the eigenvalue of $\empiricalcovmat{\episodetime{n}}$ holds true.

 	By simplifying the following expression, we get
 	\begin{eqnarray*}
	 	&& \sum\limits_{j=0}^{i-1} \left( \Gainmat{j}^\top - \left[\empiricalcovmat{\episodetime{i}}\right]_{11}^{-1} \left[\empiricalcovmat{\episodetime{i}}\right]_{12}  \right)^\top \randommatrix{j} \left( \Gainmat{j}^{\top} - \left[\empiricalcovmat{\episodetime{i}}\right]_{11}^{-1} \left[\empiricalcovmat{\episodetime{i}}\right]_{12}  \right) \\
	 	&=& \sum\limits_{j=0}^{i-1} \Gainmat{j} \randommatrix{j} \Gainmat{j}^{\top} 
	 	- \sum\limits_{j=0}^{i-1} \Gainmat{j} \randommatrix{j}  \left[\empiricalcovmat{\episodetime{i}}\right]_{11}^{-1} \left[\empiricalcovmat{\episodetime{i}}\right]_{12}  \\
	 	&-& \sum\limits_{j=0}^{i-1}  \left(\left[\empiricalcovmat{\episodetime{i}}\right]_{11}^{-1} \left[\empiricalcovmat{\episodetime{i}}\right]_{12}\right)^\top \randommatrix{j} \Gainmat{j}^{\top} 
	 	+ \sum\limits_{j=0}^{i-1}  \left( \left[\empiricalcovmat{\episodetime{i}}\right]_{11}^{-1} \left[\empiricalcovmat{\episodetime{i}}\right]_{12}  \right)^\top \randommatrix{j}  \left[\empiricalcovmat{\episodetime{i}}\right]_{11}^{-1} \left[\empiricalcovmat{\episodetime{i}}\right]_{12} \\
	 	&=& \left[\empiricalcovmat{\episodetime{i}}\right]_{22} 
	 	- \left[\empiricalcovmat{\episodetime{i}}\right]_{21} \left[\empiricalcovmat{\episodetime{i}}\right]_{11}^{-1} \left[\empiricalcovmat{\episodetime{i}}\right]_{12} 
	 	- \left(\left[\empiricalcovmat{\episodetime{i}}\right]_{11}^{-1} \left[\empiricalcovmat{\episodetime{i}}\right]_{12}\right)^\top 
	 	\left[\empiricalcovmat{\episodetime{i}}\right]_{21} \\
	 	&+& \left( \left[\empiricalcovmat{\episodetime{i}}\right]_{11}^{-1} \left[\empiricalcovmat{\episodetime{i}}\right]_{12}  \right)^\top \left[\empiricalcovmat{\episodetime{i}}\right]_{11} \left[\empiricalcovmat{\episodetime{i}}\right]_{11}^{-1} \left[\empiricalcovmat{\episodetime{i}}\right]_{12} \\
	 	&=& \Omega_i.
 	\end{eqnarray*}
 	However, we have
 	\begin{eqnarray*}
	 	\Gainmat{j}^{\top} - \left[\empiricalcovmat{\episodetime{i}}\right]_{11}^{-1} \left[\empiricalcovmat{\episodetime{i}}\right]_{12} = \left[\empiricalcovmat{\episodetime{i}}\right]_{11}^{-1} \left( 
	 	\left[\empiricalcovmat{\episodetime{i}}\right]_{11} \Gainmat{j}^{\top}  -  \sum\limits_{k=0}^{i-1} \randommatrix{k} \Gainmat{k}^\top \right)
	 	= \left[\empiricalcovmat{\episodetime{i}}\right]_{11}^{-1} 
	 	\sum\limits_{k=0}^{i-1} \randommatrix{k} \left( \Gainmat{j}-\Gainmat{k} \right)^\top,
 	\end{eqnarray*}
 	i.e.,
 	\begin{equation} \label{OmegaExpressionEq}
	 	\Omega_i =\sum\limits_{j=0}^{i-1} \left( \left[\empiricalcovmat{\episodetime{i}}\right]_{11}^{-1} 
	 	\sum\limits_{k=0}^{i-1} \randommatrix{k} \left( \Gainmat{j}-\Gainmat{k} \right)^\top \right)^\top \randommatrix{j} \left(\left[\empiricalcovmat{\episodetime{i}}\right]_{11}^{-1} 
	 	\sum\limits_{k=0}^{i-1} \randommatrix{k} \left( \Gainmat{j}-\Gainmat{k} \right)^\top\right).
 	\end{equation}
 	
 	We use the above expression to relate the matrices $\Omega_0, \Omega_1, \cdots$ to each others. First, let $\RandMat{0},\RandMat{1}, \cdots$ be a sequence of independent random $\controldim \times \statedim$ matrices with standard normal distribution
 	\begin{equation} \label{RandMatSeqEq}
	 	\RandMat{i} \sim \normaldist{0_{\controldim \times \statedim}}{I_\controldim}.
 	\end{equation}
 	Then, since $\left[ \empiricalcovmat{\episodetime{i}}^{-1} \right]_{22}=\Omega_i^{-1}$ and \eqref{EffectiveRandomizationEq}, we can let $\erterm{i} = \Omega_i^{-1/2} \RandMat{i}$. Further, for $j,k=0,1,\cdots$, denote the $\Bmat{}$-part of the differences $\empmean{k}-\empmean{j}$ by 
 	\begin{equation*}
	 	H_{kj} = \left[ 0_{\controldim \times \statedim} , I_{\controldim} \right] \left( \empmean{k} - \empmean{j} \right).
 	\end{equation*}
 	Note that the above result together with \eqref{EffectiveRandomizationEq} give
 	\begin{equation*}
 	\left[ 0_{\controldim \times \statedim} , I_{\controldim} \right] \left( \estpara{k}-\estpara{j} \right) = H_{kj} + \Omega_k^{-1/2} \RandMat{k} - \Omega_j^{-1/2} \RandMat{j}.
 	\end{equation*}
 	We will show in the sequel that the above normally distributed random matrices are the effective randomizations that Thompson sampling Algorithm~\ref{algo2} applies for exploration. For that purpose, using the directional derivatives and the optimality manifolds in Lemma~\ref{OptManifoldLemma}, we calculate $\Gainmat{k}-\Gainmat{j}$ according to $H_{kj} + \Omega_k^{1/2} \RandMat{k} - \Omega_j^{1/2} \RandMat{j}$. Plugging \eqref{TangentSpaceProjEq} in the expression for $\boldsymbol{\Delta}_{\para{1}}(X,Y)$ in Lemma~\ref{OptManifoldLemma} for
	 \begin{equation*}
		 \left[X,Y\right]=\estpara{k}^\top - \estpara{j}^\top,
	 \end{equation*}
	 we have
	 \begin{eqnarray*}
	 	\Mnorm{\itointeg{0}{\infty}{ e^{\CLmat{j}^{\top}t} \left[ \Fmatrix{j}^\top \left(\estpara{k}-\estpara{j}\right) \RiccSol{\estpara{j}} + \RiccSol{\estpara{j}} \left(\estpara{k}-\estpara{j}\right)^\top \Fmatrix{j} \right] e^{\CLmat{j}t} }{t}}{2}^2 
	 	\lesssim \frac{(\statedim+\controldim)^2 \log \episodetime{k}}{\eigmin{\BMcov{}} (\episodetime{j+1}-\episodetime{j})},
	 \end{eqnarray*}
 	and 
 	\begin{equation*}
	 	\RiccSol{\estpara{j}}Y= \RiccSol{\estpara{j}}\left( \estpara{k} - \estpara{j} \right)^\top \begin{bmatrix}
	 	0_{\statedim \times \controldim} \\ I_{\controldim}
	 	\end{bmatrix} = \RiccSol{\estpara{j}} \left( H_{kj} + \Omega_k^{-1/2} \RandMat{k} - \Omega_j^{-1/2} \RandMat{j} \right)^\top.
 	\end{equation*}
 	Putting the above two portions of $\boldsymbol{\Delta}_{\para{1}}(X,Y)$ together, since $\RandMat{k},\RandMat{j}$ are independent and standard normal random matrices, \eqref{OmegaEigenvaluesEq} implies that the latter portion of $\boldsymbol{\Delta}_{\para{1}}(X,Y)$ is the dominant one. Thus, according to Lemma~\ref{OptManifoldLemma} and the expression for the optimal feedbacks in \eqref{OptimalPolicy}, we can approximate $\Gainmat{k}-\Gainmat{j}$ in \eqref{OmegaExpressionEq} by
 	\begin{equation*}
	 	- \Qu^{-1} \left( H_{kj} + \Omega_k^{-1/2} \RandMat{k} - \Omega_j^{-1/2} \RandMat{j} \right) \RiccSol{\estpara{j}}.
 	\end{equation*}
 	We use the above approximation for the matrix $\left[\empiricalcovmat{\episodetime{i}}\right]_{11}^{-1} 
 	\sum\limits_{k=0}^{i-1} \randommatrix{k} \left( \Gainmat{j}-\Gainmat{k} \right)^\top$ in \eqref{OmegaExpressionEq}, letting the episode number $i$ grow. So, the following expression captures the limit behavior of the least eigenvalue of $\empiricalcovmat{\episodetime{n}}$ in Algorithm~\ref{algo2}:
 	\begin{eqnarray} \label{OmegaLimitEq}
 		&\lim\limits_{n \to \infty}& \frac{\Qu \Omega_n \Qu }{\episodetime{n}^{1/2} \eigmin{\BMcov{}}} = \lim\limits_{n \to \infty} \sum\limits_{j=0}^{n-1} \left( \sum\limits_{k=0}^{n-1} \widetilde{\randommatrix{k}} \RiccSol{\estpara{j}} \left( \frac{H_{kj} + \Omega_k^{-1/2} \RandMat{k} - \Omega_j^{-1/2} \RandMat{j}}{\episodetime{n}^{-1/4}} \right)^\top  \right)^\top \notag \\
 		&& \frac{\episodetime{j+1}-\episodetime{j}}{\episodetime{n}\eigmin{\BMcov{}}} \frac{\randommatrix{j}}{\episodetime{j+1}-\episodetime{j}} \left( \sum\limits_{k=0}^{n-1} \widetilde{\randommatrix{k}} \RiccSol{\estpara{j}} \left( \frac{H_{kj} + \Omega_k^{-1/2} \RandMat{k} - \Omega_j^{-1/2} \RandMat{j}}{\episodetime{n}^{-1/4}} \right)^\top  \right),
 	\end{eqnarray}
 	where 
 	\begin{equation*}
	 	\widetilde{\randommatrix{k}} = \left[\sum\limits_{i=0}^{n-1} \randommatrix{i}\right]^{-1} 
	 	\randommatrix{k}.
 	\end{equation*}
 	
 	The equation in \eqref{OmegaLimitEq} provides the limit behavior of the randomized exploration Algorithm~\ref{algo2} performs for learning to control the diffusion process. More precisely, it shows the roles of the random samples from the posteriors through the random matrices $\Omega_k^{-1/2} \RandMat{k}$, for $k=0, \cdots, n-1$, which render the limit matrix in \eqref{OmegaLimitEq} a positive definite one, as described below.
 	
 	Note that since $\sum\limits_{i=0}^{n-1} \widetilde{\randommatrix{i}}=I_\statedim$, the expression 
 	$$\sum\limits_{k=0}^{n-1} \widetilde{\randommatrix{k}} \RiccSol{\estpara{j}} \left( \frac{H_{kj} + \Omega_k^{-1/2} \RandMat{k} - \Omega_j^{-1/2} \RandMat{j}}{\episodetime{n}^{-1/4}} \right)^\top$$ 
 	is a weighted average of the random matrices $ \episodetime{n}^{1/4} \left({H_{kj} + \Omega_k^{-1/2} \RandMat{k} - \Omega_j^{-1/2} \RandMat{j}} \right)^\top$. Moreover, according to the discussions leading to \eqref{StateEmpCovBound1} and \eqref{StateEmpCovBound2}, the matrix $(\episodetime{j+1} - \episodetime{j})^{-1} \randommatrix{j}$ converge as $j$ grows to a positive definite matrix, for which all eigenvalues are larger than $\eigmin{\BMcov{}}$, modulo a constant factor. On the other hand, because the lengths of the episodes satisfies the bounded growth rates in \eqref{EpochLengthCond}, the ratios $\episodetime{n}^{-1} (\episodetime{j+1}-\episodetime{j})$ are bounded from above and below by $\underline{\ssconstant}^{n-j}$ and $\overline{\ssconstant} (\overline{\ssconstant}+1)^{n-j-1}$, and their sum over $j=0, \cdots, n-1$ is $1$. A similar property of boundedness from above and below applies to $\episodetime{j}^{-1/4} \episodetime{n}^{-1/4}$. So, the expression on the right-hand-side of  \eqref{OmegaLimitEq} is in fact a weighted average of
 	\begin{equation*}
	 	\sum\limits_{k=0}^{n-1} \widetilde{\randommatrix{k}} \RiccSol{\estpara{j}} \left( \frac{H_{kj} + \Omega_k^{-1/2} \RandMat{k} - \Omega_j^{-1/2} \RandMat{j}}{\episodetime{n}^{-1/4}} \right)^\top,
 	\end{equation*}
 	for $j=0, \cdots, n-1$.
 	
 	Note that by the distribution of the random matrices in \eqref{RandMatSeqEq}, all rows of $ \episodetime{n}^{1/4} \left({H_{kj} + \Omega_k^{-1/2} \RandMat{k} - \Omega_j^{-1/2} \RandMat{j}} \right)^\top$ are independent normal random vectors, implying that these random matrices are almost surely full-rank. Therefore, $\episodetime{n}^{-1/2} \Omega_n$ converges to a positive definite random matrix, which according to \eqref{OmegaEigenvaluesEq} implies the desired result.

\end{prf}

\newpage
\section{Further Auxiliary Lemmas} \label{append3}

\subsection{Stochastic inequality for continuous-time self-normalized martingales} 
\begin{lemm} \label{SelfNormalizedLem}
	Let $\statetwo{t}=\left[\state{t}^\top,\action{t}^\top\right]^\top$ be the observation signal and $\empiricalcovmat{t}$ be as in~\eqref{RandomLSE1}. Then, for the stochastic integral $\randommatrix{t}=\itointeg{0}{t}{ \statetwo{s} }{\BM{s}^{\top}}$, with probability at least $1-\delta$ we have
	\begin{equation*}
	\eigmax{ \randommatrix{t}^{\top} \empiricalcovmat{t}^{-1} \randommatrix{t}} \lesssim { \statedim \eigmax{\BMcov{}} \left[ \log \left( \statedim^2 ~\det {\empiricalcovmat{t}} \right) - \log \left( \delta^2 ~\det{\empiricalcovmat{0}} \right) \right] }.
	\end{equation*}
\end{lemm}
\begin{prf}
	Let the standard basis of $\R^{\statedim}$ be $\basis{1} , \cdots, \basis{\statedim}$, and choose an arbitrary one of them; $\basis{j}$. In addition, fix an arbitrary vector $v \in \R^{\statedim + \controldim}$, and consider the one-dimensional process
	\begin{equation*}
		\varphi_t = \varphi_t \left( v , \basis{j} \right) = \exp \left( v^\top \randommatrix{t} \basis{j} - \frac{1}{2} v^\top \left[ \empiricalcovmat{t} - \empiricalcovmat{0} \right] v ~ \basis{j}^\top \BMcov{} \basis{j} \right).
	\end{equation*} 
	
	We show that $\E{\varphi_t} \leq 1$. For that purpose, we approximate the integrals over the interval $[0,t]$ (that constitute $\randommatrix{t}$ and $\empiricalcovmat{t}$) by summations over a mesh of $n$ equally distanced points in the interval, and then let $n \to \infty$. So, let $\epsilon= t/n $, and for $k=0,1,\cdots, n-1$, consider
	\begin{equation*}
	\xi_k = \xi_k \left( v , \basis{j} \right)= \exp \left( v^\top \statetwo{k \epsilon} \left[ \BM{(k+1) \epsilon} - \BM{k \epsilon} \right]^\top \basis{j} - \frac{1}{2} v^\top \statetwo{k \epsilon} \statetwo{k \epsilon}^\top v ~\epsilon ~\basis{j}^\top \BMcov{} \basis{j} \right).
	\end{equation*}
	
	To proceed, let $\filter{k}$ be the sigma-field generated by the Wiener process up to time $k\epsilon$:
	\begin{equation*}
		\filter{k} = \sigfield{\BM{s}, 0 \leq s \leq k\epsilon}.
	\end{equation*}
	Note that the Wiener increment $\left[ \BM{(k+1) \epsilon} - \BM{k \epsilon} \right]^\top \basis{j}$ is a mean-zero Gaussian random variable of variance $\epsilon \basis{j}^\top \BMcov{} \basis{j}$. So, conditioned on $\filter{k}$, the independence of the increments of $\BM{t}$ renders 
	\begin{equation*}
		\E{ \exp \left( v^\top \statetwo{k \epsilon} \left[ \BM{(k+1) \epsilon} - \BM{k \epsilon} \right]^\top \basis{j} \right) \Big| \filter{k}} 
	\end{equation*}
	being the moment generating function of $\left[ \BM{(k+1) \epsilon} - \BM{k \epsilon} \right]^\top \basis{j}$ at $v^\top \statetwo{k \epsilon}$, which is
	\begin{equation*}
		\exp \left( \frac{1}{2} \epsilon \basis{j}^\top \BMcov{} \basis{j} \left[ v^\top \statetwo{k \epsilon} \right]^2  \right).
	\end{equation*}
	Therefore, it holds that $\E{ \xi_k \Big| \filter{k}} \leq 1$. This inequality, the fact that $\xi_{k-1}$ is $\filter{k}$--measurable, and the law of smoothing for conditional expectation, all together, yield to
	\begin{equation*}
		\E{ \prod\limits_{k=0}^{n-1} \xi_k } = \E{ \E{ \prod\limits_{k=0}^{n-1} \xi_k \Big| \filter{n-1} } } = \E{ \prod\limits_{k=0}^{n-2} \xi_k ~\E{ \xi_{n-1} \Big| \filter{n-1} } } \leq \E{ \prod\limits_{k=0}^{n-2} \xi_k }.
	\end{equation*}
	Repeating this procedure, we obtain the bound $\E{ \prod\limits_{k=0}^{n-1} \xi_k } \leq 1 $. Now, as $n$ grows to infinity, $\epsilon$ shrinks, and $\prod\limits_{k=0}^{n-1} \xi_k$ converges to $\varphi_t$. Thus, applying Fatou's Lemma \citep{baldi2017stochastic}, it leads to our desired result for $\varphi_t$:
	\begin{equation} \label{1ExpUppBoundEq}
		\E{\varphi_t} \leq 1.
	\end{equation}
	
	Recall that the vector $v$ above is arbitrary. Now, we let $v$ be a realization of a Gaussian random vector, and develop a modified version of the method of Gaussian mixtures~\citep{hafshejani2024learning}. To that end, suppose that $V \in \R^{\statedim + \controldim}$ has a multivariate normal probability law of mean $0 \in \R^{\statedim + \controldim}$ and covariance matrix $ \basis{j}^\top \BMcov{} \basis{j} ~\empiricalcovmat{0}^{-1}$, and is independent of everything else. So, writing the conditional expectation of $\varphi_t \left( V , \basis{j}\right)$ using the probability density of $V$, we obtain
	\begin{equation*}
		\E{ \varphi_t \left( V, \basis{j} \right) \Big| \sigfield{\BM{s}, 0 \leq s \leq t} } = \sqrt{ \frac{\det \empiricalcovmat{0}}{(2 \pi)^{\statedim+\controldim}} }  \itointeg{\R^{\statedim+ \controldim}}{}{ \varphi_t \left( v , \basis{j} \right) \exp \left( v^\top \empiricalcovmat{0} v ~ \basis{j}^\top \BMcov{} \basis{j} \right) }{v}.
	\end{equation*}
	
	Expanding the integrand according to the definition of $\varphi_t$, it becomes
	\begin{eqnarray*}
		\varphi_t \left( v , \basis{j} \right) \exp \left( v^\top \empiricalcovmat{0} v ~ \basis{j}^\top \BMcov{} \basis{j} \right) = \exp \left( v^\top \randommatrix{t} \basis{j} - \frac{1}{2} v^\top \empiricalcovmat{t}  v ~ \basis{j}^\top \BMcov{} \basis{j} \right).
	\end{eqnarray*}
	
	Next, we convert the exponent term on RHS above to the probability density function of a multivariate normal. By adding $-\frac{1}{2}\basis{j}^\top \randommatrix{t}^\top \left(  \basis{j}^\top \BMcov{} \basis{j} ~ \empiricalcovmat{t} \right)^{-1} \randommatrix{t} \basis{j} $ to the exponent, it conveys 
	\begin{eqnarray*}
		& & v^\top \randommatrix{t} \basis{j} - \frac{1}{2} v^\top \empiricalcovmat{t}  v ~ \basis{j}^\top \BMcov{} \basis{j} - \frac{1}{2} \basis{j}^\top \randommatrix{t}^\top \left(  \basis{j}^\top \BMcov{} \basis{j} ~ \empiricalcovmat{t} \right)^{-1} \randommatrix{t} \basis{j} \\
		&=& -\frac{\basis{j}^\top \BMcov{} \basis{j} }{2} \left[ v - \left(  \basis{j}^\top \BMcov{} \basis{j} ~ \empiricalcovmat{t} \right)^{-1} \randommatrix{t} \basis{j} \right]^\top \empiricalcovmat{t} \left[ v - \left(  \basis{j}^\top \BMcov{} \basis{j} ~ \empiricalcovmat{t} \right)^{-1} \randommatrix{t} \basis{j} \right].
	\end{eqnarray*}
	In addition, 
	\begin{equation*}
		\sqrt{ \frac{\det \empiricalcovmat{t}}{(2 \pi ~ \basis{j}^\top \BMcov{} \basis{j} )^{\statedim+\controldim}} } \exp \left( -\frac{\basis{j}^\top \BMcov{} \basis{j} }{2} \left[ v - \left(  \basis{j}^\top \BMcov{} \basis{j} ~ \empiricalcovmat{t} \right)^{-1} \randommatrix{t} \basis{j} \right]^\top \empiricalcovmat{t} \left[ v - \left(  \basis{j}^\top \BMcov{} \basis{j} ~ \empiricalcovmat{t} \right)^{-1} \randommatrix{t} \basis{j} \right] \right)
	\end{equation*}
	is the probability density function of a Gaussian random vector whose mean vector and covariance matrix are $\left(  \basis{j}^\top \BMcov{} \basis{j} ~ \empiricalcovmat{t} \right)^{-1} \randommatrix{t} \basis{j}$ and $\left(  \basis{j}^\top \BMcov{} \basis{j} ~ \empiricalcovmat{t} \right)^{-1}$, respectively. Hence, its integral all over $\R^{\statedim+\controldim}$ is one, and we have
	\begin{equation*}
		\E{ \varphi_t \left( V, \basis{j} \right) \Big| \sigfield{\BM{s}, 0 \leq s \leq t} } = \sqrt{ \frac{\det \empiricalcovmat{0}}{(\basis{j}^\top \BMcov{} \basis{j} )^{\statedim+\controldim} \det \empiricalcovmat{t} } }~~ \exp \left( \frac{1}{2} \basis{j}^\top \randommatrix{t}^\top \left(  \basis{j}^\top \BMcov{} \basis{j} ~ \empiricalcovmat{t} \right)^{-1} \randommatrix{t} \basis{j} \right) .
	\end{equation*}
	
	Next, we apply Markov's inequality to the above random variable and employ \eqref{1ExpUppBoundEq}, to get
	\begin{equation*}
		\PP{ \E{ \varphi_t \left( V, \basis{j} \right) \Big| \sigfield{\BM{s}, 0 \leq s \leq t} } > \frac{1}{\delta} } \leq \delta ~\E{ \E{ \varphi_t \left( V, \basis{j} \right) \Big| \sigfield{\BM{s}, 0 \leq s \leq t} } } = \delta ~ \E{\varphi_t} \leq \delta.
	\end{equation*}
	This bound can also be written as 
	\begin{equation*}
		\PP{ \exp \left( \frac{1}{2} \basis{j}^\top \randommatrix{t}^\top \left(  \basis{j}^\top \BMcov{} \basis{j} ~ \empiricalcovmat{t} \right)^{-1} \randommatrix{t} \basis{j} \right) \leq  \sqrt{ \frac{(\basis{j}^\top \BMcov{} \basis{j} )^{\statedim+\controldim} \det \empiricalcovmat{t} }{ \delta^2 \det \empiricalcovmat{0}} }} \geq 1-\delta.
	\end{equation*}
	
	In other words, with probability at least $1-\delta$ it holds that
	\begin{equation*}
		 \basis{j}^\top \randommatrix{t}^\top \empiricalcovmat{t}^{-1} \randommatrix{t} \basis{j} 
		 ~~\leq~~ \basis{j}^\top \BMcov{} \basis{j} ~ \log \frac{(\basis{j}^\top \BMcov{} \basis{j} )^{\statedim+\controldim} \det \empiricalcovmat{t} }{ \delta^2 \det \empiricalcovmat{0}} .
	\end{equation*}
	Finally, we sum up both sides of the above inequality over $j = 1, \cdots, \statedim$. On the right hand side, we simply use $\basis{j}^\top \BMcov{} \basis{j} \leq \eigmax{\BMcov{}}$. For the left hand side, since the matrix $\randommatrix{t}^\top \empiricalcovmat{t}^{-1} \randommatrix{t}$ is positive semidefinite, its trace is an upper-bound for its largest eigenvalue;
	$$\eigmax{ \randommatrix{t}^\top \empiricalcovmat{t}^{-1} \randommatrix{t}} 
	 ~\leq ~\sum\limits_{j=1}^{\statedim} \basis{j}^\top \randommatrix{t}^\top \empiricalcovmat{t}^{-1} \randommatrix{t} \basis{j} . $$ 
	So, the probability that
	\begin{equation*}
		\eigmax{ \randommatrix{t}^\top \empiricalcovmat{t}^{-1} \randommatrix{t}} 
		~~\leq~~ \statedim \eigmax{\BMcov{}} ~ \left[ \log \det \empiricalcovmat{t} - \log \left( \delta^2 \det \empiricalcovmat{0} \right) \right] + \left( \statedim + \controldim \right) \statedim \eigmax{\BMcov{}} \log \eigmax{\BMcov{}},
	\end{equation*}
	fails to hold, is at most $\statedim \delta$. Finally, by replacing $\delta$ in the above inequality with $\delta/\statedim$, it yields to the desired result.  
	
\end{prf}

\subsection{Behaviors of diffusion processes under non-optimal feedback} 
\begin{lemm} \label{LyapLemma}
	Let $\estA{},\estB{}$ be an arbitrary pair of stabilizable system matrices. Suppose that for the closed-loop matrix $\CLmat{}=\estA{}+\estB{}\Gainmat{}$, we have $\eigmax{\exp \left(\CLmat{}\right)}<1$, and $P$ satisfies
	\begin{equation*}
	\CLmat{}^{\top} P + P \CLmat{} + \Qx + \Gainmat{}^{\top} \Qu \Gainmat{}=0.
	\end{equation*}
	Then, it holds that
	\begin{equation*}
	P = \RiccSol{\estpara{}} + \itointeg{0}{\infty}{ e^{\CLmat{}^{\top}t} \left( \Gainmat{} + \Qu^{-1} \estB{}^{\top} \RiccSol{\estpara{}}\right)^{\top} \left( \Qu \Gainmat{} + \estB{}^{\top} \RiccSol{\estpara{}}\right) e^{\CLmat{}t}   }{t}  .
	\end{equation*}
\end{lemm}
\begin{prf}
	Denote $\Optgain{\estpara{}}=-\Qu^{-1} \estB{}^{\top} \RiccSol{\estpara{}}$ and $\estD{}=\estA{}+\estB{}\Optgain{\estpara{}}$. So, after doing some algebra, it is easy to show that the algebraic Riccati equation in \eqref{ARiccEq} gives
	\begin{equation*}
	\estD{}^{\top} \RiccSol{\estpara{}}+ \RiccSol{\estpara{}} \estD{} + \Qmat + \Optgain{\estpara{}}^{\top} \Rmat \Optgain{\estpara{}}.
	\end{equation*}
	Now, let $\randommatrix{} = \Gainmat{}^{\top} \Rmat \Gainmat{} - \Optgain{\estpara{}}^{\top} \Rmat \Optgain{\estpara{}}$, and subtract the above equation that $\RiccSol{\estpara{}}$ solves, from the similar one in the statement of the lemma that $P$ satisfies, to get
	\begin{equation} \label{LyapLemmaProofEq1}
	\left( \CLmat{} - \estD{1} \right)^{\top} \RiccSol{\estpara{}} + \RiccSol{\estpara{}} \left( \CLmat{} - \estD{1} \right) + \CLmat{}^{\top} \left( P- \RiccSol{\estpara{}} \right) + \left( P - \RiccSol{\estpara{}} \right) \CLmat{} + \randommatrix{}=0.
	\end{equation}
	Because $\eigmax{\exp\left(\CLmat{}\right)} <1$, by solving~\eqref{LyapLemmaProofEq1} for $P - \RiccSol{\estpara{}}$, we have
	\begin{equation*}
	P - \RiccSol{\estpara{}} = \itointeg{0}{\infty}{ e^{\CLmat{}^{\top}t} \left( \randommatrix{}+ \left[ \Gainmat{} - \Optgain{\estpara{}} \right]^{\top} \estB{}^{\top} \RiccSol{\estpara{}} + \RiccSol{\estpara{}} \estB{} \left[ \Gainmat{} - \Optgain{\estpara{}} \right] \right) e^{\CLmat{}t} }{t},
	\end{equation*}
	where the fact $\CLmat{} - \estD{} = \estB{} \left[ \Gainmat{} - \Optgain{\estpara{}} \right]$ is used above. Then, using $\estB{}^{\top} \RiccSol{\estpara{}} = - \Rmat \Optgain{\estpara{}}$, it is straightforward to see
	\begin{eqnarray}
		&&\left( \Gainmat{} - \Optgain{\estpara{}} \right)^{\top} \Rmat \left( \Gainmat{} - \Optgain{\estpara{}} \right) \notag \\
		&=& \randommatrix{} + \left[ \Gainmat{} - \Optgain{\estpara{}} \right]^{\top} \estB{}^{\top} \RiccSol{\estpara{}} + \RiccSol{\estpara{}} \estB{} \left[ \Gainmat{} - \Optgain{\estpara{}} \right], \label{LyapAuxEq}
	\end{eqnarray}
	which leads to the desired result.
\end{prf}

\subsection{Behaviors of diffusion processes in a neighborhood of the truth} 
\begin{lemm} \label{ApproxParaLem}
	Letting $\stabradii$ be as defined in~\eqref{StabRadiiEq}, assume that
	\begin{equation} \label{StabAccuracyEq}
	\Mnorm{\estpara{}-\truth}{arg2} \lesssim \frac{\eigmin{\Qu}}{ \Mnorm{\Bmat{0}}{} \Mnorm{\RiccSol{\truth}}{2} } \left( \frac{\left[ \stabradii \wedge 1 \right]^{\statedim}}{\statedim }  \wedge \frac{\eigmin{\Qx}}{ \Mnorm{\RiccSol{\truth}}{2} } \right).
	\end{equation}
	Then, for the Riccati equation in \eqref{ARiccEq} which is denoted by $\RiccSol{\para{}}$, we have $\Mnorm{\RiccSol{\estpara{}}}{2} \lesssim \Mnorm{\RiccSol{\truth}}{2}$. Furthermore, for any eigenvalue $\lambda$ of $\estA{}-\estB{} \Qu^{-1} \estB{}^{\top} \RiccSol{\estpara{}}$, it holds that $\Re \left( \lambda \right) \lesssim -\eigmin{\Qx} \Mnorm{\RiccSol{\truth}}{2}^{-1}$.
\end{lemm}

\begin{prf}
	First, let us write $\CLmat{0}=\Amat{0}- \Bmat{0} \Qu^{-1} \Bmat{0}^{\top} \RiccSol{\truth}$ and
	\begin{equation*}
	\CLmat{1}=\estA{}-\estB{}\Qu^{-1} \Bmat{0}^{\top} \RiccSol{\truth} = \CLmat{0}+\erterm{1},
	\end{equation*}
	where $\erterm{1}=\estA{}-\Amat{0}-\left( \estB{}-\Bmat{0} \right)\Qu^{-1} \Bmat{0}^{\top} \RiccSol{\truth}$. Since $\mult{} \leq \statedim$, \eqref{StabAccuracyEq} implies that $\erterm{1}$ satisfies
	\begin{equation*}
	\Mnorm{\erterm{1}}{2} \lesssim \mult{} \left[ \stabradii \wedge 1 \right]^{\mult{}}.
	\end{equation*}
	So, letting $M=\CLmat{0}$ in \eqref{EigPerturb}, Lemma~\ref{EigPerturbLem} leads to the fact that all eigenvalues of $\CLmat{1}$ are on the open left half-plane of the complex plane. Now, in Lemma~\ref{LyapLemma}, let $\Gainmat{}=-\Qu^{-1} \Bmat{0}^{\top} \RiccSol{\truth}$ and $\CLmat{}=\CLmat{1}$, to obtain the matrix denoted by $P$ in the lemma. Since $P$ satisfies
	\begin{equation*}
	\Qmat+ \Gainmat{}^{\top} \Rmat \Gainmat{}=-\CLmat{1}^{\top} P - P \CLmat{1} = -\CLmat{0}^{\top} P - P \CLmat{0} - \erterm{1}^{\top} P - P \erterm{1},
	\end{equation*}
	writing Lemma~\ref{LyapLemma} for $\CLmat{}=\CLmat{0}$, but replacing $\Qmat$ with $\Qmat+ \erterm{1}^\top P + P \erterm{1}$, we have
	\begin{eqnarray*}
		P &=& \itointeg{0}{\infty}{ e^{\CLmat{0}^{\top} t} \left[ \Qmat+ \Gainmat{}^{\top} \Rmat \Gainmat{} + \erterm{1}^{\top} P + P \erterm{1} \right] e^{\CLmat{0}t} }{t} .
	\end{eqnarray*}
	However, according to \eqref{ARiccEq}, we have
	\begin{eqnarray} \label{LyapInteg}
	\RiccSol{\truth} &=& \itointeg{0}{\infty}{ e^{\CLmat{0}^{\top} t} \left[ \Qmat+ \Gainmat{}^{\top} \Rmat \Gainmat{} \right] e^{\CLmat{0}t} }{t}.
	\end{eqnarray}
	Thus, it holds that
	\begin{equation*}
	P = \RiccSol{\truth} + \itointeg{0}{\infty}{ e^{\CLmat{0}^{\top} t} \left[ \erterm{1}^{\top} P + P \erterm{1} \right] e^{\CLmat{0}t} }{t},
	\end{equation*}
	which leads to
	\begin{equation*}
	\Mnorm{P}{2} \leq \Mnorm{\RiccSol{\truth}}{2} + 2 \Mnorm{\erterm{1}}{2} \Mnorm{P}{2} \itointeg{0}{\infty}{ \Mnorm{e^{\CLmat{0}t}}{2}^2 }{t}.
	\end{equation*}
	We will shortly show that $2 \Mnorm{\erterm{1}}{2} \itointeg{0}{\infty}{ \Mnorm{e^{\CLmat{0}t}}{2}^2 }{t} <1$. So, by Lemma~\ref{LyapLemma}, we have $\Mnorm{\RiccSol{\estpara{}}}{2} \leq \Mnorm{P}{2} \lesssim \Mnorm{\RiccSol{\truth}}{2}$, which is the desired result.
	
	To proceed, denote the closed-loop matrix by $\estD{}=\estA{}-\estB{} \Qu^{-1} \estB{}^{\top} \RiccSol{\estpara{}}$, and let the $\statedim$ dimensional unit vector $\nu$ attain the maximum of $\norm{\exp (\estD{}) \nu }{2}$, i.e., $\norm{\exp (\estD{}) \nu }{2} = \Mnorm{\exp (\estD{}) }{2}$. Then, \eqref{LyapInteg} for $\estpara{}$ (instead of $\truth$) implies that
	\begin{eqnarray*}
		\Mnorm{\RiccSol{\estpara{}}}{2} &\geq& \nu^\top \RiccSol{\estpara{}} \nu
		= \itointeg{0}{\infty}{ \nu^\top e^{\estD{}^{\top} t} \left[ \Qmat + \RiccSol{\estpara{}}^{\top} \estB{} \Rmat^{-1} \estB{}^\top \RiccSol{\estpara{}} \right] e^{\estD{} t} \nu }{t}.
	\end{eqnarray*}
	Therefore, $\eigmin{\Qmat + \RiccSol{\estpara{}}^{\top} \estB{} \Rmat^{-1} \estB{}^\top \RiccSol{\estpara{}}} \geq \eigmin{\Qx}$, together with the fact that the magnitudes of all eigenvalues are smaller than the operator norm, imply that for an arbitrary eigenvalue $\lambda$ of $\estD{}$, we have
	\begin{equation} \label{StabThmProofEq3}
	\Mnorm{\RiccSol{\truth}}{2} \gtrsim \Mnorm{\RiccSol{\estpara{}}}{2} \geq \eigmin{\Qmat} \itointeg{0}{\infty}{ e^{2 \Re \left( \lambda \right) t} }{t},
	\end{equation}
	which leads to the second desired result of the lemma. To complete the proof, we need to establish that $\Mnorm{\erterm{1}}{2} \itointeg{0}{\infty}{ \Mnorm{e^{\CLmat{0}t}}{2}^2 }{t} <1/2$. For that purpose, if we write \eqref{StabThmProofEq3} for $\truth$ instead of $\estpara{}$, the condition in \eqref{StabAccuracyEq} implies the above bound.
\end{prf}

\subsection{Perturbation analysis for algebraic Riccati equation in \eqref{ARiccEq}} 
\begin{lemm} \label{LipschitzLemma}
	Assume that \eqref{StabAccuracyEq} holds. Then, we have
	\begin{equation*}
	\Mnorm{\RiccSol{\estpara{}}-\RiccSol{\truth}}{2}
	\lesssim \frac{ \Mnorm{\RiccSol{\truth}}{2}^2 }{\eigmin{\Qmat}} \left( 1 \vee \Mnorm{ \Qu^{-1} \Bmat{0}^{\top} \RiccSol{\truth}}{2} \right)  \Mnorm{\estpara{}-\truth}{}.
	\end{equation*}
\end{lemm}
\begin{prf}
	First, fix the dynamics matrix $\estpara{}$, and let $\curve$ be a linear segment connecting $\truth$ and $\estpara{}$:
	\begin{equation*}
	\curve=\left\{ (1-\ssconstant) \truth + \ssconstant \estpara{} \right\}_{0 \leq \ssconstant \leq 1}.
	\end{equation*}
	Let $\para{1} \in \curve$ be arbitrary. Then, the derivative of $\RiccSol{\para{}}$ at $\para{1}$ in the direction of $\curve$ can be found by using the difference matrices $\erterm{A}=\estA{}-\Amat{0}$, $\erterm{B}=\estB{}-\Bmat{0}$. Denote $\erterm{}=\left[\erterm{A},\erterm{B}\right]^\top$. Then, we find $\RiccSol{\para{2}}$, where $\para{2}=\para{1}+ \epsilon \erterm{}$, for an infinitesimal value of $\epsilon$.  So, we have
	\begin{eqnarray*}
		&& \RiccSol{\para{1}} \Bmat{2} \Rmat^{-1} \Bmat{2}^{\top} \RiccSol{\para{1}} \\
		&=& \epsilon \RiccSol{\para{1}} \erterm{B} \Rmat^{-1} \Bmat{2}^{\top} \RiccSol{\para{1}} + \RiccSol{\para{1}} \Bmat{1} \Rmat^{-1} \Bmat{2}^{\top} \RiccSol{\para{1}} \\
		&=& \order{\epsilon^2} + \epsilon \RiccSol{\para{1}} \erterm{B} \Rmat^{-1} \Bmat{1}^{\top} \RiccSol{\para{1}} 
		+
		\epsilon \RiccSol{\para{1}} \Bmat{1} \Rmat^{-1} \erterm{B}^{\top} \RiccSol{\para{1}} +
		\RiccSol{\para{1}} \Bmat{1} \Rmat^{-1} \Bmat{1}^{\top} \RiccSol{\para{1}}.
	\end{eqnarray*}
	Therefore, we can calculate $\RiccSol{\para{2}} \Bmat{2} \Rmat^{-1} \Bmat{2}^{\top} \RiccSol{\para{2}}$. To that end, let $P= \RiccSol{\para{2}}-\RiccSol{\para{1}}$, write $\RiccSol{\para{2}}$ in terms of $P,\RiccSol{\para{1}}$, and use the above result to get
	\begin{eqnarray}
		&&\RiccSol{\para{2}} \Bmat{2} \Rmat^{-1} \Bmat{2}^{\top} \RiccSol{\para{2}} \notag\\
		&=& \RiccSol{\para{2}} \Bmat{2} \Rmat^{-1} \Bmat{2}^{\top} P + \RiccSol{\para{2}} \Bmat{2} \Rmat^{-1} \Bmat{2}^{\top} \RiccSol{\para{1}} \notag\\
		&=&  \order{\Mnorm{P}{2}^2} + \RiccSol{\para{1}} \Bmat{2} \Rmat^{-1} \Bmat{2}^{\top} P
		+ P \Bmat{2} \Rmat^{-1} \Bmat{2}^{\top} \RiccSol{\para{1}}
		+ \RiccSol{\para{1}} \Bmat{2} \Rmat^{-1} \Bmat{2}^{\top} \RiccSol{\para{1}} \notag\\
		&=& \order{\Mnorm{P}{2}^2} + \RiccSol{\para{1}} \Bmat{2} \Rmat^{-1} \Bmat{2}^{\top} P
		+ P \Bmat{2} \Rmat^{-1} \Bmat{2}^{\top} \RiccSol{\para{1}} \notag\\
		&+& \order{\epsilon^2} + \epsilon \RiccSol{\para{1}} \erterm{B} \Rmat^{-1} \Bmat{1}^{\top} \RiccSol{\para{1}} 
		+
		\epsilon \RiccSol{\para{1}} \Bmat{1} \Rmat^{-1} \erterm{B}^{\top} \RiccSol{\para{1}} \notag \\ 
		&+&
		\RiccSol{\para{1}} \Bmat{1} \Rmat^{-1} \Bmat{1}^{\top} \RiccSol{\para{1}} \label{LipschitzLemmaProofEq1}.	
	\end{eqnarray}
	Again, expanding $\Amat{2}=\Amat{1}+\erterm{A}$ and $\RiccSol{\para{2}}=\RiccSol{\para{1}}+P$, it yields to
	\begin{eqnarray*}
		&& \Amat{2}^{\top} \RiccSol{\para{2}} + \RiccSol{\para{2}} \Amat{2}\\
		&=& \Amat{2}^{\top} \RiccSol{\para{1}} + \Amat{2}^{\top} P + \RiccSol{\para{1}} \Amat{2} + P \Amat{2} \\
		&=& \Amat{1}^{\top} \RiccSol{\para{1}} + \epsilon \erterm{A}^{\top} \RiccSol{\para{1}} + \Amat{2}^{\top} P \\
		&+& \RiccSol{\para{1}} \Amat{1} + \epsilon \RiccSol{\para{1}} \erterm{A} + P \Amat{2}.
	\end{eqnarray*}
	To proceed, plug in the continuous-time algebraic Riccati equation in~\eqref{ARiccEq} for $\para{1},\para{2}$  below in the above expression: 
	\begin{eqnarray*}
		 \Amat{2}^{\top} \RiccSol{\para{2}} + \RiccSol{\para{2}} \Amat{2}
		&=& \RiccSol{\para{2}} \Bmat{2} \Qu^{-1} \Bmat{2}^{\top} \RiccSol{\para{2}} + \Qx , \\
		 \Amat{1}^{\top} \RiccSol{\para{1}} + \RiccSol{\para{1}} \Amat{1}
		&=& \RiccSol{\para{1}} \Bmat{1} \Qu^{-1} \Bmat{1}^{\top} \RiccSol{\para{1}} + \Qx.
	\end{eqnarray*}
	
	So, we obtain
	\begin{eqnarray*}
		&& \Amat{2}^{\top} \RiccSol{\para{2}} + \RiccSol{\para{2}} \Amat{2}  - \Amat{1}^{\top} \RiccSol{\para{1}}  - \RiccSol{\para{1}} \Amat{1} \\
		&=& \epsilon \erterm{A}^{\top} \RiccSol{\para{1}}  
		 + \epsilon \RiccSol{\para{1}} \erterm{A} + P \Amat{2} + \Amat{2}^{\top} P \\
		 &=& \RiccSol{\para{2}} \Bmat{2} \Qu^{-1} \Bmat{2}^{\top} \RiccSol{\para{2}} 
		 - \RiccSol{\para{1}} \Bmat{1} \Qu^{-1} \Bmat{1}^{\top} \RiccSol{\para{1}} \\
		 &=& \order{\Mnorm{P}{2}^2} + \RiccSol{\para{1}} \Bmat{2} \Rmat^{-1} \Bmat{2}^{\top} P
		 + P \Bmat{2} \Rmat^{-1} \Bmat{2}^{\top} \RiccSol{\para{1}} \notag\\
		 &+& \order{\epsilon^2} + \epsilon \RiccSol{\para{1}} \erterm{B} \Rmat^{-1} \Bmat{1}^{\top} \RiccSol{\para{1}} 
		 +
		 \epsilon \RiccSol{\para{1}} \Bmat{1} \Rmat^{-1} \erterm{B}^{\top} \RiccSol{\para{1}},
	\end{eqnarray*}
	where in the last equality above, we used \eqref{LipschitzLemmaProofEq1}. Now, rearrange the terms in the above statement to get an equation that does not contain any expression in term of $\para{2}$. So, it becomes
	\begin{eqnarray*}
		0 &=& \left[ \Amat{2}^{\top} - \RiccSol{\para{1}} \Bmat{2} \Rmat^{-1} \Bmat{2}^{\top} \right] P + P \left[ \Amat{2} - \Bmat{2} \Rmat^{-1} \Bmat{2}^{\top} \RiccSol{\para{1}} \right] - \order{\Mnorm{P}{2}^2}   - \order{\epsilon^2} \\
		&+& \epsilon \erterm{A}^{\top} \RiccSol{\para{1}} + \epsilon \RiccSol{\para{1}} \erterm{A} - \epsilon \RiccSol{\para{1}} \erterm{B} \Rmat^{-1} \Bmat{1}^{\top} \RiccSol{\para{1}} -
		\epsilon \RiccSol{\para{1}} \Bmat{1} \Rmat^{-1} \erterm{B}^{\top} \RiccSol{\para{1}}   .
	\end{eqnarray*}
	Next, to simplify the above equality, define the followings:
	\begin{eqnarray*}
		\auxA &=& \Amat{2} - \Bmat{2} \Rmat^{-1} \Bmat{2}^{\top} \RiccSol{\para{1}},\\
		\Optgain{\para{1}} &=& - \Qu^{-1} \Bmat{1}^{\top} \RiccSol{\para{1}}, \\
		\auxQ &=& \epsilon \RiccSol{\para{1}} \Big[ \erterm{A} + \erterm{B} \Optgain{\para{1}} \Big] +
		\epsilon \Big[ \Optgain{\para{1}}^{\top} \erterm{B}^{\top} + \erterm{A}^{\top} \Big] \RiccSol{\para{1}} - \order{\epsilon^2} .
	\end{eqnarray*}
	So, writing our equation in terms of $\auxA,\Optgain{\para{1}},\auxQ$, it gives
	\begin{equation} \label{LipschitzLemProofEq1}
	0 = \auxA^{\top} P + P \auxA - \order{\Mnorm{P}{2}^2} + \auxQ.
	\end{equation}	
	
	The discussion after \eqref{OptimalPolicy} states that all eigenvalues of $\CLmat{1}=\Amat{1} - \Bmat{1} \Rmat^{-1} \Bmat{1}^{\top} \RiccSol{\para{1}}$ lie in the open left half-plane. Therefore, if $\epsilon$ is small enough, real-parts of all eigenvalues of $\auxA$ are negative, according to Lemma~\ref{EigPerturbLem}. Therefore, \eqref{LipschitzLemProofEq1} implies that
	\begin{equation*}
	\eigmax{P} \leq \eigmax{\itointeg{0}{\infty}{ e^{\auxA^{\top}t} \auxQ e^{\auxA t} }{t}} \leq  \Mnorm{\auxQ}{2} \itointeg{0}{\infty}{ \Mnorm{e^{\auxA t}}{2}^2 }{t}.
	\end{equation*}
	So, as $\epsilon$ decays, $\auxQ$ vanishes, which by the above inequality shows that $P$ shrinks as $\epsilon$ tends to zero. Further, as $\epsilon$ decays, $\auxA$ converges to $\CLmat{1}$. Thus,  by \eqref{LipschitzLemProofEq1}, we have
	\begin{equation} \label{LipschitzLemProofEq2}
	\lim\limits_{\epsilon \to 0} \epsilon^{-1} P = \itointeg{0}{\infty}{ e^{\CLmat{1}^{\top}t} \left( \RiccSol{\para{1}} \left[\erterm{A} + \erterm{B} \Optgain{\para{1}}\right] + \left[\erterm{A} + \erterm{B} \Optgain{\para{1}}\right]^{\top}  \RiccSol{\para{1}} \right) e^{\CLmat{1} t} }{t}.
	\end{equation}
	Recall that the above expression is the derivative of $\RiccSol{\para{}}$ at $\para{1}$, along the linear segment $\curve$. Thus, integrating along $\curve$, \eqref{LipschitzLemProofEq2} and Cauchy-Schwarz Inequality imply that
	\begin{eqnarray*}
		\Mnorm{\RiccSol{\estpara{}}-\RiccSol{\truth}}{2}
		&\lesssim& \Mnorm{\estpara{}-\truth}{} \sup\limits_{\para{1} \in \curve}  \Mnorm{\RiccSol{\para{1}}}{2} \left( 1 \vee \Mnorm{\Optgain{\para{1}}}{2} \right) \itointeg{0}{\infty}{ \Mnorm{e^{\CLmat{1}t}}{2}^2 }{t} .
	\end{eqnarray*}
	Finally, using Lemma~\ref{ApproxParaLem}, \eqref{LyapInteg}, and \eqref{StabThmProofEq3}, we obtain the desired result.
\end{prf}
\end{APPENDICES}

\end{document}